\documentclass[final,12pt]{colt2026} 

\usepackage{mathrsfs}
\usepackage{dsfont}
\usepackage{hyperref}
\hypersetup{
    colorlinks = true,
    linkcolor = blue,
    citecolor=blue
    }
\usepackage[
        noabbrev,
        capitalise
    ]
    {cleveref}
\usepackage{ulem}
\usepackage{enumerate}

\newtheorem{proposition}{Proposition}

\newtheorem{lemma}{Lemma}

\newtheorem{assumptionalt}{Assumption}[theorem]

\newenvironment{hypp}[1]{
  \renewcommand\theassumptionalt{#1}
  \assumptionalt
}{\endassumptionalt}

\title[Convergence Rates for Distribution Matching with Sliced Optimal Transport]{Convergence Rates for Distribution Matching with \\
Sliced Optimal Transport}
\usepackage{times}

\newcommand{\Id}{\text{Id}}
\newcommand{\RR}{\mathbb{R}}
\newcommand{\Tr}{\text{Tr}}
\newcommand{\ie}{\textit{i.e.}}
\newcommand{\eg}{\textit{e.g.}}
\newcommand{\rmd}{\mathrm{d}}
\newcommand{\Jac}{\text{Jac}\,}
\newcommand{\Lt}{\mathbf{L}^2}

\DefineNamedColor{named}{PineGreen}{cmyk}{0.92,0,0.59,0.25}

\coltauthor{%
 \Name{Gauthier Thurin} \Email{gthurin@mail.di.ens.fr}\\
 \addr CNRS, ENS Paris, France
 \AND
 \Name{Claire Boyer}\\
 \addr LMO, Université Paris-Saclay, Orsay, France ; Institut universitaire de France
 \AND 
 \Name{Kimia Nadjahi} \\
 \addr CNRS, ENS Paris, France
}

\begin{document}

\maketitle

\begin{abstract}%
    We study 
    the slice-matching scheme, an efficient 
     iterative method for distribution matching based on sliced optimal transport.
    We investigate convergence to the target distribution and derive quantitative non-asymptotic rates. 
    To this end, we establish Lojasiewicz-type inequalities for the Sliced-Wasserstein objective. 
    A key challenge is to control along the trajectory the constants in these inequalities. 
    We show that this becomes tractable 
    for Gaussian distributions. 
    Specifically, 
    eigenvalues are controlled 
    when matching along random orthonormal bases at each iteration. 
    We complement our theory with numerical experiments and illustrate the predicted dependence on dimension and step-size, as well as the stabilizing effect of orthonormal-basis sampling.\\

\end{abstract}


\section{Introduction}

Many problems in modern machine learning require comparing and matching probability distributions, as in generative modeling \citep{marzouk2016introduction,pmlr-v202-grenioux23a}, density estimation \citep{wang2022minimax,irons2022triangular} or domain adaptation \citep{courty2016optimal}. The goal is typically to transform a source distribution in order to match a more complex target distribution.

\paragraph{Distribution matching and optimal transport.} Distribution matching can be naturally formalized through optimal transport (OT), which provides both a geometrically meaningful distance between probability measures and, when it exists, a transport map pushing a source distribution $\sigma$ to a target distribution $\mu$ \citep{villani2008optimal,ambrosio2007gradient}. 
OT-based methods have led to major theoretical and algorithmic advances across
machine learning, image processing and scientific computing
\citep{peyre2019computational,santambrogio2015optimal}. However, computing OT maps is in general expensive both computationally and statistically~\citep{HutterMinimax2021,chewi2024statistical}.

\paragraph{Iterative approaches and measure interpolations.} 
The high cost of OT has
motivated alternative approaches that decompose the transport problem into simpler subproblems. A key idea is to build an interpolation between $\sigma$ and $\mu$ through a sequence of elementary transformations, rather than estimating a single global transport map. This idea underlies many iterative correction schemes: although each step may only partially reduce the discrepancy between $\sigma$ and $\mu$, their composition is expected to gradually align them. Among all possible interpolations, the McCann interpolation~\citep{mccann1997convexity} plays a distinguished theoretical role, as it corresponds to geodesics in Wasserstein space, but it is rarely tractable. A generic iterative sequence of measures that mimics McCann’s interpolation can be constructed through 
\begin{equation}
\label{iterativemeasures} \widehat{\sigma}_{k+1} = \big((1-\gamma_k) \Id + \gamma_k \widehat{T}_k\big)_\sharp \widehat{\sigma}_k, 
\end{equation} 
where $\widehat{T}_k$ is an approximate transport map from $\widehat{\sigma}_k$ to $\mu$, $(\gamma_k)_k$ a sequence of step sizes. Here, $T_\sharp \sigma$ denotes the pushforward of $\sigma$ by the function $T$: if $X\sim\sigma$, then $T(X) \sim T_\sharp \sigma$.

Different choices for $\widehat{T}_k$ have been proposed, such as entropy-regularized OT~\citep{kassraie2024progressive} and neural-network parameterizations in diffusion or flow-based models~\citep{song2021scorebased,albergo2025}. In this work, we focus on sliced optimal transport, a computationally efficient alternative 
that leverages one-dimensional projections
~\citep{pitie2007automated,rabin2011wasserstein,rabin2012}. 

\paragraph{Sliced optimal transport and slice-matching maps.} The Sliced-Wasserstein distance (SW) compares two distributions by projecting them onto one-dimensional subspaces and averaging the resulting Wasserstein distances~\citep{rabin2011wasserstein,rabin2012}. Thanks to its scalability and simple implementation, SW has attracted growing interest in large-scale applications, including generative modeling~\citep{deshpande2019max,wu2019sliced,liutkus2019sliced,kolouri2018slicedae,dai21a,coeurdoux2022,du2023}. This empirical success has in turn motivated theoretical work on the geometry induced by sliced OT, sample complexity, and convergence properties of associated algorithms~\citep{nadjahi2019asymptotic,  nadjahi2020statistical,manole2022,tanguy2023convergence,tanguy2025properties,li2023measure,vauthier2025properties}.

Although sliced OT does not directly provide transport maps or geodesics~\citep{kitagawa2024,park2025geometry}, several constructions have been proposed in this spirit \citep{liu2025expected,mahey2023fast}. 
In particular, \textit{slice-matching maps}~\citep{pitie2007automated,li2024approximation} correspond to Wasserstein gradients of the SW functional \citep{li2023measure}.
For a direction $\theta \in \mathbb{S}^{d-1}$, let $\sigma^\theta$ and $\mu^\theta$ denote the push-forwards of $\sigma$ and $\mu$ by the projection $x \mapsto \langle x,\theta\rangle$. Denoting by $T_{\sigma^\theta}^{\mu^\theta} : \mathbb{R} \to \mathbb{R}$ the univariate optimal transport map from $\sigma^\theta$ to $\mu^\theta$, the associated slice-matching map is defined by \begin{equation}\label{slicemap1D} T_{\sigma,\theta}(x) = x + \big(T_{\sigma^\theta}^{\mu^\theta}(\theta^T x) - \theta^\top  x\big) \theta \,. 
\end{equation} 
Since the probability mass is transported along a single direction, $T_{\sigma, \theta}$ does not transport $\sigma$ to $\mu$. The \textit{Iterative Distribution Transfer} (IDT) algorithm \citep{pitie2007automated} therefore constructs an iterative composition of slice-matching maps, corresponding to \eqref{iterativemeasures} with constant step sizes $\gamma_k = 1$. Using random directions $\theta$ at each iteration, this procedure is expected to gradually push $\sigma$ to $\mu$ and has been successfully applied in practice.

\paragraph{Related works.} The IDT algorithm \citep{pitie2007automated} was introduced before the Sliced-Wasserstein distance~\citep{rabin2011wasserstein,rabin2012} and was later interpreted as an iterative sliced OT procedure. Early works established convergence of the IDT iterates when the target is the standard Gaussian distribution and studied its continuous-time limit, often referred to as the \textit{Sliced-Wasserstein flow}~\citep{pitie2007automated,bonnotte2013unidimensional}. More recently, \citet{cozzi2025} proved convergence of SW flows to the isotropic Gaussian. Relatively few results are available on the convergence of sliced OT procedures beyond the Gaussian setting. A more general analysis is conducted in \citet{li2023measure}, which reinterprets IDT as a stochastic gradient descent method (SGD) on SW and accounts for time discretization and randomness in the sampled directions. They prove asymptotic convergence of the discrete-time dynamics under strong assumptions, notably that the iterates remain in a compact set containing no other critical points than the target measure. In parallel, several works have studied SW as a loss between discrete measures and highlight the existence of nontrivial critical points, which motivate noisy or regularized variants of SGD~\citep{tanguy2024reconstructing,tanguy2025properties,vauthier2025properties}.

\paragraph{Contributions.} 
The main goal of this paper is to establish convergence rates for the slice-matching scheme~\citep{li2023measure}. Our approach is based on identifying Polyak--\L{}ojasiewicz (PL) inequalities for the Sliced-Wasserstein objective, which bound the loss by the squared norm of its Wasserstein gradient. 
These inequalities imply quantitative convergence rates to the target distribution. The main technical challenge is that the associated constants depend on lower and upper bounds on the density of the iterates, which are difficult to control along the trajectory.

We address this difficulty within the class of elliptic distributions, for which slice-matching maps are linear. In this regime, controlling the density of the iterates amounts to controlling the eigenvalues of their covariance matrices. When the target distribution is isotropic, we show that these eigenvalues can be controlled in expectation, which in turn yields explicit convergence rates. 
Crucially, such spectral control holds from the very first iteration when the updates use random orthonormal bases of directions. This stands in contrast with the single-direction setting, where the lack of orthogonality leads to larger fluctuations in the covariance structure before stabilization.

\paragraph{Structure.} Section~\ref{secMathFramework} introduces the mathematical framework. Preliminary convergence results to critical points are discussed in Section~\ref{CvgceGradientSmooth}. Section~\ref{PL_KL_ineq} presents our main results on \L{}ojasiewicz- and PL-type inequalities and on the control of the associated constants. Numerical experiments are reported in Section~\ref{numexp}, followed by a conclusion. 
Technical proofs are deferred to the appendices.
Python codes are available at \url{https://github.com/gauthierthurin/SlicedMaps}.

\paragraph{Notation.}
For any probability measure $\nu$ on $\mathbb{R}^d$, let $\mathrm{M}_2(\nu) = \int_{\mathbb{R}^d} \|x\|^2 \rmd \nu(x)$ be its second moment. 
$\mathcal{P}_2(\RR^d)$ refers to the set of measures with a finite second moment and $\mathcal{P}_{2,ac}(\mathbb{R}^d) \subset \mathcal{P}_2(\mathbb{R}^d)$ is the set of absolutely continuous measures with respect to the Lebesgue measure. We denote the Euclidean norm and inner product on $\RR^d$ by $\|\cdot\|$ and $\langle\cdot,\cdot\rangle$. 
For $\nu\in\mathcal{P}_2(\RR^d)$, we define $\Lt(\nu) = \{ f : \RR^d \to \RR^d\ :\ \int_{\RR^d} \| f(x) \|^2 \mathrm{d}\nu(x) \leq +\infty \}$, and for $f,g\in \Lt(\nu)$, $\langle f,g\rangle_\nu = \int_{\RR^d} \langle f(x),g(x)\rangle \mathrm{d}\nu(x)$, $\|f\|_\nu = \sqrt{\langle f,f\rangle_\nu}$.
Let $\mathbb{S}^{d-1} = \{ \theta \in \RR^d: \Vert \theta \Vert = 1\}$ be the unit sphere in $\RR^d$.
For any $\theta \in \mathbb{S}^{d-1}$, $\pi_\theta: \RR^d \rightarrow \RR $ is the projection $\pi_\theta(x) = \langle x,\theta\rangle$. Finally, $\lambda_i(A)$ refers to the $i$-th smallest eigenvalue of a matrix $A$, with $\lambda_{\rm min}(A)$ the smallest and $\lambda_{\rm max}(A)$ the largest. 

\section{Background on the Slice-Matching Scheme}\label{secMathFramework}

We begin by reviewing the definition of optimal transport and its properties for one-dimensional measures, which motivates slicing.
Let $T_\sigma^\mu$ denote the OT map from $\sigma$ to $\mu$, defined as a minimizer in the Wasserstein distance: $W_2^2(\sigma,\mu) = \inf_{T: \, T_\sharp \sigma = \mu} \mathbb{E}_{X\sim\sigma} \big\| X - T(X) \big\|^2$. In dimension one, the optimal transport map admits a closed-form expression,
$
T_\sigma^\mu = F_\mu^{-1} \circ F_\sigma ,
$
where $F_\rho$ is the cumulative distribution function of $\rho \in \mathcal{P}_2(\mathbb{R})$.
This motivates the definition of the Sliced-Wasserstein distance, which averages one-dimensional Wasserstein distances over random projections:
\[
SW_2^2(\sigma,\mu) = \int_{\mathbb{S}^{d-1}} W_2^2(\sigma^\theta,\mu^\theta)\, \mathrm{d}\mathcal{U}(\theta),
\]
where $\sigma^\theta = (\pi_\theta)_\sharp \sigma$ and $\mu^\theta = (\pi_\theta)_\sharp \mu$, and $\mathcal{U}$ is the uniform distribution on $\mathbb{S}^{d-1}$.

\paragraph{Slice-matching maps and scheme.}
We now introduce the slice-matching construction that underlies the iterative scheme studied in this paper.
Let $\mu \in \mathcal{P}_{2,ac}(\mathbb{R}^d)$ be a target probability measure, and let
$P = [\theta_1,\ldots,\theta_d] \in \mathbb{R}^{d\times d}$ be an orthonormal basis of $\mathbb{R}^d$.
For any direction $\theta \in \mathbb{S}^{d-1}$, denote by $t_\theta = T_{\sigma^\theta}^{\mu^\theta}$ the one-dimensional optimal transport map pushing the projected measure $\sigma^\theta$ onto $\mu^\theta$.

Rather than transporting mass along a single direction, we simultaneously match $d$ orthogonal one-dimensional projections.
This leads to the definition of the (matrix-)slice-matching map
\begin{equation}\label{slicemapMatrix}
\forall x \in \mathbb{R}^d,\quad
T_{\sigma,P}(x)
= x + P
\begin{bmatrix}
t_{\theta_1}(\theta_1^\top x) - \theta_1^\top x \\
t_{\theta_2}(\theta_2^\top x) - \theta_2^\top x \\
\vdots \\
t_{\theta_d}(\theta_d^\top x) - \theta_d^\top x
\end{bmatrix}
= \sum_{i=1}^d t_{\theta_i}(\theta_i^\top x)\, \theta_i ,
\end{equation}
where the last equality follows from the fact that $P$ is an orthonormal basis.
Using several orthogonal directions at each iteration has been observed to significantly improve both stability and empirical performance
\citep{pitie2007automated,bonneel2015sliced,li2023measure}.
From a theoretical standpoint, matrix-slice-matching maps enjoy a moment-matching property
\citep[Proposition~3.6]{li2024approximation}, which will play a central role in our analysis:
\begin{align*}
\mathbb{E}_{Y \sim (T_{\sigma,P})_\sharp \sigma}[Y]
= \mathbb{E}_{Y \sim \mu}[Y],
\qquad
\mathrm{M}_2\big((T_{\sigma,P})_\sharp \sigma\big)
= \mathrm{M}_2(\mu).
\end{align*}
The \textit{slice-matching scheme}, main focus of this paper, is defined as follows:
Starting from an initial distribution $\sigma_0 = \sigma \in \mathcal{P}_{2,ac}(\mathbb{R}^d)$, the iterates are given by
\begin{align}
\forall k \geq 0, \qquad
\sigma_{k+1}
= \big((1-\gamma_k)\Id + \gamma_k T_{\sigma_k,P_{k+1}}\big)_\sharp \sigma_k ,
\label{IDT}
\end{align}
where $(P_k)_{k \geq 1}$ is an i.i.d.\ sequence of random orthonormal bases drawn according to the Haar measure on $O(d)$ (the set of $d\times d$ orthonormal matrices) and $(\gamma_k)_{k \geq 0}$ consist of positive step sizes satisfying the Robbins-Monro conditions
\begin{equation}\label{choicestep-size}
\sum_{k \geq 0} \gamma_k = +\infty,
\qquad
\sum_{k \geq 0} \gamma_k^2 < +\infty .
\end{equation}

\paragraph{Stochastic gradient descent perspective.}
The slice-matching scheme admits a natural interpretation as a stochastic gradient descent procedure in the $2$-Wasserstein space for a Sliced-Wasserstein loss \citep{li2023measure}.
Specifically, consider the variational problem
\begin{equation}
\min_{\sigma \in \mathcal{P}_2(\mathbb{R}^d)} \; \mathscr{F}(\sigma),
\qquad
\text{with}
\qquad
\mathscr{F}(\sigma) = \frac{d}{2}\, SW_2^2(\sigma,\mu).
\end{equation}
For $P = [\theta_1,\ldots,\theta_d]$ an orthonormal basis of $\mathbb{R}^d$,
defining
\(
\mathscr{F}(\sigma,P)
= \frac{1}{2}\sum_{\ell=1}^d W_2^2(\sigma^{\theta_\ell},\mu^{\theta_\ell}),
\)
one has the decomposition
\[
\mathscr{F}(\sigma) = \mathbb{E}_P\big[\mathscr{F}(\sigma,P)\big],
\]
where the expectation is taken with respect to $P$\footnote{This equality follows from the invariance of the Haar measure, which ensures that the marginal distribution of each direction $\theta_\ell$ is uniform on $\mathbb{S}^{d-1}$, even though the directions are not independent.}. Both $\mathscr{F}$ and $\mathscr{F}(\cdot,P)$ depend on the target measure $\mu$, a dependence that we omit in the notation for simplicity.
The Wasserstein gradient of the random functional $\mathscr{F}(\cdot,P)$ is given by
\[
\nabla_{W_2} \mathscr{F}(\sigma,P) = \Id - T_{\sigma,P},
\]
and provides an unbiased estimator of the full Wasserstein gradient:
$
\mathbb{E}_P\big[ \nabla_{W_2} \mathscr{F}(\sigma,P) \big]
= \nabla_{W_2} \mathscr{F}(\sigma),
$
see \citet{rabin2011wasserstein,bonnotte2013unidimensional,li2023measure,cozzi2025} and \Cref{PropGradients} (Appendix \ref{appPropertiesLoss}). As a consequence, the slice-matching iteration \eqref{IDT} can be rewritten as a stochastic gradient descent update in Wasserstein space.
For any $k \geq 0$,
\begin{equation}
\sigma_{k+1}
= \big(\Id - \gamma_k (\Id - T_{\sigma_k,P_{k+1}})\big)_\sharp \sigma_k
= \big(\Id - \gamma_k \nabla_{W_2} \mathscr{F}(\sigma_k,P_{k+1})\big)_\sharp \sigma_k .
\end{equation}
For completeness, \Cref{app:remindWassSpace} recalls basic notions of differentiation in Wasserstein space.

\paragraph{Bounded gradients.}
\citet{cozzi2025} show that second-order moments are bounded along the Sliced-Wasserstein flow. In our discrete time setting that incorporates stochastic choices of directions $P_{k+1}$, we can show that the same holds as a result of the aforementioned moment-matching property of slice-matching maps (see Proposition \ref{prop_momentsbounded}, Appendix \ref{appPropertiesLoss}).
Combining this with $\Vert \nabla_{W_2} \mathscr{F}(\sigma) \Vert^2_{\sigma} \leq  2 \mathscr{F}(\sigma)$ (by Jensen's inequality; see Proposition \ref{PropGradients}, Appendix \ref{appPropertiesLoss}), one has  
\begin{equation*}
    \forall k \geq 0, \quad \Vert \nabla_{W_2} \mathscr{F}(\sigma_k) \Vert^2_{\sigma_k} \leq 2 \mathscr{F}(\sigma_k) \leq 4 \mathrm{M}_2(\mu)\,. 
\end{equation*}

\paragraph{Smoothness and non-convexity.} 
A key property for SGD is the \textit{smoothness} of the objective function. 
It is shown in \citet{vauthier2025properties} (and recalled in \Cref{appsmoothness}) 
that $\mathscr{F}$ is 1-smooth in $\mathcal{P}_2(\RR^d)$ endowed with $W_2$:
for any $\sigma_1,\sigma_2 \in \mathcal{P}_{2}(\RR^d)$ such that the OT map $T_{\sigma_1}^{\sigma_2}$ exists,
\begin{equation}\label{SmoothnessLoss2bis}
    \mathscr{F}(\sigma_2) \leq \mathscr{F}(\sigma_1) + \langle \nabla \mathscr{F}(\sigma_1),T_{\sigma_1}^{\sigma_{2}}-\Id \,  \rangle_{\sigma_1} + \frac{1}{2}W_2^2(\sigma_1,\sigma_2).
    \end{equation}
    Smoothness alone, however, is not sufficient to guarantee almost-sure convergence towards $\mu$. 
    In Wasserstein spaces, convergence rates typically rely on \textit{geodesic convexity} \citep{ambrosio2007gradient}, which $\mathscr{F}$ does not satisfy in general \citep{vauthier2025properties}. Nevertheless, convergence is observed in practice \citep{pitie2007automated,rabin2011wasserstein}, which suggests that the optimization landscape remains highly structured, as studied in the next section. 


\section{Preliminary Analysis: Convergence to Critical Points}\label{CvgceGradientSmooth}

We recall convergence results from \citet{li2023measure} and derive new results about averages of gradient norms with standard proofs that use the smoothness property.


\paragraph{Descent lemma.}
The following lemma is a key recursion inequality that serves as a standard descent condition in stochastic optimization.
The proof follows by the smoothness property \eqref{SmoothnessLoss2bis} and direct computations, in the same fashion as for optimization over Euclidean spaces. 

\begin{lemma}[\citet{li2023measure}, Lemma A.1]\label{lemmaA1limoosmuller}
Let $(\sigma_k)_{k \geq 1}$ be the iterates generated by the slice-matching scheme \eqref{IDT}. 
Then, for any $k \geq 0$,
\begin{equation}\label{recursiveIneq}
\mathbb{E}[\mathscr{F}(\sigma_{k+1})\vert \mathcal{A}_{k}] \leq 
(1 + \gamma_k^2) \mathscr{F}(\sigma_k) -\gamma_k \Vert \nabla_{W_2} \mathscr{F}(\sigma_k) \Vert_{\sigma_k}^2 \,,
\end{equation}
where $\mathcal{A}_k$ is the $\sigma$-field generated by $(P_1,\dots,P_k)$. 
\end{lemma}

Given recursion~\eqref{recursiveIneq} and step-sizes assumptions~\eqref{choicestep-size}, a direct application of Robbins-Siegmund theorem \citep{robbins1971convergence} implies that $(\mathscr{F}(\sigma_k))_{k \geq 0}$ converges almost surely to a finite random variable, and that
\begin{equation}\label{RobbinsSiegmund}
    \sum_{k\geq 1}\gamma_k \big\Vert \nabla_{W_2} \mathscr{F}(\sigma_k) \big\Vert_{\sigma_k}^2 < +\infty \quad a.s.
\end{equation}
\noindent An immediate byproduct is that a subsequence of $(\Vert \nabla_{W_2} \mathscr{F}(\sigma_k) \Vert_{\sigma_k})_{k \geq 1}$ converges almost surely to $0$, or equivalently 
$\lim \inf_{k\rightarrow +\infty}\Vert \nabla_{W_2} \mathscr{F}(\sigma_k) \Vert_{\sigma_k} = 0$. 
Besides, $(\Vert \nabla_{W_2} \mathscr{F}(\sigma_k) \Vert_{\sigma_k})_{k \geq 1}$ converges almost surely to $0$ if the sequence $(\sigma_k)_{k \geq 1}$ remains in a compact subset of $(\mathcal{P}_{2,ac}(\RR^d),W_2)$ \citep[Theorem 2]{li2023measure}. 
This holds true for instance if $\sigma_0$ and $\mu$ are continuous and compactly supported, or under finite third-order moments \citep[Remark 9]{li2023measure}. 
Under the additional assumption that
$
\nabla \mathscr{F}(\sigma) = 0\Longleftrightarrow \sigma = \mu,
$
the limit of $\sigma_k$ must be $\mu$ almost surely. 
To the best of our knowledge, the only known sufficient condition for this equivalence is that densities are strictly positive on their compact support \citep[Lemma~5.7.2]{bonnotte2013unidimensional}.

\paragraph{Convergence Guarantees to Critical Points.}
The next proposition establishes convergence toward a critical point using standard arguments, up to a random reshuffling of the indices \citep{ghadimi2013stochastic}. This result is weaker than the almost sure convergence $\sigma_k\overset{a.s.}{\rightarrow}\mu$ from
\citet[][Theorem 2]{li2023measure}, but it has the benefit of requiring no additional assumptions than the ones of \Cref{lemmaA1limoosmuller}. 
Here, this means absolute continuity for $\sigma$ and $\mu$, although smoothness 
\eqref{SmoothnessLoss2bis} holds in fact in the more difficult setting of \citet{vauthier2025properties} where $(\sigma_k)$ are discrete. 
In this case, the next two propositions could be extended.

\begin{proposition}\label{convergenceGradient}
    For any $K \in \mathbb{N}$, let $i(K)$ be a random index such that 
    $
    \forall k \in \{1, \dots, K\}, \; \mathbb{P} (i(K) = k) = 1/K
    $. Then, $(\Vert \nabla \mathscr{F}(\sigma_{i(K)})\Vert_{\sigma_{i(K)}}^2)_{K \geq 0}$ converges in probability towards $0$, \textnormal{i.e.}, 
    $$
    \forall \epsilon >0, \quad \lim\limits_{K\rightarrow+\infty}
    \mathbb{P}\big(\Vert \nabla \mathscr{F}(\sigma_{i(K)})\Vert_{\sigma_{i(K)}}^2 >\epsilon\big)=0\,. 
    $$
\end{proposition}
Turning to convergence rates, assuming smoothness and boundedness of the iterates only yields the following result, which concerns a weighted average of the gradients.  

\begin{proposition}\label{propratePonderatedAvgGrad} 
    For a number $K$ of iterations, 
    define the weigths $\omega_j = \gamma_j /\sum_{k=1}^K \gamma_k$, for any $0\leq j \leq K$, where $(\gamma_j)_j$ are the chosen learning rates. Then,
\begin{equation}\label{ratePonderatedAvgGrad}
    \sum_{k=0}^K \omega_k \mathbb{E}[ \Vert \nabla_{W_2}\mathscr{F}(\sigma_k)\Vert_{\sigma_k}^2] \leq \frac{\mathscr{F}(\sigma_0)+4\mathrm{M}_2(\mu)\sum_{k=0}^K \gamma_k^2}{\sum_{k=0}^K \gamma_k}\,.
\end{equation}
\end{proposition}
\noindent When choosing $\gamma_k = 1/(k+1)^\alpha$ for $1/2< \alpha < 1$, considering that the numerator is bounded by a constant, Proposition \ref{propratePonderatedAvgGrad} yields a rate of order $K^{\alpha-1}$, since
$
\sum_{k=0}^K \gamma_k \geq \frac{1}{1-\alpha}(K^{1-\alpha}-1).
$
We also emphasize that the bound \eqref{ratePonderatedAvgGrad} would tend to zero for a constant step-size $\gamma_k = 1/\sqrt{K+1}$ given a finite time horizon $K$ \citep[as in][]{ghadimi2013stochastic,khaled2022better}. 

These propositions complement the related work by \citet{vauthier2025properties}
that also study convergence towards critical points.
Their setting is different in that they consider a discrete source $\sigma$, a continuous target $\mu$, a constant learning rate and their gradients are theoretically computed from all directions $\theta\in \mathbb{S}^{d-1}$, as opposed to our stochastic gradients along finitely many directions. 

The convergence results obtained so far 
are standard for stochastic optimization of smooth losses with bounded gradients \citep{bottou2018optimization,dossal2024optimization}. 
For completeness, proofs are provided in \Cref{proofssmoothSGD}.
In the remainder of this paper, we will assume appropriate continuity conditions, allowing us to strengthen and extend the preceding results.

In particular, our \L{}ojasiewicz inequalities imply that the assumptions of \citet[Theorem 2]{li2023measure} hold for Gaussian measures. This readily gives almost-sure convergence, as stated hereafter and proved in \Cref{proofAScvgceGauss}. 

\begin{proposition}\label{ASconvergencegaussianHyp}
    Let $\sigma = \mathcal{N}(0,\Sigma)$ and $\mu = \mathcal{N}(0,\Lambda)$, with $\Sigma,\Lambda \in \mathbb{R}^{d \times d}$ strictly positive definite. Let $(\gamma_k)$ satisfy the Robbins-Monro conditions \eqref{choicestep-size}. Then, $\lim_{k\to+\infty} \mathcal{F}(\sigma_k) = 0$ almost surely. 
\end{proposition}
The almost-sure convergence of the objective can be converted into convergence in $W_2$, using the compactness of the iterates and the metric properties of $SW_2$, as in the proof of \citep[Corollary~4.3]{cozzi2025}.

\begin{corollary}
    Under the assumptions of \Cref{ASconvergencegaussianHyp}, $\lim_{k\to+\infty} W_2(\sigma_k,\mu) = 0$ almost surely.   
\end{corollary}

\section{Convergence Analysis under \L{}ojasiewicz Inequalities} \label{PL_KL_ineq}

This section is devoted to the derivation of quantitative convergence rates for the slice-matching scheme. Our main result concerns Gaussian source and target measures.

\subsection{Main result: convergence analysis for Gaussian measures}

\begin{theorem}\label{VitesseN01}
    Assume $\sigma = \mathcal{N}(0,\Sigma)$ and $\mu = \mathcal{N}(0,{\bf I}_d)$, where $\Sigma \in \mathbb{R}^{d \times d}$ is symmetric positive definite. 
    Let $\gamma_k = 1/(k+1)^\alpha$. 
    For $2/3<\alpha<1$, it exists $C>0$ such that, for all $k\geq1$,
    $$
    \mathbb{E}[\mathscr{F}(\sigma_k)] \leq \frac{C}{ k^{2\alpha-1}}.
    $$
    For $0 <\alpha<2/3$, it exists $C>0$ such that, for all $k\geq1$ and for all $0<\epsilon<\min(\alpha,1-\alpha)$,
    $$
    \mathbb{E}[\mathscr{F}(\sigma_k)]  \leq \frac{C}{ k^{1-\alpha-\epsilon}}.
    $$
\end{theorem}
The complete proof is deferred to \Cref{ProofMainResult}. The remainder of this section presents the main ingredients and is organized as follows. We first introduce a general framework showing how convergence rates follow from a \textit{random Polyak--Łojasiewicz (PL) inequality} along the trajectory. We then discuss how such inequalities can be established in a \textit{static} fashion under density bounds, and why propagating these bounds is difficult in general. Finally, we show that the Gaussian structure allows one to control the corresponding PL constants through spectral estimates on covariance matrices, which leads to \Cref{VitesseN01}.

\subsection{Step 1: From (random) PL inequalities to rates}
 
Our starting point is a gradient-variance decomposition (see Appendix~\ref{app:differentiability_critical_points}, \Cref{PropGradients}) which isolates Łojasiewicz-type inequalities as the key ingredient. Denoting $\overline{T}_{\sigma}=\mathbb{E}_P[T_{\sigma,P}]$, one has
\begin{equation}\label{decompL_Gradient}
2\mathscr{F}(\sigma)
=
\Vert \nabla_{W_2}\mathscr{F}(\sigma)\Vert_\sigma^2
+
\mathbb{E}_P\!\left[\Vert \overline{T}_\sigma - T_{\sigma,P}\Vert_\sigma^2\right].
\end{equation}
If the variance term is controlled by the squared Wasserstein gradient norm, \ie, if there exists $s>0$ such that
\[
\mathbb{E}_P\!\left[\Vert \overline{T}_\sigma - T_{\sigma,P}\Vert_\sigma^2\right]
\le
s^2\Vert \nabla_{W_2}\mathscr{F}(\sigma)\Vert_\sigma^2,
\]
then~\eqref{decompL_Gradient} yields a Polyak--Łojasiewicz inequality
\[
\mathscr{F}(\sigma)
\le
B\Vert \nabla_{W_2}\mathscr{F}(\sigma)\Vert_\sigma^2,
\qquad
B=\tfrac{1+s^2}{2},
\]
a standard condition to prove convergence rates in nonconvex optimization
\citep[\eg][]{garrigos2023handbook}.
This motivates the search for PL inequalities that hold \textit{along the iterates} $(\sigma_k)_{k\ge0}$ with constants that can be controlled. We formalize this requirement through the following random \L ojasiewicz-type condition \citep{kurdyka2000proof,attouch2010proximal}.
\begin{hypp}{A}\label{hyp_PL}
    For some $\tau \in\{1,2\}$ and any $k \geq 1$, $\mathscr{F}(\sigma_k)^\tau \leq B_k \Vert \nabla_{W_2} \mathscr{F}(\sigma_k) \Vert_{\sigma_k}^2$ with $(B_k)_{k \geq 1}$ a sequence of positive random variables s.t. $\sup_{k \geq 1} \mathbb{E} [ B_k^p] \leq c_p$ with $c_p \in (0, +\infty)$ for all $p\in \mathbb{N}^*$.
\end{hypp} 

By combining such 
inequalities along the trajectory with the descent recursion for $\mathscr{F}(\sigma_k)$ (Lemma \ref{lemmaA1limoosmuller}), we obtain the following rates. 
\begin{theorem}\label{thmPL} 
Consider Assumption~\ref{hyp_PL} with $\tau = 1$. Choose the step sequence as $\gamma_k = 1/(k+1)^\alpha$.
\begin{enumerate}
    \item[(i)] If $0 < \alpha < 2/3$,  then, for any $k \geq 1$, 
    $\mathbb{E}[\mathscr{F}(\sigma_k)] \lesssim k^{-(1-\alpha-\epsilon)}$  for all $0<\epsilon<\min(\alpha,1-\alpha)$.
    \item[(ii)] If $2/3 < \alpha < 1$,  then, for any $k \geq 1$, 
    $\mathbb{E}[\mathscr{F}(\sigma_k)] \lesssim k^{-(2\alpha-1)}.$
\end{enumerate}

\noindent Alternatively, consider Assumption~\ref{hyp_PL} with $\tau = 2$. 
For $p\geq 2\alpha/(2-3\alpha)$, let $\gamma =  (2\mathrm{M}_2(\mu) \sqrt{c_p})^{3/2}$. 
Let $\gamma_k = 1/(k+\gamma)^\alpha$ with $1/2<\alpha <2/3$.
Then, for any $k \geq 1,\; \mathbb{E}[\mathscr{F}(\sigma_k)] \lesssim
1 / (k+\gamma)^{2\alpha - 1}$.
\end{theorem}
Only finitely many moments of $B_k$ are required for the analysis. More precisely, the proof requires $\sup_{k\geq 1} \mathbb{E}[B_k^p]<\infty$ for some $p>4\alpha/(1-\alpha)$ when $\tau=1$, and for some $p\geq 2\alpha/(2-3\alpha)$ when $\tau=2$. For simplicity of exposition, Assumption~\ref{hyp_PL} is stated with uniform bounds for all $p \in \mathbb{N}^*$.

Beyond the slice-matching setting, the proof strategy applies more generally to optimization schemes with smooth objectives, whose gradients are bounded and that satisfy Assumption~\ref{hyp_PL}. 
The argument follows a standard template: one first derives a descent recursion, and then applies an appropriate variant of Chung’s lemma \citep{chung1954stochastic,jiang2024generalizedversionchungslemma}. 
The main additional difficulty here is that the PL constant $B_k$ is random. We address this by working on events of the form $\{B_k \leq g_k^{-1}\}$ where $g_k\rightarrow 0$ is chosen so that these events eventually occur almost surely. Similar arguments appear in \citet[Theorem~4.2]{godichon2019lp} and \citet[Theorem~3.6]{BBAOS} to leverage local strong convexity.
The main remaining difficulty is therefore to verify Assumption~\ref{hyp_PL} for the slice-matching iterates.

\subsection{Step 2: Static PL inequalities and bounded densities}

In this section, we show that \L ojasiewicz-type inequalities can be established in a \textit{static} manner, \ie, for fixed measures with uniformly bounded densities.

\paragraph{Gradient domination for bounded densities.} For notational simplicity, we identify any $\sigma \in \mathcal{P}_{2,ac}(\RR^d)$ with its density. Given a reference measure $\nu \!\in\!\mathcal{P}_{2,ac}(\mathbb{R}^d)$, we consider the convenient setting of measures with uniformly bounded densities
\begin{equation}\label{densitiesBounded}
    \mathcal{P}_{\nu,m,M}(\mathbb{R}^d)
    =
    \{\sigma \in \mathcal{P}(\mathbb{R}^d) : m\nu \le \sigma \le M\nu\}\,,
\end{equation}
for which the following gradient-domination inequality can be obtained.
\begin{proposition}\label{propPoincareflatconvex}
     Assume that $\nu\in \mathcal{P}_{2, ac}(\RR^d)$ satisfies a Poincar\'e inequality with constant $C_\nu > 0$, \ie, for any $f : \mathbb{R}^d \to \RR$ such that $\| \nabla f \|^2_\nu < +\infty$, 
    $
        \mbox{Var}_{\nu}(f) \; \triangleq \; \left\| f - \mathbb{E}_\nu[f] \right\|^2_{\nu} \; \leq \; C_\nu \Vert \nabla f \Vert_\nu^2 
    $.
    Then, if $\mu \in \mathcal{P}_{\nu, m, M}(\RR^d)$, for any $\sigma \in \mathcal{P}_{\nu, m, M}(\RR^d)$,\; $$\mathscr{F}(\sigma)\leq 2C_\nu \frac{M}{m} \big\| 
    \nabla_{W_2}\mathscr{F} (\sigma) \big\|_{\sigma}.$$
\end{proposition}
The proof follows arguments similar to \citet[Lemma~3.3]{chizat2025convergence}.  
Note that, combined with $\| \nabla_{W_2}\mathscr{F}(\sigma) \|_\sigma^2\leq 2\mathscr{F}(\sigma)$, we obtain the two-sided estimate $$
\Vert \nabla_{W_2}\mathscr{F}(\sigma) \Vert_\sigma^2/2 \leq \mathscr{F}(\sigma)\leq 2C_\nu \frac{M}{m} \big\| \nabla_{W_2}\mathscr{F} (\sigma)\big\|_{\sigma}\,.
$$
In particular, $\nabla_{W_2} \mathscr{F}(\sigma) = 0$ if and only if $\mathscr{F}(\sigma) = 0$, \ie, $\sigma = \mu$. We therefore retrieve a characterization of critical points by \citet[Lemma~5.7.2]{bonnotte2013unidimensional}, where compactness of the support is no longer required.  

\paragraph{PL inequality for Gaussians.} 
We now turn to the Gaussian setting, in which PL inequalities can be established.
We consider the class
\begin{equation}\label{eq:gaussian_set}
    \mathcal{G}_{m,M}
    =
    \{\rho_\Sigma : \Sigma \in \mathrm{S}_{++}^d,\;
    m\mathbf{I}_d \preceq \Sigma \preceq M\mathbf{I}_d\},
\end{equation}
where $\rho_{\Sigma} = \mathcal{N}(0, \Sigma)$, and  $\mathrm{S}_{++}^d$ is the set of positive definite $d \times d$ matrices. The notation $\preceq$ refers to the Loewner partial order: for two symmetric matrices $(A,B)$, $A \preceq B$ if and only if $B-A$ is positive semi-definite. 
Therefore, $\mathcal{G}_{m,M}$ corresponds to Gaussian measures with uniformly bounded covariance eigenvalues.

\begin{proposition}[PL inequality on $\mathcal{G}_{m,M}$]
\label{corollarydiagGaussians}
Let $\sigma = \rho_{\Sigma}$ and $\mu = \rho_{\Lambda}$ such that $\Sigma, \Lambda$ are simultaneously diagonalizable by an orthogonal matrix (\ie, co-diagonalizable). Assume $\rho_\Sigma, \rho_\Lambda \in \mathcal{G}_{m,M}$. Let $C_d = d(d+2) M/m\,$. Then,
\begin{equation}\label{eqcorollarydiagGaussians}
\mathscr{F}(\sigma)
\;\le\;
\frac{C_{d}}{2}
\Bigl(1 + \frac{M}{m}\Bigr)
\;\big\| \nabla_{W_2} \mathscr{F}(\sigma) \big\|_{\sigma}^{2}\,.
\end{equation}
\end{proposition}
\Cref{corollarydiagGaussians} is proved by adapting \citet[Theorem 19]{chewi2020gradient}, which yields an intermediate inequality relating $\mathscr{F}(\sigma)$ and $\| \nabla_{W_2} \mathscr{F}(\sigma) \|$ for $\sigma, \mu \in \mathcal{G}_{m,M}$ (see Appendix \ref{app:proof_PL_Gaussian}). We then refine it into a PL inequality by proving that, for co-diagonalizable covariances,
\begin{equation} \label{eq:W2_vs_SW2_gauss}
    W_2^2(\rho_{\Sigma}, \rho_{\Lambda}) \leq C_d\,SW_2^2(\rho_{\Sigma}, \rho_{\Lambda}).
\end{equation}
To our knowledge, this is the first comparison between $W_2$ and $SW_2$ with polynomial dimension dependence,
instead of exponential dependence obtained in general settings, \eg, \citet[Theorem 5.1.5]{bonnotte2013unidimensional} and \citet{carlier2025sharpcomparisonsslicedstandard}. 
This result may be of independent interest for other research problems involving Gaussian distributions and the Bures-Wasserstein metric.

\paragraph{From static inequalities to iterate stability.} To use \Cref{propPoincareflatconvex} (or \Cref{corollarydiagGaussians}) in a convergence analysis, one must ensure that the iterates $(\sigma_k)_{k\ge0}$ remain in $\mathcal{P}_{\nu,m,M}(\mathbb{R}^d)$ (or $\mathcal{G}_{m,M}$)  with constants $m,M$ uniform in $k$. However, if $\sigma_k$ satisfies such bounds, propagating them to $\sigma_{k+1}$ is challenging.
Indeed, since $\sigma_{k+1} = S_{k\#}\sigma_k$, the change-of-variables formula yields
\begin{equation}\label{changeVarDens}
    \sigma_{k+1}(S_k(x))
    =
    \frac{\sigma_k(x)}{\det \Jac[S_k](x)}.
\end{equation}
Thus, propagating density bounds reduces to controlling $\det\Jac[S_k]$, which typically requires strong regularity estimates on $S_k$; see \eg,
\citet{caffarelli1992regularity,caffarelli2000monotonicity,bobkov2019one,park2025geometry}.
One possible way to circumvent this difficulty in general settings is to introduce diffusion through entropic regularization \citep{chizat2025convergence}, but this leads to a different class of distribution-matching algorithms \citep{liutkus2019sliced} and falls outside the scope of the present work. This observation motivates restricting attention to settings, such as the Gaussian case, where the relevant constants can instead be controlled through an alternative, more tractable mechanism.

\subsection{Step 3: Propagating PL constants along the trajectory in the Gaussian case}\label{SecEigenControlConvergRates}

To apply \Cref{thmPL}, it remains to verify Assumption~\ref{hyp_PL} along the slice-matching trajectory, in the Gaussian setting. 

\paragraph{Slice-matching on the Bures-Wasserstein manifold.} We begin by making the slice-matching updates explicit when matching two Gaussians. Let $\mu = \rho_\Lambda$ and for a fixed $k \in \mathbb{N}$, $\sigma_k = \rho_{\Sigma_k}$. Then, for any $\theta \in \mathbb{S}^{d-1}$, the one-dimensional projections satisfy $\sigma_k^\theta = \mathcal{N}(0, \theta^\top  \Sigma \theta)$, $\mu^\theta = \mathcal{N}(0, \theta^\top  \Lambda \theta)$, and the corresponding optimal transport map between these marginals is linear and given by 
$$
T_{\sigma_k^\theta}^{\mu^\theta}(s) = \tau_\theta s, 
\quad \text{with} \quad  \tau_\theta = \sqrt{\theta^\top  \Lambda \theta/\theta^\top  \Sigma \theta}.
$$    
For $P_{k+1} = [\theta_1, \cdots, \theta_d]$, define the diagonal matrix $D_k=\mathrm{diag}(\tau_{\theta_1},\ldots,\tau_{\theta_d})$. The resulting slice-matching map $T_{\sigma_k, P_{k+1}}$ is also linear:  $\forall x \in \RR^d, \;T_{\sigma_k, P_{k+1}}(x)=P_{k+1}D_kP_{k+1}^\top  x$. As a consequence, the iterates remain Gaussian \citep{altschuler2021averaging}, \ie, $\sigma_k=\rho_{\Sigma_k}$, with covariance matrices evolving according to the following recursion
\begin{equation}\label{expliciteUpdateSigma}
\Sigma_{k+1} =
A_k\,\Sigma_k\,A_k^\top ,
\hspace{1cm}
A_k=(1-\gamma_k)\mathbf{I}_d+\gamma_k P_{k+1}D_kP_{k+1}^\top  \,.
\end{equation}

\begin{remark}[Centered Gaussians]
If $\gamma_1=1$, the first iteration enforces equality of the means of  $\sigma_1$ and $\mu$ due to the moment matching property of slice-matching maps \citep[Proposition 3.6]{li2024approximation}. Therefore, we may assume without loss of generality that $\sigma$ and $\mu$ are centered. 
\end{remark}

\begin{remark}[Elliptically contoured distributions]
All results of this section extend beyond the Gaussian case to elliptically contoured distributions. The key structural property used throughout is the linearity of OT maps, which also holds in this broader class \citep[Theorem~2.1]{gelbrich1990formula}.
\end{remark}

\paragraph{Control of PL constants along the trajectory.} 
The convergence analysis relies on PL inequalities whose constants depend inversely on the smallest eigenvalue of the covariance matrices $\Sigma_k$. Therefore, obtaining quantitative convergence rates requires uniform (in $k$) control of $1/\lambda_{\min}(\Sigma_k)$, in expectation and with finite moments. We thus proceed in three steps:

\begin{enumerate}[(a)]
\item \underline{``Static" \L{}ojasiewicz inequalities with random constants.}
Under the trace bound $\Tr(\Sigma_k) \leq \Tr(\Lambda)$ (\Cref{prop_momentsbounded}), one has
$\lambda_{\min}(\Sigma_k) \leq \lambda_{\max}(\Lambda)$.
Consequently, \Cref{KLGaussians,corollarydiagGaussians} yield, for $\tau\in\{1,2\}$,
\begin{equation}\label{eq:loj_with_random_constant}
\mathscr{F}(\sigma_k)^{\tau}
\;\le\;
B_k\,\bigl\|\nabla_{W_2}\mathscr{F}(\sigma_k)\bigr\|_{\sigma_k}^{2},
\quad
B_k \;\lesssim\; \frac{1}{\lambda_{\min}(\Sigma_k)}.
\end{equation}
Thus, the PL constant along the trajectory is random and may deteriorate if $\lambda_{\min}(\Sigma_k)$ becomes small, which motivates a quantitative control of this quantity.

\item \underline{Recursion on $\lambda_{\min}(\Sigma_{k+1})$.} We exploit the explicit covariance update \eqref{expliciteUpdateSigma}. We show that for any $k \geq 0$, there exists a direction $\theta_i$ among the columns of $P_{k+1}$ such that (\Cref{ControlEigSigk})
\begin{equation}\label{recIneqEigenvaluesSigma_kMainbody}
\sqrt{\lambda_{\min}(\Sigma_{k+1})}
\;\ge\;
\sqrt{\lambda_{\min}(\Sigma_k)}\bigl(1-\gamma_k+\gamma_k\tau_{\theta_i}\bigr)\,,
\quad
\tau_{\theta_i}=\sqrt{\frac{\theta_i^\top\Lambda\theta_i}{\theta_i^\top\Sigma_k\theta_i}}.
\end{equation}
Since $\Sigma_0\succ 0$ and $\Lambda\succ 0$, it holds by induction that $\lambda_{\min}(\Sigma_k)> 0$ for all finite $k$. Hence, the PL inequality in \Cref{corollarydiagGaussians} is well-defined along the trajectory.

\item \underline{Moment control of $1/\lambda_{\min}(\Sigma_k)$.}
We now leverage the recursion~\eqref{recIneqEigenvaluesSigma_kMainbody} to bound $1/\lambda_{\min}(\Sigma_k)$ in expectation. 
A sufficient condition is provided by \Cref{sufficientconditionarbitrarycov}: for some $p\geq 1$,
\begin{equation}\label{eq:sufficient_condition}
\mathbb{E}\!\left[\sum_{k\ge 0}\gamma_k\,
\mathbb{E}_{\theta}\!\left[\left(\frac{\theta^\top\Sigma_k\theta}{\theta^\top\Lambda\theta}\right)^{p}-1\right]\right]<\infty.
\end{equation}
We are able to verify~\eqref{eq:sufficient_condition} in the isotropic target case $\Lambda=\mathbf{I}_d$, although our numerical experiments suggest that \eqref{eq:sufficient_condition} is verified for more general target covariances.
More precisely, for any $p\in\mathbb{N}^*$, $\sup_{k\ge 1}\mathbb{E}\!\left[\lambda_{\min}(\Sigma_k)^{-p}\right]<\infty$ when $\Lambda=\mathbf{I}_d$ (\Cref{ControlEigSigk2}).
\end{enumerate}
Combining the PL inequality \eqref{eq:loj_with_random_constant} with the above moment bounds shows that Assumption~\ref{hyp_PL} holds along the Gaussian slice-matching trajectory. Applying \Cref{thmPL} then yields the convergence rate stated in \Cref{VitesseN01}.

\section{Numerical Experiments}\label{numexp}

\subsection{Matching Gaussians}

We implement the slice-matching scheme with source $\sigma=\mathcal{N}(0,\Sigma)$ and target $\mu=\mathcal{N}(0,\mathbf{I}_d)$ to illustrate our theoretical insights from \Cref{PL_KL_ineq}. The updates are computed exactly following the explicit covariance recursion~\eqref{expliciteUpdateSigma}. We run the algorithm for different dimensions $d \in [5,100]$ and step-size schedules $\gamma_k=(k+1)^{-\alpha}$ with $\alpha\in[0,1)$. For each $(d,\alpha)$, we perform $N=10$ independent runs (independent initializations of $\Sigma$), and we track the loss $SW_2^2(\sigma_k,\mu)$ (to verify convergence) and the extreme eigenvalues $\lambda_{\min}(\Sigma_k)$, $\lambda_{\max}(\Sigma_k)$.

\paragraph{Convergence and impact of $(d, \alpha)$.} \Cref{ContinuousIDTtargetstandard} reports $SW_2^2(\sigma_k,\mu)$ as a function of the iteration $k$. For all tested dimensions $d$, the loss decreases, indicating convergence of the iterates toward the target measure. As $d$ increases, the decay becomes slower, in agreement with our theoretical results, since the constants in our bounds scale polynomially with $d$. Similarly, the extreme eigenvalues converge to 1, which confirms that $\Sigma_k$ becomes $\mathbf{I}_d$. 
\Cref{ContinuousIDTtargetstandard} also shows that smaller values of $\alpha$ (\ie, more aggressive step sizes) yield faster empirical convergence, with $\alpha\in\{0,0.1\}$ typically performing best. This behavior is not captured by our non-asymptotic analysis, derived for $\alpha > 0.5$. 
Extending the theory to values of $\alpha$ close to 0 remains an open problem.

\begin{figure}[t!]
    \centering
    {\subfigure[Setting]{\includegraphics[width=0.18\textwidth]{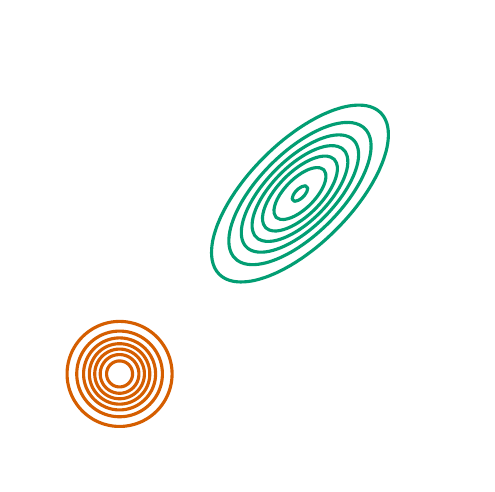} \label{Viz_BetweenGaussiansContinuous}}}
    {\subfigure[$\alpha=0$]{\includegraphics[width=0.19\textwidth]{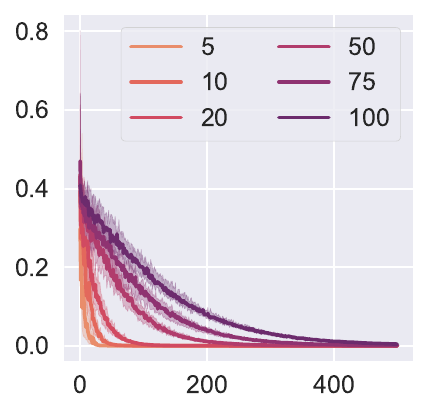} }}
    {\subfigure[$\alpha=0.1$]{\includegraphics[width=0.19\textwidth]{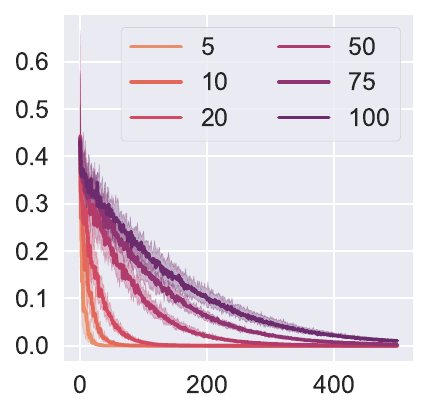} }}
    {\subfigure[$\alpha=0.51$]{\includegraphics[width=0.19\textwidth]{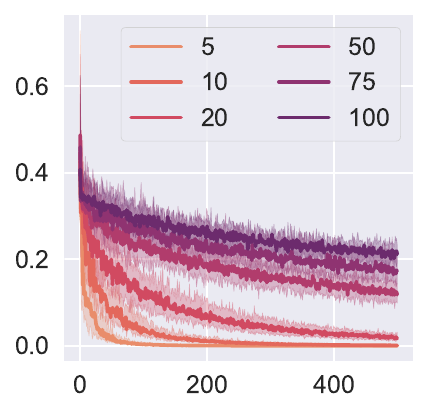} }}
    {\subfigure[$\alpha=0.9$]{\includegraphics[width=0.19\textwidth]{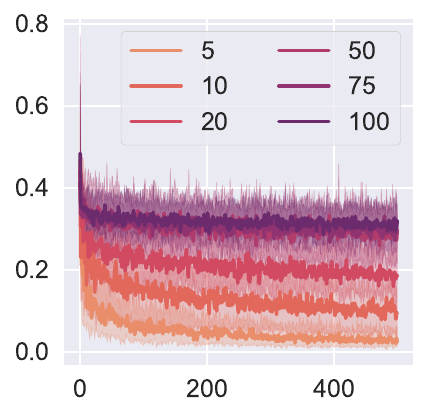}}}
    \caption{Evolution of $SW_2^2(\sigma_k, \mu)$ when $\sigma =\mathcal{N}(0,\Sigma)$ and $\mu=\mathcal{N}(0,{\bf I }_d)$}
    \label{ContinuousIDTtargetstandard}
\end{figure}
\begin{figure}[t!]
    \centering
    {\subfigure[$\alpha=0$]{\includegraphics[width=0.23\linewidth]{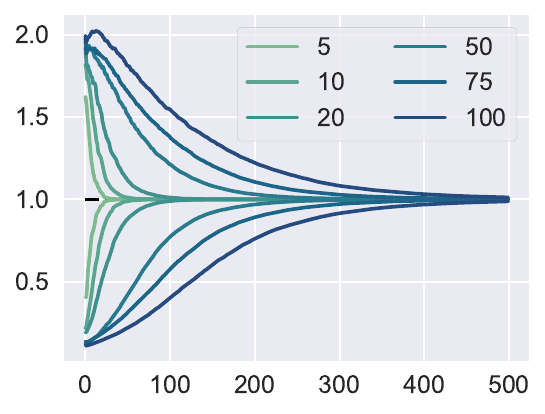} }}
    {\subfigure[$\alpha=0.1$]{\includegraphics[width=0.23\linewidth]{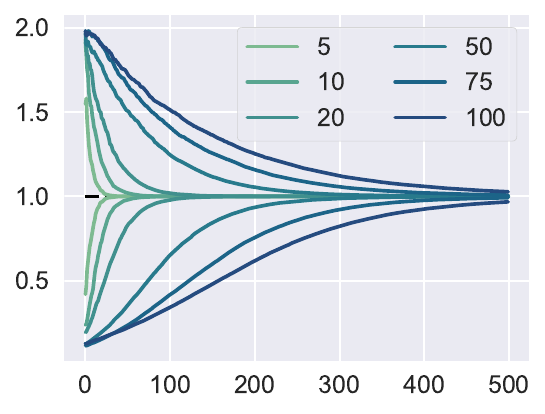} }}
    {\subfigure[$\alpha=0.51$]{\includegraphics[width=0.23\linewidth]{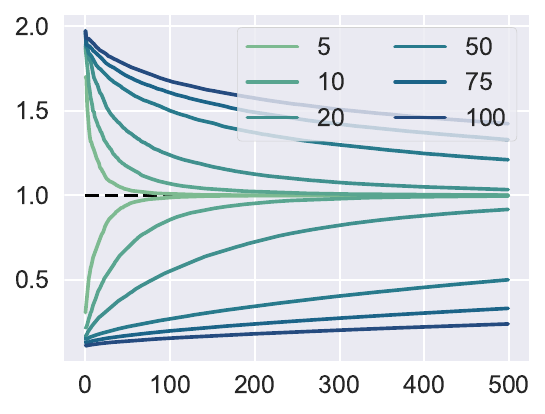} }}
    {\subfigure[$\alpha=0.9$]{\includegraphics[width=0.23\linewidth]{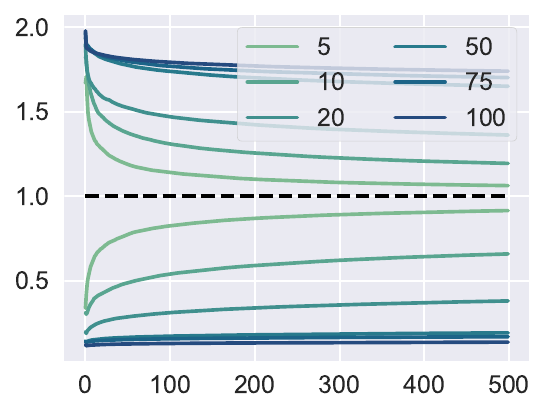} }}
    \caption{Minimum and maximum eigenvalues of $\Sigma_k$ when $\sigma =\mathcal{N}(0,\Sigma)$ and $\mu=\mathcal{N}(0,{\bf I}_d)$}
    \label{curve_eig}
\end{figure}

\paragraph{Eigenvalue control.}
A key ingredient in our proof is to control $\lambda_{\min}(\Sigma_k)$ 
along the trajectory in order to verify 
\Cref{hyp_PL}.
In the isotropic target case $\mu=\mathcal{N}(0,\mathbf{I}_d)$, the theory predicts that once the second-moment bound $\mathrm{M}_2(\sigma_k) \leq \mathrm{M}_2(\mu)$ holds, which happens from the first iteration (\Cref{prop_momentsbounded}), the eigenvalues remain uniformly bounded over $k$ (\Cref{ControlEigSigk2}). 
This behavior can be observed in \Cref{curve_eig}: the extreme eigenvalues settle in a fixed range from the first iteration. 
We emphasize that this behavior is due the moment-matching property inherent to the choice of an \textit{orthonormal basis} $P_{k+1}$ at each iteration. 
Another variant samples a single $\theta_{k+1} \in \mathbb{S}^{d-1}$ per iteration and updates only along that direction. 
The resulting extreme eigenvalues are shown in \Cref{curve_eig_COMPAR}, 
and exhibit larger fluctuations before stabilizing, 
which correlates with slower loss decay.  
The benefit of random orthonormal bases is consistent with recent work on sampling strategies in sliced OT~\citep{sisouk2025a}.

\subsection{Beyond the Gaussian-to-Gaussian Setting}

\Cref{BlobToBlob} considers discrete empirical distributions of $n=500$ samples. 
In each run, the source and target are sampled from a Gaussian mixture with randomly-generated mixture components. 
We plot $SW_2^2(\sigma_k,\mu)$ over iterations for $N=10$ independent runs, across the same dimensions and step-size schedules as in the Gaussian setting. 
We observe the same trends:
the loss decreases for all $d$, convergence slows down as $d$ increases, and smaller values of $\alpha$ typically yield faster convergence. It is worth noting that $\alpha=0.1$ outperforms $\alpha=0$ in our experiments, which illustrates the interest of slice-matching algorithm (where $(\gamma_k)$ is decaying) over IDT (where $\gamma_k = 1$).
We provide additional experiments on empirical measures 
in Appendix~\ref{OtherNumExp}. 
While these discrete settings are not covered by our theory, the empirical convergence suggests that regularity may hold more broadly, despite identified technical issues \citep{tanguy2025properties,vauthier2025properties}.


\begin{figure}[t!]
    \centering
    {\subfigure[$\alpha=0$]{\includegraphics[width=0.2\textwidth]{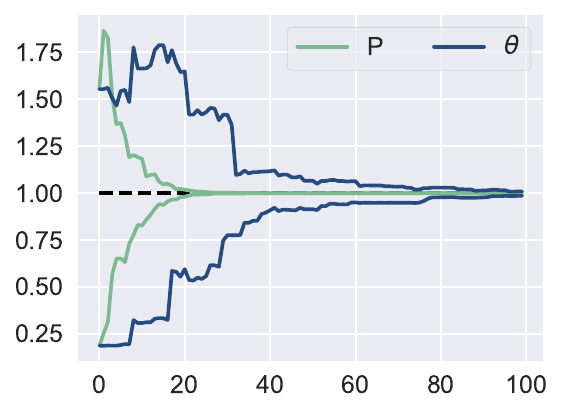} }}
    {\subfigure[$\alpha=0.1$]{\includegraphics[width=0.2\textwidth]{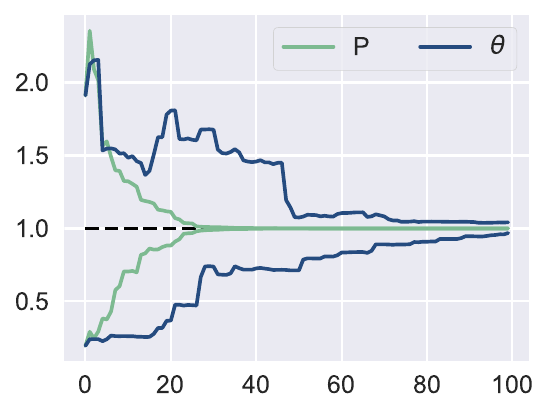} }}
    {\subfigure[$\alpha=0.51$]{\includegraphics[width=0.2\textwidth]{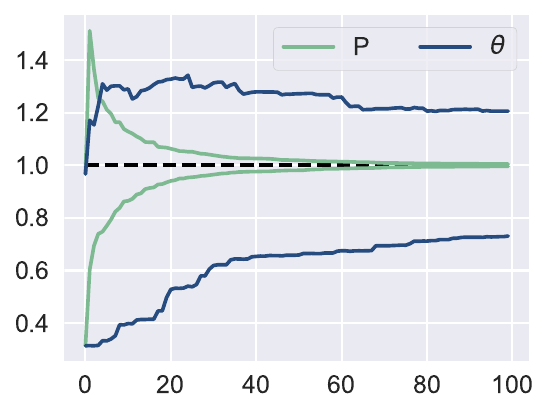} }}
    {\subfigure[$\alpha=0.9$]{\includegraphics[width=0.2\textwidth]{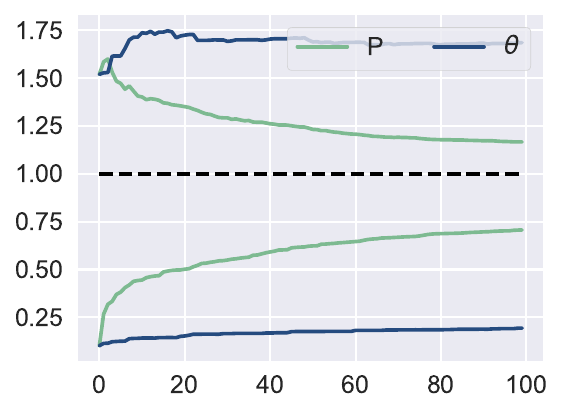} }}
    \caption{Comparison of sampling strategies: single direction $\theta_{k+1}$ or orthonormal basis $P_{k+1}$. We report $\lambda_{\min}(\Sigma_k)$ and $\lambda_{\max}(\Sigma_k)$ with $\sigma =\mathcal{N}(0,\Sigma)$, $\mu=\mathcal{N}(0,{\bf I}_d)$, $d = 5$.}
    \label{curve_eig_COMPAR}
\end{figure}

\begin{figure}[t!]
    \centering
    {\subfigure[Setting]{\includegraphics[width=0.18\textwidth]{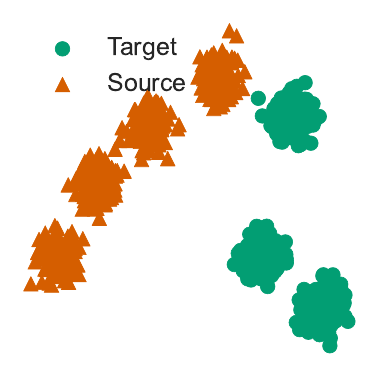} }}
    {\subfigure[$\alpha = 0$]{\includegraphics[width=0.19\textwidth]{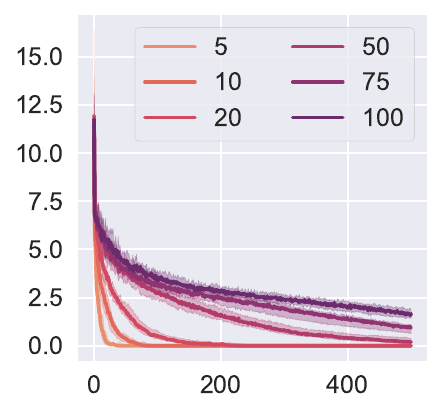} }}
    {\subfigure[$\alpha = 0.1$]{\includegraphics[width=0.183\textwidth]{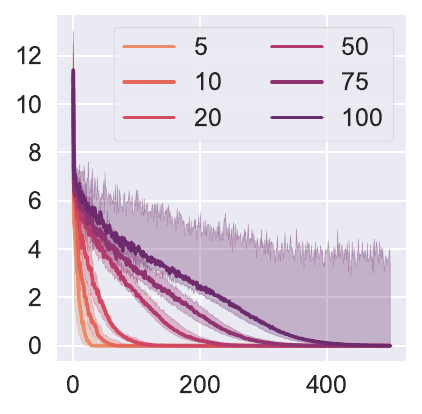}}}
    {\subfigure[$\alpha = 0.51$]{\includegraphics[width=0.19\textwidth]{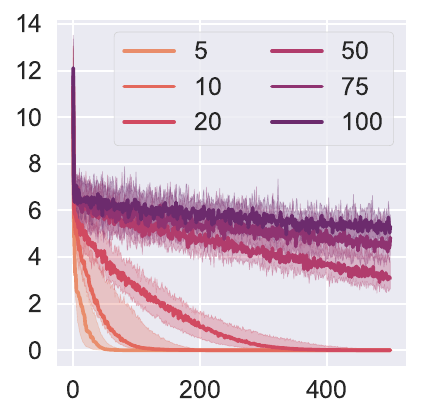}}}
    {\subfigure[$\alpha = 0.9$]{\includegraphics[width=0.183\textwidth]{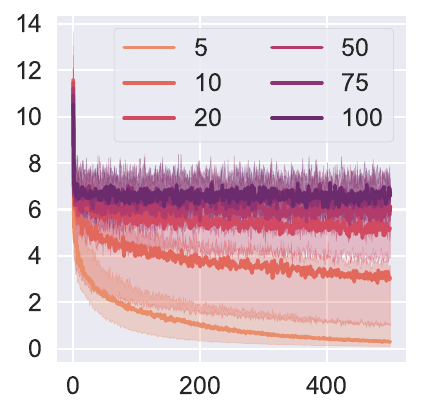}}}
    \caption{Evolution of $SW_2^2(\sigma_k, \mu)$ for discrete source and target distributions. 
    The source and target samples are distributed from Gaussian mixtures.}
    \label{BlobToBlob}
\end{figure}

\section{Conclusion and Perspectives}

We established convergence rates for the slice-matching algorithm 
through
\L{}ojasiewicz-type inequalities for the Sliced-Wasserstein objective.
We show that controlling the associated constants is tractable in the Gaussian (or elliptic) setting when sampling random orthonormal bases. 
A main limitation is that our explicit rate requires an isotropic Gaussian target, similarly to~\cite{cozzi2025}. 
Extending the theory to general Gaussian targets and non-elliptic distributions remains open. A promising direction is to introduce regularization 
(\eg, diffusive terms)
to help maintain regularity along the dynamics~\citep{liutkus2019sliced,tanguy2025properties,chizat2025convergence}. 
Finally, our experiments show faster convergence with 
orthonormal bases of directions and step-size schedules $\gamma_k=1/(k+1)^{\alpha}$ with small $\alpha$. 
This behavior is not explained
by our theorems and may require tools beyond decreasing-step stochastic approximation, for example Markov chains~\citep{dieuleveut2020bridging}.

\bibliography{bib}

\begin{thebibliography}{77}
\providecommand{\natexlab}[1]{#1}
\providecommand{\url}[1]{\texttt{#1}}
\expandafter\ifx\csname urlstyle\endcsname\relax
  \providecommand{\doi}[1]{doi: #1}\else
  \providecommand{\doi}{doi: \begingroup \urlstyle{rm}\Url}\fi

\bibitem[Agueh and Carlier(2011)]{agueh2011barycenters}
Martial Agueh and Guillaume Carlier.
\newblock Barycenters in the wasserstein space.
\newblock \emph{SIAM Journal on Mathematical Analysis}, 43\penalty0
  (2):\penalty0 904--924, 2011.

\bibitem[Albergo et~al.(2025)Albergo, Boffi, and Vanden-Eijnden]{albergo2025}
Michael Albergo, Nicholas~M. Boffi, and Eric Vanden-Eijnden.
\newblock Stochastic interpolants: A unifying framework for flows and
  diffusions.
\newblock \emph{Journal of Machine Learning Research}, 26\penalty0
  (209):\penalty0 1--80, 2025.

\bibitem[Altschuler et~al.(2021)Altschuler, Chewi, Gerber, and
  Stromme]{altschuler2021averaging}
Jason Altschuler, Sinho Chewi, Patrik~R Gerber, and Austin Stromme.
\newblock Averaging on the bures-wasserstein manifold: dimension-free
  convergence of gradient descent.
\newblock \emph{Advances in Neural Information Processing Systems},
  34:\penalty0 22132--22145, 2021.

\bibitem[Ambrosio and Savar{\'e}(2007)]{ambrosio2007gradient}
Luigi Ambrosio and Giuseppe Savar{\'e}.
\newblock Gradient flows of probability measures.
\newblock In \emph{Handbook of differential equations: evolutionary equations},
  volume~3, pages 1--136. Elsevier, 2007.

\bibitem[Attouch et~al.(2010)Attouch, Bolte, Redont, and
  Soubeyran]{attouch2010proximal}
H{\'e}dy Attouch, J{\'e}r{\^o}me Bolte, Patrick Redont, and Antoine Soubeyran.
\newblock Proximal alternating minimization and projection methods for
  nonconvex problems: An approach based on the kurdyka-{\l}ojasiewicz
  inequality.
\newblock \emph{Mathematics of operations research}, 35\penalty0 (2):\penalty0
  438--457, 2010.

\bibitem[Bercu and Bigot(2021)]{BBAOS}
Bernard Bercu and J{\'e}r{\'e}mie Bigot.
\newblock {Asymptotic distribution and convergence rates of stochastic
  algorithms for entropic optimal transportation between probability measures}.
\newblock \emph{The Annals of Statistics}, 49\penalty0 (2):\penalty0 968 --
  987, 2021.
\newblock \doi{10.1214/20-AOS1987}.

\bibitem[Bobkov and Ledoux(2019)]{bobkov2019one}
Sergey Bobkov and Michel Ledoux.
\newblock \emph{One-dimensional empirical measures, order statistics, and
  Kantorovich transport distances}, volume 261.
\newblock American Mathematical Society, 2019.

\bibitem[Bonet et~al.(2024)Bonet, Uscidda, David, Aubin-Frankowski, and
  Korba]{bonetmirror}
Cl{\'e}ment Bonet, Th{\'e}o Uscidda, Adam David, Pierre-Cyril Aubin-Frankowski,
  and Anna Korba.
\newblock Mirror and preconditioned gradient descent in wasserstein space.
\newblock In \emph{The Thirty-eighth Annual Conference on Neural Information
  Processing Systems}, 2024.

\bibitem[Bonneel et~al.(2015)Bonneel, Rabin, Peyr{\'e}, and
  Pfister]{bonneel2015sliced}
Nicolas Bonneel, Julien Rabin, Gabriel Peyr{\'e}, and Hanspeter Pfister.
\newblock Sliced and radon wasserstein barycenters of measures.
\newblock \emph{Journal of Mathematical Imaging and Vision}, 51\penalty0
  (1):\penalty0 22--45, 2015.

\bibitem[Bonnet(2019)]{bonnet2019pontryagin}
Beno{\^\i}t Bonnet.
\newblock A pontryagin maximum principle in wasserstein spaces for constrained
  optimal control problems.
\newblock \emph{ESAIM: Control, Optimisation and Calculus of Variations},
  25:\penalty0 52, 2019.

\bibitem[Bonnotte(2013)]{bonnotte2013unidimensional}
Nicolas Bonnotte.
\newblock \emph{Unidimensional and evolution methods for optimal
  transportation}.
\newblock PhD thesis, Universit{\'e} Paris Sud-Paris XI; Scuola normale
  superiore (Pise, Italie), 2013.

\bibitem[Bottou et~al.(2018)Bottou, Curtis, and
  Nocedal]{bottou2018optimization}
L{\'e}on Bottou, Frank~E Curtis, and Jorge Nocedal.
\newblock Optimization methods for large-scale machine learning.
\newblock \emph{SIAM review}, 60\penalty0 (2):\penalty0 223--311, 2018.

\bibitem[Brenier(1991)]{brenier1991polar}
Yann Brenier.
\newblock Polar factorization and monotone rearrangement of vector-valued
  functions.
\newblock \emph{Communications on pure and applied mathematics}, 44\penalty0
  (4):\penalty0 375--417, 1991.

\bibitem[Caffarelli(1992)]{caffarelli1992regularity}
Luis~A Caffarelli.
\newblock The regularity of mappings with a convex potential.
\newblock \emph{Journal of the American Mathematical Society}, 5\penalty0
  (1):\penalty0 99--104, 1992.

\bibitem[Caffarelli(2000)]{caffarelli2000monotonicity}
Luis~A Caffarelli.
\newblock Monotonicity properties of optimal transportation and the fkg and
  related inequalities.
\newblock \emph{Communications in Mathematical Physics}, 214\penalty0
  (3):\penalty0 547--563, 2000.

\bibitem[Carlier et~al.(2025)Carlier, Figalli, Mérigot, and
  Wang]{carlier2025sharpcomparisonsslicedstandard}
Guillaume Carlier, Alessio Figalli, Quentin Mérigot, and Yi~Wang.
\newblock {Sharp comparisons between sliced and standard $1$-Wasserstein
  distances}, 2025.

\bibitem[Chewi et~al.(2020)Chewi, Maunu, Rigollet, and
  Stromme]{chewi2020gradient}
Sinho Chewi, Tyler Maunu, Philippe Rigollet, and Austin~J Stromme.
\newblock {Gradient descent algorithms for Bures-Wasserstein barycenters}.
\newblock In \emph{Conference on Learning Theory}, 2020.

\bibitem[Chewi et~al.(2024)Chewi, Niles-Weed, and
  Rigollet]{chewi2024statistical}
Sinho Chewi, Jonathan Niles-Weed, and Philippe Rigollet.
\newblock Statistical optimal transport.
\newblock \emph{arXiv:2407.18163}, 3, 2024.

\bibitem[Chizat et~al.(2025)Chizat, Colombo, and
  Fern{\'a}ndez-Real]{chizat2025convergence}
L{\'e}na{\"\i}c Chizat, Maria Colombo, and Xavier Fern{\'a}ndez-Real.
\newblock Convergence of drift-diffusion pdes arising as wasserstein gradient
  flows of convex functions.
\newblock \emph{arXiv:2507.12385}, 2025.

\bibitem[Chung(1954)]{chung1954stochastic}
Kai~Lai Chung.
\newblock On a stochastic approximation method.
\newblock \emph{The Annals of Mathematical Statistics}, pages 463--483, 1954.

\bibitem[Coeurdoux et~al.(2022)Coeurdoux, Dobigeon, and
  Chainais]{coeurdoux2022}
F.~Coeurdoux, N.~Dobigeon, and P.~Chainais.
\newblock Sliced-wasserstein normalizing flows: beyond maximum likelihood
  training.
\newblock In \emph{Proc. European Symposium on Artificial Neural Networks,
  Computational Intelligence and Machine Learning (ESANN)}, Bruges, Belgium,
  Oct. 2022.

\bibitem[Courty et~al.(2016)Courty, Flamary, Tuia, and
  Rakotomamonjy]{courty2016optimal}
Nicolas Courty, R{\'e}mi Flamary, Devis Tuia, and Alain Rakotomamonjy.
\newblock Optimal transport for domain adaptation.
\newblock \emph{IEEE transactions on pattern analysis and machine
  intelligence}, 2016.

\bibitem[Cozzi and Santambrogio(2025)]{cozzi2025}
Giacomo Cozzi and Filippo Santambrogio.
\newblock Long-time asymptotics of the sliced-wasserstein flow.
\newblock \emph{SIAM Journal on Imaging Sciences}, 18\penalty0 (1):\penalty0
  1--19, 2025.
\newblock \doi{10.1137/24M1656414}.

\bibitem[Cuesta and Matr{\'a}n(1989)]{cuesta1989notes}
Juan~Antonio Cuesta and Carlos Matr{\'a}n.
\newblock Notes on the wasserstein metric in hilbert spaces.
\newblock \emph{The Annals of Probability}, pages 1264--1276, 1989.

\bibitem[Dai and Seljak(2021)]{dai21a}
Biwei Dai and Uros Seljak.
\newblock Sliced iterative normalizing flows.
\newblock In Marina Meila and Tong Zhang, editors, \emph{Proceedings of the
  38th International Conference on Machine Learning}, volume 139 of
  \emph{Proceedings of Machine Learning Research}, pages 2352--2364. PMLR,
  18--24 Jul 2021.

\bibitem[Deshpande et~al.(2019)Deshpande, Hu, Sun, Pyrros, Siddiqui, Koyejo,
  Zhao, Forsyth, and Schwing]{deshpande2019max}
I.~Deshpande, Y.-T. Hu, R.~Sun, A.~Pyrros, N.~Siddiqui, S.~Koyejo, Z.~Zhao,
  D.~Forsyth, and A.~Schwing.
\newblock Max-sliced wasserstein distance and its use for gans.
\newblock In \emph{IEEE/CVF CVPR}, 2019.

\bibitem[Dieuleveut et~al.(2020)Dieuleveut, Durmus, and
  Bach]{dieuleveut2020bridging}
Aymeric Dieuleveut, Alain Durmus, and Francis Bach.
\newblock Bridging the gap between constant step size stochastic gradient
  descent and markov chains.
\newblock \emph{The Annals of Statistics}, 48\penalty0 (3):\penalty0
  1348--1382, 2020.

\bibitem[Dossal et~al.(2024)Dossal, Hurault, and
  Papadakis]{dossal2024optimization}
Charles Dossal, Samuel Hurault, and Nicolas Papadakis.
\newblock Optimization with first order algorithms.
\newblock \emph{arXiv:2410.19506}, 2024.

\bibitem[Du et~al.(2023)Du, Li, Pang, Yan, and Lin]{du2023}
Chao Du, Tianbo Li, Tianyu Pang, Shuicheng Yan, and Min Lin.
\newblock Nonparametric generative modeling with conditional
  sliced-{W}asserstein flows.
\newblock In \emph{Proceedings of the 40th International Conference on Machine
  Learning}, volume 202 of \emph{Proceedings of Machine Learning Research},
  pages 8565--8584. PMLR, 23--29 Jul 2023.

\bibitem[Duflo(1996)]{duflo1996algorithmes}
Marie Duflo.
\newblock \emph{Algorithmes stochastiques}, volume~23.
\newblock Springer, 1996.

\bibitem[Garrigos and Gower(2023)]{garrigos2023handbook}
Guillaume Garrigos and Robert~M Gower.
\newblock Handbook of convergence theorems for (stochastic) gradient methods.
\newblock \emph{arXiv:2301.11235}, 2023.

\bibitem[Gelbrich(1990)]{gelbrich1990formula}
Matthias Gelbrich.
\newblock On a formula for the l2 wasserstein metric between measures on
  euclidean and hilbert spaces.
\newblock \emph{Mathematische Nachrichten}, 147\penalty0 (1):\penalty0
  185--203, 1990.

\bibitem[Ghadimi and Lan(2013)]{ghadimi2013stochastic}
Saeed Ghadimi and Guanghui Lan.
\newblock Stochastic first-and zeroth-order methods for nonconvex stochastic
  programming.
\newblock \emph{SIAM journal on optimization}, 23\penalty0 (4):\penalty0
  2341--2368, 2013.

\bibitem[Godichon-Baggioni(2019)]{godichon2019lp}
Antoine Godichon-Baggioni.
\newblock Lp and almost sure rates of convergence of averaged stochastic
  gradient algorithms: locally strongly convex objective.
\newblock \emph{ESAIM: Probability and Statistics}, 23:\penalty0 841--873,
  2019.

\bibitem[Grenioux et~al.(2023)Grenioux, Oliviero~Durmus, Moulines, and
  Gabri\'{e}]{pmlr-v202-grenioux23a}
Louis Grenioux, Alain Oliviero~Durmus, Eric Moulines, and Marylou Gabri\'{e}.
\newblock On sampling with approximate transport maps.
\newblock In \emph{Proceedings of the 40th International Conference on Machine
  Learning}, 2023.

\bibitem[Hütter and Rigollet(2021)]{HutterMinimax2021}
Jan-Christian Hütter and Philippe Rigollet.
\newblock Minimax estimation of smooth optimal transport maps.
\newblock \emph{The Annals of Statistics}, 49\penalty0 (2):\penalty0
  1166--1194, 2021.

\bibitem[Irons et~al.(2022)Irons, Scetbon, Pal, and
  Harchaoui]{irons2022triangular}
Nicholas~J Irons, Meyer Scetbon, Soumik Pal, and Zaid Harchaoui.
\newblock Triangular flows for generative modeling: Statistical consistency,
  smoothness classes, and fast rates.
\newblock In \emph{International Conference on Artificial Intelligence and
  Statistics}, pages 10161--10195. PMLR, 2022.

\bibitem[Jiang et~al.(2024)Jiang, Li, Milzarek, and
  Qiu]{jiang2024generalizedversionchungslemma}
Li~Jiang, Xiao Li, Andre Milzarek, and Junwen Qiu.
\newblock {A Generalized Version of Chung's Lemma and its Applications}, 2024.

\bibitem[Kassraie et~al.(2024)Kassraie, Pooladian, Klein, Thornton, Niles-Weed,
  and Cuturi]{kassraie2024progressive}
Parnian Kassraie, Aram-Alexandre Pooladian, Michal Klein, James Thornton,
  Jonathan Niles-Weed, and Marco Cuturi.
\newblock Progressive entropic optimal transport solvers.
\newblock \emph{Advances in Neural Information Processing Systems},
  37:\penalty0 19561--19590, 2024.

\bibitem[Khaled and Richt{\'a}rik(2023)]{khaled2022better}
Ahmed Khaled and Peter Richt{\'a}rik.
\newblock Better theory for {SGD} in the nonconvex world.
\newblock \emph{Transactions on Machine Learning Research}, 2023.
\newblock ISSN 2835-8856.
\newblock Survey Certification.

\bibitem[Kitagawa and Takatsu(2024)]{kitagawa2024}
Jun Kitagawa and Asuka Takatsu.
\newblock Sliced optimal transport: is it a suitable replacement?, 2024.

\bibitem[Kloeckner(2010)]{kloeckner2010geometric}
Beno{\^\i}t Kloeckner.
\newblock A geometric study of wasserstein spaces: Euclidean spaces.
\newblock \emph{Annali della Scuola Normale Superiore di Pisa-Classe di
  Scienze}, 9\penalty0 (2):\penalty0 297--323, 2010.

\bibitem[Kolouri et~al.(2018)Kolouri, Pope, Martin, and
  Rohde]{kolouri2018slicedae}
Soheil Kolouri, Phillip~E Pope, Charles~E Martin, and Gustavo~K Rohde.
\newblock Sliced wasserstein auto-encoders.
\newblock In \emph{International Conference on Learning Representations}, 2018.

\bibitem[Kurdyka et~al.(2000)Kurdyka, Mostowski, and
  Parusi{\'n}ski]{kurdyka2000proof}
Krzysztof Kurdyka, Tadeusz Mostowski, and Adam Parusi{\'n}ski.
\newblock Proof of the gradient conjecture of r. thom.
\newblock \emph{Annals of Mathematics}, pages 763--792, 2000.

\bibitem[Lanzetti et~al.(2025)Lanzetti, Bolognani, and
  D{\"o}rfler]{lanzetti2025first}
Nicolas Lanzetti, Saverio Bolognani, and Florian D{\"o}rfler.
\newblock First-order conditions for optimization in the wasserstein space.
\newblock \emph{SIAM Journal on Mathematics of Data Science}, 7\penalty0
  (1):\penalty0 274--300, 2025.

\bibitem[Li and Moosm{\"u}ller(2024)]{li2024approximation}
Shiying Li and Caroline Moosm{\"u}ller.
\newblock Approximation properties of slice-matching operators.
\newblock \emph{Sampling Theory, Signal Processing, and Data Analysis},
  22\penalty0 (1):\penalty0 15, 2024.

\bibitem[Li et~al.(2023)Li, Moosmueller, and Wang]{li2023measure}
Shiying Li, Caroline Moosmueller, and Yongzhe Wang.
\newblock Measure transfer via stochastic slicing and matching.
\newblock \emph{arXiv:2307.05705}, 2023.

\bibitem[Liu et~al.(2025)Liu, Martin, Bai, Shahbazi, Thorpe, Aldroubi, and
  Kolouri]{liu2025expected}
Xinran Liu, Rocio~Diaz Martin, Yikun Bai, Ashkan Shahbazi, Matthew Thorpe,
  Akram Aldroubi, and Soheil Kolouri.
\newblock Expected sliced transport plans.
\newblock In \emph{The Thirteenth International Conference on Learning
  Representations}, 2025.

\bibitem[Liutkus et~al.(2019)Liutkus, Simsekli, Majewski, Durmus, and
  St{\"o}ter]{liutkus2019sliced}
Antoine Liutkus, Umut Simsekli, Szymon Majewski, Alain Durmus, and
  Fabian-Robert St{\"o}ter.
\newblock Sliced-wasserstein flows: Nonparametric generative modeling via
  optimal transport and diffusions.
\newblock In \emph{International Conference on machine learning}, 2019.

\bibitem[Ma(2023)]{ma2023absolute}
Jianyu Ma.
\newblock Absolute continuity of wasserstein barycenters on manifolds with a
  lower ricci curvature bound.
\newblock \emph{arXiv:2310.13832}, 2023.

\bibitem[Mahey et~al.(2023)Mahey, Chapel, Gasso, Bonet, and
  Courty]{mahey2023fast}
Guillaume Mahey, Laetitia Chapel, Gilles Gasso, Cl{\'e}ment Bonet, and Nicolas
  Courty.
\newblock {Fast Optimal Transport through Sliced Generalized Wasserstein
  Geodesics}.
\newblock In \emph{Thirty-seventh Conference on Neural Information Processing
  Systems}, 2023.

\bibitem[Manole et~al.(2022)Manole, Balakrishnan, and Wasserman]{manole2022}
Tudor Manole, Sivaraman Balakrishnan, and Larry Wasserman.
\newblock {Minimax confidence intervals for the Sliced Wasserstein distance}.
\newblock \emph{Electronic Journal of Statistics}, 16\penalty0 (1):\penalty0
  2252 -- 2345, 2022.
\newblock \doi{10.1214/22-EJS2001}.

\bibitem[Marzouk et~al.(2016)Marzouk, Moselhy, Parno, and
  Spantini]{marzouk2016introduction}
Youssef Marzouk, Tarek Moselhy, Matthew Parno, and Alessio Spantini.
\newblock \emph{Sampling via Measure Transport: An Introduction}.
\newblock Springer International Publishing, 2016.

\bibitem[McCann(1997)]{mccann1997convexity}
Robert~J McCann.
\newblock A convexity principle for interacting gases.
\newblock \emph{Advances in mathematics}, 128\penalty0 (1):\penalty0 153--179,
  1997.

\bibitem[Moulines and Bach(2011)]{moulines2011non}
Eric Moulines and Francis Bach.
\newblock Non-asymptotic analysis of stochastic approximation algorithms for
  machine learning.
\newblock \emph{Advances in neural information processing systems}, 24, 2011.

\bibitem[Nadjahi et~al.(2019)Nadjahi, Durmus, Simsekli, and
  Badeau]{nadjahi2019asymptotic}
Kimia Nadjahi, Alain Durmus, Umut Simsekli, and Roland Badeau.
\newblock Asymptotic guarantees for learning generative models with the
  sliced-wasserstein distance.
\newblock \emph{Advances in Neural Information Processing Systems}, 32, 2019.

\bibitem[Nadjahi et~al.(2020)Nadjahi, Durmus, Chizat, Kolouri, Shahrampour, and
  Simsekli]{nadjahi2020statistical}
Kimia Nadjahi, Alain Durmus, L{\'e}na{\"\i}c Chizat, Soheil Kolouri, Shahin
  Shahrampour, and Umut Simsekli.
\newblock Statistical and topological properties of sliced probability
  divergences.
\newblock \emph{Advances in Neural Information Processing Systems},
  33:\penalty0 20802--20812, 2020.

\bibitem[Ostrowski(1959)]{ostrowski1959quantitative}
Alexander~M Ostrowski.
\newblock A quantitative formulation of sylvester's law of inertia.
\newblock \emph{Proceedings of the National Academy of Sciences}, 45\penalty0
  (5):\penalty0 740--744, 1959.

\bibitem[Park and Slep{\v{c}}ev(2025)]{park2025geometry}
Sangmin Park and Dejan Slep{\v{c}}ev.
\newblock Geometry and analytic properties of the sliced wasserstein space.
\newblock \emph{Journal of Functional Analysis}, 289\penalty0 (7):\penalty0
  110975, 2025.

\bibitem[Peyr{\'e} et~al.(2019)Peyr{\'e}, Cuturi,
  et~al.]{peyre2019computational}
Gabriel Peyr{\'e}, Marco Cuturi, et~al.
\newblock Computational optimal transport: With applications to data science.
\newblock \emph{Foundations and Trends{\textregistered} in Machine Learning},
  11\penalty0 (5-6):\penalty0 355--607, 2019.

\bibitem[Piti{\'e} et~al.(2007)Piti{\'e}, Kokaram, and
  Dahyot]{pitie2007automated}
Fran{\c{c}}ois Piti{\'e}, Anil~C Kokaram, and Rozenn Dahyot.
\newblock Automated colour grading using colour distribution transfer.
\newblock \emph{Computer Vision and Image Understanding}, 107\penalty0
  (1-2):\penalty0 123--137, 2007.

\bibitem[Rabin et~al.(2011)Rabin, Peyr{\'e}, Delon, and
  Bernot]{rabin2011wasserstein}
Julien Rabin, Gabriel Peyr{\'e}, Julie Delon, and Marc Bernot.
\newblock Wasserstein barycenter and its application to texture mixing.
\newblock In \emph{International conference on scale space and variational
  methods in computer vision}, pages 435--446. Springer, 2011.

\bibitem[Rabin et~al.(2012)Rabin, Peyr{\'e}, Delon, and Bernot]{rabin2012}
Julien Rabin, Gabriel Peyr{\'e}, Julie Delon, and Marc Bernot.
\newblock Wasserstein barycenter and its application to texture mixing.
\newblock In Alfred~M. Bruckstein, Bart~M. ter Haar~Romeny, Alexander~M.
  Bronstein, and Michael~M. Bronstein, editors, \emph{Scale Space and
  Variational Methods in Computer Vision}, pages 435--446, Berlin, Heidelberg,
  2012. Springer Berlin Heidelberg.
\newblock ISBN 978-3-642-24785-9.

\bibitem[Robbins and Siegmund(1971)]{robbins1971convergence}
Herbert Robbins and David Siegmund.
\newblock A convergence theorem for non negative almost supermartingales and
  some applications.
\newblock In \emph{Optimizing methods in statistics}, pages 233--257. Elsevier,
  1971.

\bibitem[Rockafellar(1970)]{rockafellar-1970a}
R.~Tyrrell Rockafellar.
\newblock \emph{Convex analysis}.
\newblock Princeton University Press, 1970.

\bibitem[Santambrogio(2015)]{santambrogio2015optimal}
Filippo Santambrogio.
\newblock \emph{Optimal transport for applied mathematicians}, volume~87.
\newblock Springer, 2015.

\bibitem[Sisouk et~al.(2025)Sisouk, Delon, and Tierny]{sisouk2025a}
Keanu Sisouk, Julie Delon, and Julien Tierny.
\newblock {A User's Guide to Sampling Strategies for Sliced Optimal Transport}.
\newblock \emph{Transactions on Machine Learning Research}, 2025.
\newblock ISSN 2835-8856.
\newblock Survey Certification.

\bibitem[Song et~al.(2021)Song, Sohl-Dickstein, Kingma, Kumar, Ermon, and
  Poole]{song2021scorebased}
Yang Song, Jascha Sohl-Dickstein, Diederik~P Kingma, Abhishek Kumar, Stefano
  Ermon, and Ben Poole.
\newblock Score-based generative modeling through stochastic differential
  equations.
\newblock In \emph{International Conference on Learning Representations}, 2021.

\bibitem[Tanguy(2023)]{tanguy2023convergence}
Eloi Tanguy.
\newblock Convergence of {SGD} for training neural networks with sliced
  wasserstein losses.
\newblock \emph{Transactions on Machine Learning Research}, 2023.
\newblock ISSN 2835-8856.

\bibitem[Tanguy et~al.(2024)Tanguy, Flamary, and
  Delon]{tanguy2024reconstructing}
Eloi Tanguy, R{\'e}mi Flamary, and Julie Delon.
\newblock Reconstructing discrete measures from projections. consequences on
  the empirical sliced wasserstein distance.
\newblock \emph{Comptes Rendus. Math{\'e}matique}, 362\penalty0 (G10):\penalty0
  1121--1129, 2024.

\bibitem[Tanguy et~al.(2025)Tanguy, Flamary, and Delon]{tanguy2025properties}
Eloi Tanguy, R{\'e}mi Flamary, and Julie Delon.
\newblock {Properties of discrete sliced Wasserstein losses}.
\newblock \emph{Mathematics of Computation}, 94\penalty0 (353):\penalty0
  1411--1465, 2025.

\bibitem[Vauthier et~al.(2025)Vauthier, M{\'e}rigot, and
  Korba]{vauthier2025properties}
Christophe Vauthier, Quentin M{\'e}rigot, and Anna Korba.
\newblock {Properties of Wasserstein Gradient Flows for the Sliced-Wasserstein
  Distance}.
\newblock \emph{arXiv:2502.06525}, 2025.

\bibitem[Villani(2008)]{villani2008optimal}
C{\'e}dric Villani.
\newblock \emph{Optimal transport: old and new}, volume 338.
\newblock Springer, 2008.

\bibitem[Wang and Marzouk(2022)]{wang2022minimax}
Sven Wang and Youssef Marzouk.
\newblock On minimax density estimation via measure transport.
\newblock \emph{arXiv:2207.10231}, 2022.

\bibitem[Wiens(1992)]{wiens1992moments}
Douglas~P Wiens.
\newblock On moments of quadratic forms in non-spherically distributed
  variables.
\newblock \emph{Statistics}, 23\penalty0 (3):\penalty0 265--270, 1992.

\bibitem[Wu et~al.(2019)Wu, Huang, Acharya, Li, Thoma, Paudel, and
  Gool]{wu2019sliced}
Jiqing Wu, Zhiwu Huang, Dinesh Acharya, Wen Li, Janine Thoma, Danda~Pani
  Paudel, and Luc~Van Gool.
\newblock Sliced wasserstein generative models.
\newblock In \emph{Proceedings of the IEEE/CVF Conference on Computer Vision
  and Pattern Recognition}, 2019.

\bibitem[Zhou(2018)]{zhou2018fenchel}
Xingyu Zhou.
\newblock On the fenchel duality between strong convexity and lipschitz
  continuous gradient.
\newblock \emph{arXiv:1803.06573}, 2018.

\end{thebibliography}

\newpage 
\appendix

\section{Reminders on Wasserstein space}\label{app:remindWassSpace}

This appendix gathers existing results useful for optimization over the space of probability distributions. For further details, we refer the interesting reader to the classical references \cite{ambrosio2007gradient,santambrogio2015optimal}. 
First, recall that, for $\psi^c(y) = \inf_x \{ \frac{1}{2}\Vert x-y\Vert^2 - \psi(x)\}$ the $c$-transform of $\psi$, 
the dual of Kantorovich OT problem writes
\begin{equation}\label{dualKanto}
    W_2^2(\alpha,\beta) = \sup_{\psi\in L^1(\alpha)} \int \psi \rmd\alpha + \int \psi^c \rmd\beta.
\end{equation}
The solution of the latter is called the Kantorovich potential, and it is unique (up to translations) under finiteness of second-order moments, with $\alpha$ giving no mass to $d-1$ surfaces \citep[Theorem 1.22]{santambrogio2015optimal}. 

\subsection{Curves and convexity in Wasserstein space}\label{app:curves}

The Wasserstein space $(\mathcal{P}_2(\mathbb{R}^d),W_2)$ is the space of square-integrable probability distributions endowed with the Wasserstein distance $W_2$. 
A first way to construct an absolutely continuous curve between two measures $\sigma_0$ and $\sigma_1$ is the flat interpolation, given, for $t\in[0,1]$, by
\begin{equation}\label{FlatInterpolCurve}
\sigma_t = (1-t)\sigma_0 + t \sigma_1. 
\end{equation}
This convex combination between densities ignores the geometry induced by the Wasserstein distance.  
In contrast, denoting by $T_{\sigma_0}^{\sigma_1}$ the OT map from $\sigma_0$ to $\sigma_1$, another interpolation is given by
\begin{equation}\label{GeodesicCurve}
\sigma_t = \big( (1-t) \Id + t T_{\sigma_0}^{\sigma_1}\big)_\sharp \sigma_0. 
\end{equation}
Due to the fact that $(1-t) \Id + t T_{\sigma_0}^{\sigma_1}$ is the gradient of a convex function, it is the solution of Monge OT problem \citep{brenier1991polar,cuesta1989notes}. 
Hence, $\sigma_t$ corresponds to the shortest path between $\sigma_0$ and $\sigma_1$, in the sense that
$$
\forall \, 0 \leq s \leq t \leq 1, \hspace{1cm} W_2 (\sigma_s, \sigma_t) = (t-s) W_2 (\sigma_0,\sigma_1). 
$$
While \eqref{FlatInterpolCurve} corresponds to a mixture model between $\sigma_0$ and $\sigma_1$, the interpolant \eqref{GeodesicCurve} is more of a barycenter \citep{agueh2011barycenters,rabin2011wasserstein} and it is a building block for gradient flows in the Wasserstein space \citep{ambrosio2007gradient}. 
Interestingly enough, $W_2^2 (\cdot,\sigma)$ is strictly convex along \eqref{FlatInterpolCurve} as soon as $\sigma$ is absolutely continuous \citep[Proposition 7.19]{santambrogio2015optimal}. 
It is not hard to see that the same property holds for the Sliced-Wasserstein distance, with arguments reminiscent to the ones of \citet[Proposition 2.10]{ma2023absolute} for Wasserstein barycenters. 
Such convexity along \eqref{FlatInterpolCurve} must be understood with respect to the $2$-norm between densities. 
The analog along \eqref{GeodesicCurve}, with respect to the Wasserstein distance, writes as follows. 

\begin{definition}
    $\mathcal{F}$ is geodesically $\alpha$-convex if, for all $\sigma_0,\sigma_1\in \mathcal{P}_2(\mathbb{R}^d)$ and $\sigma_t = ( (1-t) \Id + t T_{\sigma_0}^{\sigma_1})_\sharp \sigma_0$,
    $$
    \mathcal{F}(\sigma_t) \leq (1-t) \mathcal{F}(\sigma_0) + t \mathcal{F}(\sigma_1) - \frac{\alpha}{2} t (1-t) W_2^2(\sigma_0,\sigma_1). 
    $$
\end{definition}

Unfortunately, the reverse inequality holds for $ W_2^2(\cdot,\sigma)$ in general dimension \citep[Theorem 7.3.2]{ambrosio2007gradient}, and a fortiori for the Sliced-Wasserstein distance up to integration over the projection directions \citep[Appendix A.5]{vauthier2025properties}. 
These facts are discussed in \Cref{smoothness}.  

The situation is very different in dimension $d=1$, due to the particular properties of $(\mathcal{P}_2(\mathbb{R}),W_2)$. 
In this setting, the composition of OT maps preserves their monotonicity (hence the optimality) and $W_2$ rewrites with $Q_0,Q_1$ the quantile functions of $\sigma_0,\sigma_1$:
\begin{equation}\label{UnivWassDist}
W_2^2 (\sigma_0, \sigma_1) = 
\int_0^1 \Vert Q_0(t) - Q_1 (t)\Vert^2 dt
\end{equation}
or, equivalently, $W_2^2 (\sigma_0, \sigma_1) = \Vert T_\rho^{\sigma_0} - T_\rho^{\sigma_1} \Vert_{\rho}^2$ for any pivot measure $\rho\in \mathcal{P}_{2,ac}(\mathbb{R}^d)$. 
As a byproduct, the geodesics in \eqref{GeodesicCurve} coincide with the generalized geodesics $\sigma_t = \big( (1-t) T_{\rho}^{\sigma_0} + t T_{\rho}^{\sigma_1}\big)_\sharp \rho$, and one can find in \citet[Chapter 9]{ambrosio2007gradient} that they verify the generalized parallelogram rule
\begin{equation}\label{parallelogramrule}
    W_2^2(\sigma_t, \sigma) = (1-t) W_2^2(\sigma_0,\sigma) + t W_2^2(\sigma_1,\sigma) - t(1-t) W_2^2(\sigma_0,\sigma_1).
\end{equation}
This can be easily verified by expanding the square in  $W_2^2(\sigma_t, \sigma)$ via \eqref{UnivWassDist} and using the tricks $t^2 = t - t(1-t)$ and $(1-t)^2 = (1-t) - t(1-t)$, as in 
\citet[Proposition 4.1]{kloeckner2010geometric}. 
Next, we turn to differentiation along geodesics. 
Unfortunately, \eqref{parallelogramrule} does not imply the same parallelogram identity for the Sliced-Wasserstein distance, 
as it would require to identify a path $\sigma_t$ in $\mathbb{R}^d$ along the
map $\overline{T}:x \mapsto \int_\theta T_{\sigma^\theta}^{\mu^\theta}(x) d\mathcal{U}(\theta)$ with all projected generalized geodesics $\sigma_t^\theta$, which is not true.   

\subsection{Differentiation along geodesics}\label{App:differentiation}

We borrow the differential structure of $(\mathcal{P}_{2}(\mathbb{R}^d),W_2)$ as described in \eg, \cite{bonnet2019pontryagin,bonetmirror,lanzetti2025first}. 
The tangent space of $\mathcal{P}_{2}(\mathbb{R}^d)$ at $\sigma$ is defined by 
$$
\mathcal{T}_\sigma = \overline{\{ \nabla \psi : \psi \in C^\infty_c(\RR^d) \}},
$$
where the closure is taken with respect to $\Lt(\sigma)$, 
the set of $\sigma$-square integrable functions from $\RR^d$ to $\RR^d$, 
and where $C^\infty_c(\RR^d)$ is the set of infinitely differentiable functions with compact support. 
Consider $\mathcal{F}: \mathcal{P}_{2}(\mathbb{R}^d) \rightarrow \RR$. 
At any $\sigma\in \mathcal{P}_{2}(\mathbb{R}^d)$ such that $\mathcal{F}(\sigma)<+\infty$, the Wasserstein gradient $\nabla_{W_2}\mathcal{F}(\sigma)$ is the unique vector in $\mathcal{T}_\sigma$ verifying, for any $\mu \in \mathcal{P}_{2}(\mathbb{R}^d)$ and any optimal coupling $\gamma \in \Pi(\sigma,\mu)$\footnote{The set of couplings between $\sigma$ and $\mu$ is $\Pi(\sigma,\mu)= \{ \pi \in \mathcal{P}_2(\RR^d \times \RR^d): \pi(A\times \RR^d)= \sigma(A) \, , \, \pi(\RR^d\times B)= \mu(B) \}$.}, 
\begin{equation}\label{WassGradF}
    \mathcal{F}(\mu) = \mathcal{F}(\sigma) + \int \langle \nabla_{W_2}\mathcal{F}(\sigma)(x), y-x \rangle \rmd\gamma(x,y) + o(W_2(\sigma,\mu)). 
\end{equation}
One way to compute the Wasserstein gradient is by taking   
$
\nabla_{W_2}\mathcal{F}(\sigma) = \nabla \frac{\delta\mathcal{F}}{\delta \sigma} (\sigma).  
$
for $\frac{\delta\mathcal{F}}{\delta \sigma} (\sigma)$ the first variation 
\citep[Definition 7.12]{santambrogio2015optimal} defined as follows. 
Firstly, a measure $\rho\in \mathcal{P}_{2,ac}(\mathbb{R}^d)$ is \textit{regular} for $\mathcal{F}$ if, for every $\overline{\rho}\in \mathcal{P}_{2,ac}(\mathbb{R}^d)$ with $L^\infty$ density and compact support, $\mathcal{F}((1-t)\rho +t \overline{\rho}) <+\infty$ for every $t \in [0,1]$. 
With this at hand, if $\rho$ is regular for $\mathcal{F}$, the first variation $\frac{\delta\mathcal{F}}{\delta \rho} (\rho)$ verifies
\begin{equation*}
\frac{\rmd}{\rmd t} \mathcal{F}(\rho + t \xi)_{\vert t = 0} = \lim_{t\rightarrow 0} \frac{\mathcal{F}(\rho + t \xi) - \mathcal{F}(\rho)}{t} = \int \frac{\delta\mathcal{F}}{\delta \rho} (\rho) \rmd \xi, 
\end{equation*}
for all $\xi= \overline{\rho} - \rho$ with $\overline{\rho}\in \mathcal{P}_{2,ac}(\mathbb{R}^d)$ with $L^\infty$ density and compact support.    
The following useful remark is taken from \citep[Remark 7.14]{santambrogio2015optimal}. 
\begin{remark}\label{rhoreguler}
For $\sigma \in \mathcal{P}_{2,ac}(\mathbb{R}^d)$ and $\mathcal{F}:\rho \mapsto W_2^2(\rho,\sigma)$, any $\rho\in \mathcal{P}_{2,ac}(\mathbb{R}^d)$ is regular if and only if $\mathcal{F}(\rho)<+\infty$. 
Indeed, for every $\overline{\rho}\in \mathcal{P}_{2,ac}(\mathbb{R}^d)$, $\mathcal{F}((1-t)\rho +t \overline{\rho}) \leq (1-t)\mathcal{F}(\rho) +t \mathcal{F}(\overline{\rho})$ by strict convexity~\citep[Proposition 7.19]{santambrogio2015optimal}. 
Hence, as soon as $\overline{\rho}$ is compactly supported, $\mathcal{F}(\overline{\rho})<+\infty$ and $\rho$ is regular if $\mathcal{F}(\rho)<+\infty$. The reciprocal is immediate by taking $t = 0$ in the definition of a regular measure. 
\end{remark}

When considering the Wasserstein distance $\sigma \mapsto W_2^2 (\sigma,\mu)$, the first variation is $\psi$, the Kantorovich potential that is solution of \eqref{dualKanto} \citep[Proposition 7.17,][]{santambrogio2015optimal}, and a similar statement holds for $\sigma \mapsto SW_2^2 (\sigma,\mu)$ \citep{cozzi2025}.
We discuss this in \Cref{PropGradients}. 

\section{Proofs of Sections \ref{secMathFramework} and \ref{CvgceGradientSmooth}}\label{appPropertiesLoss}

In this section, we detail the properties of the functional to be minimized. 
We discuss differentiability and critical points, before turning to smoothness and boundedness of the gradient, the latter being a byproduct of boundedness of moments along the iterations. 

\subsection{Differentiability, critical points}
\label{app:differentiability_critical_points}

The next proposition describes Wasserstein gradients of our sliced objective, as previously provided in \cite{bonnotte2013unidimensional,cozzi2025}. 
We also detail simple properties of the gradient norm, that are important with the purpose of SGD.

\begin{proposition}\label{PropGradients}
    Given that $\mu \in \mathcal{P}_{2,ac}(\mathbb{R}^d)$ is compactly supported, 
    the Wasserstein gradients of $\mathscr{F}(\cdot,\theta)$ and $\mathscr{F}$ at any $\sigma \in \mathcal{P}_{2,ac}(\mathbb{R}^d)$ compactly supported are given by  
    $$
    \nabla_{W_2} \mathscr{F}(\sigma,P) = \Id - T_{\sigma,P} 
    \hspace{1cm} \mbox{and} \hspace{1cm} 
    \nabla_{W_2} \mathscr{F}(\sigma) = d\int \theta(\Id - {T_{\sigma^\theta}^{\mu^\theta}})\circ \pi_\theta \, \rmd\mathcal{U}(\theta),
    $$
    where ${T_{\sigma^\theta}^{\mu^\theta}}$ denotes the one-dimensional OT map pushing $\sigma^\theta$ to $\mu^\theta$. 
    It follows that 
    $$
    \Vert \nabla_{W_2} \mathscr{F}(\sigma,P) \Vert_\sigma^2 
    = \sum_{\ell=1}^d W_2^2 (\sigma^{\theta_\ell},\mu^{\theta_\ell})
    \hspace{1cm} \mbox{and} \hspace{1cm} 
    \Vert \nabla_{W_2} \mathscr{F}(\sigma) \Vert_\sigma^2 \leq d SW_2^2(\sigma,\mu),
    $$
    that is
    $$
    \Vert \nabla_{W_2} \mathscr{F}(\sigma,P) \Vert_\sigma^2 =  2\mathscr{F}(\sigma,P)
    \hspace{1cm} \mbox{and} \hspace{1cm} 
    \Vert \nabla_{W_2} \mathscr{F}(\sigma) \Vert^2_{\sigma} \leq  2 \mathscr{F}(\sigma). 
    $$
    In fact, this can be refined in the following decomposition, for $\overline{T}_{\sigma}= \mathbb{E}_P[T_{\sigma,P}]$, \footnote{Equivalently, 
        $
        dSW_2^2(\sigma,\mu) = \Vert \Id - \overline{T}_{\sigma} \Vert^2_\sigma + \mathbb{E}_P \big[ \Vert \overline{T}_{\sigma} - T_{\sigma,P}\Vert^2_\sigma \big].
        $}
    \begin{equation*}
    2\mathscr{F}(\sigma) = \Vert \nabla_{W_2} \mathscr{F}(\sigma)\Vert^2_\sigma + \mathbb{E}_P[\Vert \overline{T}_{\sigma} - T_{\sigma,P} \Vert^2_\sigma]. 
    \end{equation*}
\end{proposition}
\begin{proof}
 For a given basis $P$, the Wasserstein gradient of $\sigma \mapsto \mathscr{F}(\sigma, P)$ is given by the euclidean gradient of its first variation, i.e., $\nabla_{W_2} \mathscr{F}(\sigma,P) = \nabla\frac{\delta \mathscr{F}}{\delta \sigma}( \sigma,P)$, as recalled in ~\Cref{App:differentiation}. 
From \citet[Proposition 7.17,][]{santambrogio2015optimal}, the first variation of $\sigma \mapsto W_2^2(\sigma^\theta,\mu^\theta)$ is given by $\varphi_\theta(\langle \cdot,\theta \rangle)$, for $\varphi_\theta$ the first Kantorovich potential $\varphi_\theta$ for the OT problem from $\sigma^\theta$ to $\mu^\theta$ (assuming compactness of the underlying supports). 
Thus, the first variation of 
$
\mathscr{F}(\sigma,P) = \sum_{\ell=1}^d  W_2^2(\sigma^{\theta_\ell},\mu^{\theta_\ell})$
is 
$
\sum_{\ell=1}^d \varphi_{\theta_\ell}(\langle \cdot, \theta_\ell\rangle),
$
and the euclidean gradient is given through $\nabla \varphi_{\theta_\ell}(\langle x, {\theta_\ell}\rangle) = {\theta_\ell} (x^\top {\theta_\ell} - {T_{\sigma^{\theta_\ell}}^{\mu^{\theta_\ell}}}(x^\top {\theta_\ell})) $. 
The first result directly follows as
$$
\nabla_{W_2}\mathscr{F}(\sigma,P) = \sum_{\ell=1}^d {\theta_\ell} (\Id - {T_{\sigma^{\theta_\ell}}^{\mu^{\theta_\ell}}})\circ \pi_{\theta_\ell}  = \Id - T_{\sigma,P}.
$$
Regarding the integrated version over the directions, the first variation of $\mathscr{F}(\sigma) = d SW_2^2(\sigma,\mu)$ is 
\begin{equation}\label{FirstVarL}
x \mapsto d \int \varphi_\theta (x^\top \theta) \rmd\mathcal{U}(\theta), 
\end{equation}
as stated in \citet{cozzi2025}. 
The detail of this calculus requires interchanging a limit and an integral, because, by definition,  
$$
\int \frac{\delta \mathscr{F}}{\delta \sigma}(\sigma) \rmd\xi= \lim_{t\rightarrow 0} \frac{\mathscr{F}(\sigma + t\xi) - \mathscr{F}(\sigma)}{t} 
=
\lim_{t\rightarrow 0}
\mathbb{E}_P  \frac{\mathscr{F}(\sigma + t\xi,P) - \mathscr{F}(\sigma,P)}{t} ,
$$
for all $\xi= \overline{\rho} - \rho$ with $\overline{\rho}\in \mathcal{P}_{2,ac}(\mathbb{R}^d)$ with $L^\infty$ density and compact support. 
Under compact assumptions, this can be treated as in the last step of the proof of \citet[Proposition 7.17]{santambrogio2015optimal}.  
A direct consequence of \eqref{FirstVarL} is that $\nabla_{W_2} \mathscr{F}(\sigma) = d\int \theta(\Id - {T_{\sigma^\theta}^{\mu^\theta}})\circ \pi_\theta \, d\mathcal{U}(\theta)$. Now, 
$$
\Vert \nabla_{W_2} \mathscr{F}(\sigma,P) \Vert_\sigma^2 
=  \sum_{\ell=1}^d \Vert \langle \cdot, {\theta_\ell} \rangle - {T_{\sigma^{\theta_\ell}}^{\mu^{\theta_\ell}}}(\langle \cdot, {\theta_\ell} \rangle)  \Vert_\sigma^2 
    = \sum_{\ell=1}^d W_2^2 (\sigma^{\theta_\ell},\mu^{\theta_\ell}) = 2\mathscr{F}(\sigma,P), 
$$
so that
\begin{equation}\label{ExpectedNoise}
    \mathbb{E}_P \big( \Vert \nabla_{W_2} \mathscr{F}(\sigma,P ) \Vert^2_\sigma \big) 
    = \mathbb{E}_P \big( 2  \mathscr{F}(\sigma,P) \big) 
    = 2 \mathscr{F}(\sigma).
    \end{equation}
Also, Jensen's inequality implies
\begin{equation*}
\Vert \nabla_{W_2} \mathscr{F}(\sigma) \Vert_\sigma^2 
=
\Vert \mathbb{E}_P \nabla_{W_2} \mathscr{F}(\sigma,P) \Vert_\sigma^2 
\leq \mathbb{E}_P\Vert \nabla_{W_2} \mathscr{F}(\sigma,P) \Vert_\sigma^2 
\leq 2 \mathscr{F}(\sigma).
\end{equation*}
Finally, recall the decomposition
    $$
    \nabla_{W_2} \mathscr{F}(\sigma,P) = \nabla_{W_2}\mathscr{F}(\sigma) + \overline{T}_{\sigma} - T_{\sigma,P}. 
    $$
    Taking the square norm, developing the square and using that $\mathbb{E}_P[ \overline{T}_{\sigma} - T_{\sigma,P}]=0$, one obtains that 
        \begin{equation*}
            \mathbb{E}_P \big( \Vert \nabla_{W_2} \mathscr{F}(\sigma,P ) \Vert^2_\sigma \big)  =
              \Vert \nabla_{W_2} \mathscr{F}(\sigma)\Vert^2_\sigma + \mathbb{E}_P \Vert \overline{T}_{\sigma} - T_{\sigma,P} \Vert^2_\sigma. 
        \end{equation*}
        Combining the above with \eqref{ExpectedNoise},         
    provides the desired decomposition.
\end{proof}

\begin{remark}
    Without compacity, Proposition \ref{PropGradients} does not hold, but 
one can still define $\nabla_{W_2}\mathscr{F}(\sigma)$ directly through $\nabla_{W_2} \mathscr{F}(\sigma) = \int \theta(\Id - {T_{\sigma^\theta}^{\mu^\theta}})\circ \pi_\theta \, d\mathcal{U}(\theta)$. 
In this case, $\nabla_{W_2}\mathscr{F}(\sigma)$ belongs to the subdifferential of $\mathscr{F}(\sigma)$ \citep[Proposition 4.7(b),][]{vauthier2025properties}. 
\end{remark}

\noindent \citet{vauthier2025properties} describe different possible notions of critical points, including the following.  

\begin{definition}[Definition 4.2 from \cite{vauthier2025properties}]
    A measure $\sigma$ is a barycentric Lagrangian critical point for $SW_2^2(\cdot,\mu)$ if, 
    $$
    \frac{1}{d}x = \int_{\mathbb{S}^{d-1}} \theta {T_{\sigma^\theta}^{\mu^\theta}} (x^\top \theta)  \rmd\mathcal{U}(\theta) \hspace{1cm} \mbox{for }\sigma\mbox{-a.e. } x\,,
    $$
    where ${T_{\sigma^\theta}^{\mu^\theta}}$ corresponds to the OT map from $\sigma^\theta$ to $\mu^\theta$, the pushforward measures of $\sigma$ and $\mu$ by $x\mapsto \theta^\top x$. 
\end{definition}

With our notations, a critical point of $\sigma \mapsto \mathscr{F}(\sigma)$ verifies $\Vert \nabla_{W_2}\mathscr{F}(\sigma)\Vert_\sigma=0$, hence 
$$
\int \Vert \frac{x}{d} - \int \theta T_{\sigma^\theta }(x^\top \theta) d\mathcal{U}(\theta)\Vert^2 \rmd\sigma(x)=0,
$$
and it is a barycentric Lagrangian critical point for $SW_2^2(\cdot,\mu)$\footnote{One might note that $\int \theta \theta^\top  x \rmd\mathcal{U}(\theta)=\int \theta \theta^\top  \rmd\mathcal{U}(\theta)x = x/d$.}. 
We stress that this only implies ${T_{\sigma^\theta}^{\mu^\theta}} = \Id$ and $\sigma^\theta=\mu^\theta$ on average $w.r.t.$ $\theta$, which is weaker than $SW_2(\sigma,\mu)=0$ where ${T_{\sigma^\theta}^{\mu^\theta}} = \Id$ for $\mathcal{U}$-$a.e.~\theta \in \mathbb{S}^{d-1}$. 
Although critical points of $\mathscr{F}$ may differ from $\mu$ \citep{vauthier2025properties}, the next lemma describes conditions under which it must coincide. 

\begin{lemma}[Lemma 5.7.2 from \cite{bonnotte2013unidimensional}] \label{lem_stationarypoints}
    Suppose that the target measure $\mu\in \mathcal{P}_{2,ac}(B(0,r))$ has a strictly positive density. Then, $\sigma = \mu$ 
    if and only if $\nabla_{W_2}\mathscr{F}(\sigma) = 0$.
\end{lemma}
\Cref{lem_stationarypoints} provides assumptions under which convergence towards a critical point implies convergence towards the target measure $\mu$. 

\subsection{Smoothness}\label{appsmoothness}

\begin{lemma}\label{smoothness}
For $\sigma,\mu \in \mathcal{P}_2(\RR^d)$, let $\mathcal{F}(\sigma) = \frac{1}{2}SW_2^2(\sigma,\mu) = \frac{1}{d}\mathscr{F}(\sigma)$. 
Let $T_0,T_1\in \Lt(\sigma)$ such that $\sigma_0 = T_0 \sharp \sigma$ and $\sigma_1 = T_1 \sharp \sigma$. 
For $t \in (0,1)$ and $\sigma_t = ((1-t)T_0 + tT_1)_\sharp \sigma$, 
\begin{equation}\label{SmoothnessLoss1}
    SW_2^2(\sigma_t,\mu) \geq (1-t)SW_2^2(\sigma_0,\mu) + t SW_2^2(\sigma_1,\mu) - t(1-t) \frac{1}{d} \Vert T_0 - T_1 \Vert^2_\sigma, 
\end{equation}
and 
\begin{equation}\label{SmoothnessLoss1bis}
    \langle \nabla_{W_2}\mathcal{F}(\sigma_0)\circ T_0 - \nabla_{W_2}\mathcal{F}(\sigma_1)\circ T_1, T_0 - T_1 \rangle_\sigma \leq \frac{1}{d}\Vert T_0-T_1\Vert_\sigma^2. 
\end{equation}
Besides, for any $\sigma,\mu \in \mathcal{P}_2(\RR^d)$, for any $T\in \Lt(\sigma)$, 
\begin{equation}\label{SmoothnessLoss2}
    SW_2^2(T_\sharp \sigma,\mu) \leq SW_2^2(\sigma,\mu) 
    + 
    2 \langle \nabla_{W_2} \mathcal{F}(\sigma), T-\Id \rangle_\sigma 
    + 
    \frac{1}{d} \Vert T - \Id \Vert_\sigma^2. 
\end{equation}
   
\end{lemma}

\Cref{smoothness} is simply a rewriting of results from \citet[Proposition 4.7]{vauthier2025properties}, with \eqref{SmoothnessLoss1bis} being a well-known equivalent characterization for smoothness \citep{zhou2018fenchel}. 

\begin{proof}
Define $\mathcal{F}_\sigma: T \mapsto \mathcal{F}(T_\sharp \sigma)$ on $(\Lt(\sigma),\| \cdot \|_\sigma)$. 
One has that $\nabla \mathcal{F}_\sigma(T) = \nabla_{W_2}\mathcal{F}(T_\sharp \sigma)\circ T$ \citep[Proposition 1]{bonetmirror}, so that smoothness results on $\mathcal{F}_\sigma$ are equivalent to that on $\mathscr{F}$, except that the linear structure of $(\Lt(\sigma),\| \cdot \|_\sigma)$ is easier to deal with than $(\mathcal{P}_2(\RR^d),W_2)$. 

The inequality  \eqref{SmoothnessLoss1bis} is a rewriting of \citet[Proposition 4.7]{vauthier2025properties}, that itself follows from the semi-concavity of Wasserstein distances along generalized geodesics \citep[Theorem 7.3.2]{ambrosio2007gradient}. 
A change-of-variables $(T - \Id = \xi)$ in \citep[Proposition 4.7(a)]{vauthier2025properties} gives us that
$$
G_\sigma : T \mapsto \frac{1}{2d} \Vert T - \Id \Vert_\sigma^2 - \mathcal{F}_\sigma(T) 
$$
is convex on $(\Lt(\sigma), \Vert \cdot \Vert_\sigma)$. 
But note that
$
\nabla G_\sigma(T) =
\frac{1}{d}(T - \Id) - \nabla \mathcal{F}_\sigma(T).
$
Hence, first-order conditions for convexity \citep[Proposition 13]{bonetmirror} applied to $F_\sigma$ yield
$$
\langle \nabla G_\sigma(T_0) - \nabla G_\sigma(T_1), T_0 - T_1 \rangle_\sigma \geq 0,
$$
which directly implies \eqref{SmoothnessLoss1bis}. 
Besides, one can find in \citet[Appendix B.6]{vauthier2025properties}, namely the equations (140) and (155), that, for any $\sigma,\mu\in\mathcal{P}_2(\RR^d)$:
\begin{enumerate}
    \item[(a)] for $\xi_0,\xi_1 \in \Lt(\sigma)$, $\xi_t=(1-t)(\Id + \xi_0) + t (\Id + \xi_1)$ and $\sigma_t = (\xi_t)_\sharp\sigma$:
    $$
    SW_2^2(\sigma_t, \mu) \geq (1-t) SW_2^2(\sigma_0, \mu) + t SW_2^2(\sigma_1, \mu) - t(1-t)\frac{1}{d} \Vert \xi_0 - \xi_1 \Vert_\sigma^2,
    $$
    \item[(b)] for $\xi \in \Lt(\sigma)$: $SW_2^2((\Id + \xi)_\sharp \sigma,\mu) \leq SW_2^2(\sigma,\mu) 
    + 2 \langle \nabla_{W_2}\mathcal{F}(\sigma), \xi \rangle_\sigma 
    + \frac{1}{d}\Vert \xi \Vert_\sigma^2. $
\end{enumerate}
When replacing $T_i = \Id + \xi_i$ and $T = \Id + \xi$, one recovers immediately \eqref{SmoothnessLoss1} and \eqref{SmoothnessLoss2}.
\end{proof}

\begin{corollary}\label{corsmoothness}
    $\mathscr{F}$ is $1$-smooth with respect to the Wasserstein distance on $\mathcal{P}_{2}(\RR^d)$, \ie, for any $\sigma_1,\sigma_2 \in \mathcal{P}_{2}(\RR^d)$, 
    \begin{equation*}
    \mathscr{F}(\sigma_2) \leq \mathscr{F}(\sigma_1) + \langle \nabla_{W_2} \mathscr{F}(\sigma_1),T_{\sigma_1}^{\sigma_{2}}-\Id \,  \rangle_{\sigma_1} + \frac{1}{2}W_2^2(\sigma_1,\sigma_2) \,.
    \end{equation*}
\end{corollary}
\begin{proof}
For any $\sigma_1,\sigma_2\in \mathcal{P}_2(\RR^d)$, if the OT map $T_{\sigma_1}^{\sigma_{2}}$ from $\sigma_1$ to $\sigma_2$ exists, then $\Vert T_{\sigma_1}^{\sigma_2}-\Id \Vert_{\sigma_1}^2 = W_2^2(\sigma_1,\sigma_2)$. The final result follows from \eqref{SmoothnessLoss2}.  
\end{proof}

The following lemma is new, although it is not required for our main results. 
It resembles a well-known smoothness property, but we stress that, even in Euclidean settings, it is not equivalent to the previous inequalities \citep{zhou2018fenchel}. 

\begin{lemma}\label{smoothnessSW1}
Fix $\sigma_1,\sigma_2 \in \mathcal{P}_{2,ac}(\mathbb{R}^d)$.  
    If the density function of the target $\mu$ is strictly larger than $1/\kappa>0$ on its compact domain, then 
    $$
    \Vert \nabla_{W_2}\mathscr{F}(\sigma_1)- \nabla_{W_2}\mathscr{F}(\sigma_2)\Vert_{\lambda}^2 \leq 2 \kappa\, SW_1(\sigma_1,\sigma_2),
    $$
    for $\lambda$ the Lebesgue measure. 
\end{lemma}
\begin{proof}
By Jensen's inequality, 
    \begin{align*}
        \Vert \nabla_{W_2}\mathscr{F}(\sigma_1)- \nabla_{W_2}\mathscr{F}(\sigma_2)\Vert_{\lambda}^2 
        &\leq \int  \int \vert  T_{\sigma_1^\theta}^{\mu^\theta}(x^\top \theta) - T_{\sigma_2^\theta}^{\mu^\theta}(x^\top \theta)\vert^2  \rmd\lambda(x)\rmd\mathcal{U}(\theta),\\
        &\leq \int \int \vert  
        C_{\mu^\theta}^{-1}\circ C_{\sigma_1^\theta}(x^\top \theta) -C_{\mu^\theta}^{-1}\circ C_{\sigma_2^\theta}(x^\top \theta) \vert^2 \rmd\lambda(x)\rmd\mathcal{U}(\theta),
    \end{align*}
    for $C_{\rho}$ the univariate distribution function of $\rho\in \mathcal{P}_2(\mathbb{R}^d)$. 
    For all $\theta \in \mathbb{S}^{d-1}$, the quantile function $C_{\mu^\theta}$ is $\kappa$-lipschitz with $1/\kappa$ the essential infimum of the density of $\mu$ on its domain, \citep{bobkov2019one}. Then, the result follows from
    \begin{align*}
        \Vert \nabla_{W_2}\mathscr{F}(\sigma_1)- \nabla_{W_2}\mathscr{F}(\sigma_2)\Vert_{\lambda}^2 
        &\leq 2\kappa \int \int \vert  
        C_{\sigma_1^\theta}(x^\top \theta) - C_{\sigma_2^\theta}(x^\top \theta) \vert \rmd\lambda(x) \rmd\mathcal{U}(\theta).
    \end{align*}
\end{proof}

\subsection{Moments are bounded}

An important assumption when dealing with stochastic algorithms is the boundedness of the gradient norm. 
Here, a direct consequence of Proposition \ref{PropGradients} is that, using respectively the definition of the Haar measure and Jensen's inequality,
$$
    \mathbb{E}_P \Vert \nabla_{W_2} \mathscr{F}(\sigma,P)\Vert_{\sigma}^2 
    = 
    \mathbb{E}_P \sum_{\ell = 1 }^d W_2^2(\sigma^{\theta_\ell}, \mu^{\theta_\ell})  = 2\mathscr{F}(\sigma)
    \quad \mbox{and} \quad
    \Vert \nabla_{W_2} \mathscr{F}(\sigma)\Vert^2 \leq 2 \mathscr{F}(\sigma).
$$

Hence, the gradient norm is bounded as long as the objective $\mathscr{F}(\sigma)$ is. 
We now show that the second-order moments remain bounded along the IDT iterations, a fact that implies a bound on $(\mathscr{F}(\sigma_k))_k$.
Bounds on the moments along the Sliced-Wasserstein flow were proved in \cite{cozzi2025}, in a continuous-time setting, whereas we deal with discrete step sizes $(\gamma_k)$. 
Denote by $\mathrm{M}_2(\rho) = \int \Vert \cdot \Vert^2 d\rho$ the second-order moment of a probability distribution $\rho\in \mathcal{P}_2(\mathbb{R}^d)$. 

\begin{proposition}\label{prop_momentsbounded} 
    Moments are bounded by $\mathrm{M}_2(\mu)$ along the IDT iterations \eqref{IDT}. 
    In other words, for any $k\geq 1$, we have 
    $$
    \mathrm{M}_2(\sigma_k) \leq \mathrm{M}_2(\mu).
    $$ 
    Consequently, for $M=4 \mathrm{M}_2(\mu)$,
    $$
    W_2^2(\sigma_k,\mu) \leq M
    \hspace{1cm} \mbox{and} \hspace{1cm}
    SW_2^2(\sigma_k,\mu) \leq \frac{M}{d}. 
    $$
\end{proposition}

\begin{proof}
    This result is a consequence of the moment-matching property of sliced maps. Indeed, as shown in \citet[Proposition 3.6]{li2024approximation}, for all $k\geq 0$, 
    \begin{align}
    \int \Vert T_{\sigma_k,P_{k+1}}(x) \Vert^2 \rmd\sigma_k(x) 
    &= \int \Vert \sum_{\ell=1}^d \theta_\ell t_{\theta_\ell}(\theta_\ell^\top  x ) \Vert^2 \rmd\sigma_k(x) 
    =\sum_{\ell=1}^d  \int  \Vert t_{\theta_\ell}(\theta_\ell^\top  x ) \Vert^2 \rmd\sigma_k(x),\nonumber \\
    &= \sum_{\ell=1}^d \mathrm{M}_2(\mu^{\theta_\ell}) 
    =   \sum_{\ell=1}^d \int \langle y, \theta_\ell \rangle^2 \rmd\mu(y) = \int \Vert y \Vert^2 \rmd\mu(y) = \mathrm{M}_2(\mu).\label{momentmatchingprop}
    \end{align}
    
    With this at hand, we proceed by induction. 
    At initialization, for $k=1$, we have that $\gamma_1 = 1$ and $\mathrm{M}_2(\sigma_1) = \int \Vert T_{\sigma_0,P_{1}}(x) \Vert^2 d\sigma_0(x) = \mathrm{M}_2(\mu)$. 
    For the induction step, assume that there exists an index $k \in \mathbb{N}^*$ such that $\mathrm{M}_2(\sigma_k) \leq \mathrm{M}_2(\mu) $. 
    Then, by convexity of $\Vert \cdot \Vert^2$, 
    \begin{align}\label{recursiveOnMoments}
        \mathrm{M}_2(\sigma_{k+1}) &= \int \Vert x \Vert^2 \rmd\sigma_{k+1}(x) = \int \Vert (1-\gamma_k)x + \gamma_k T_{\sigma_k,P_{k+1}}(x) \Vert^2 \rmd\sigma_k(x),\nonumber\\
        &\leq (1-\gamma_k) \mathrm{M}_2(\sigma_k) + \gamma_k \int \Vert T_{\sigma_k,P_{k+1}}(x) \Vert^2 \rmd\sigma_k(x). 
    \end{align}
    Plugging \eqref{momentmatchingprop} in \eqref{recursiveOnMoments} and invoking the induction hypothesis that $\mathrm{M}_2(\sigma_k)\leq \mathrm{M}_2(\mu)$, the desired result on the moment boundedness follows:
    \begin{align}\label{InegOnMoments}
        \mathrm{M}_2(\sigma_{k+1}) 
        \leq (1-\gamma_k) \mathrm{M}_2(\sigma_k) + \gamma_k \mathrm{M}_2(\mu) 
        \leq (1-\gamma_k) \mathrm{M}_2(\mu)  + \gamma_k \mathrm{M}_2(\mu)  \leq \mathrm{M}_2(\mu) . 
    \end{align}
    Next, to obtain the bound on the Wasserstein distance $W_2^2(\sigma_k,\mu)$, let us call $T^*$ the OT map from $\sigma_k$ to $\mu$. 
    Young's inequality for products together with the change-of-variable $\mathbb{E}_{X\sim \sigma_k}(\Vert T^*(X)\Vert^2) = \mathbb{E}_{Y\sim \mu} (\Vert Y \Vert^2)$ leads to
    \begin{align*}
        W_2^2(\sigma_k,\mu) &= \mathbb{E}_{X \sim \sigma_k}[\Vert X-T^*(X)\Vert^2] \\
        &= \mathbb{E}[\Vert X\Vert^2] + \mathbb{E}[\Vert T^*(X)\Vert^2] - 2\mathbb{E}[\langle X,T^*(X)\rangle] \\
        &\leq 2\left( \mathbb{E}[\Vert X\Vert^2] + \mathbb{E}[\Vert T^*(X)\Vert^2] \right) \\
        &\leq 2(\mathrm{M}_2(\sigma_k) + \mathrm{M}_2(\mu) )\,,
    \end{align*}
    that is the desired result. 
\end{proof}

\subsection{Standard proofs for non-convex smooth optimization}\label{proofssmoothSGD}

The next two demonstrations are standard~\citep{bottou2018optimization,dossal2024optimization}. 

\paragraph{Proof of Proposition \ref{convergenceGradient}.}

From \eqref{RobbinsSiegmund}, the sequences $a_k=\gamma_k \Vert \nabla_{W_2} \mathscr{F}(\sigma_k)\Vert_{\sigma_k}^2$ and $b_k = \gamma_k^{-1}$ verify
$$
\sum_{k\geq 1} a_k <+\infty 
\hspace{1cm} \mbox{and} \hspace{1cm} \lim_{k\rightarrow +\infty} b_k = +\infty. 
$$
Hence, by Kronecker's lemma,
\begin{equation}\label{KroneckerLemma}
\lim\limits_{K\rightarrow +\infty} \gamma_K \sum_{k=1}^K \Vert \nabla_{W_2} \mathscr{F}(\sigma_k)\Vert_{\sigma_k}^2 = 0.
\end{equation}
Let $\epsilon >0$. Markov's inequality yields,
$$
\mathbb{P}\big(\Vert \nabla \mathscr{F}(\sigma_{i(K)})\Vert_{\sigma_{i(K)}}^2 >\epsilon\big) \leq \frac{1}{\epsilon} \mathbb{E} \big( \Vert \nabla \mathscr{F}(\sigma_{i(K)})\Vert_{\sigma_{i(K)}}^2 \big).
$$
The above expectation is taken with respect to the stochastic iterates $\sigma_k$ as well as the random choice of $i(K)$. Since these two sources of randomness are independent,
$$
\mathbb{E} \big( \Vert \nabla \mathscr{F}(\sigma_{i(K)})\Vert_{\sigma_{i(K)}}^2 \big)=
\mathbb{E}_{K} \mathbb{E}_{i(K)}\big( \Vert \nabla \mathscr{F}(\sigma_{i(K)})\Vert_{\sigma_{i(K)}}^2 \big),
$$
where $\mathbb{E}_{K}$ denotes the expectation over the stochastic iterates $\sigma_1,\cdots,\sigma_K$. Therefore, 
\begin{align*}
    \mathbb{P}\big(\Vert \nabla \mathscr{F}(\sigma_{i(K)})\Vert_{\sigma_{i(K)}}^2 >\epsilon\big) \leq \frac{1}{\epsilon} \mathbb{E}_{K} \left( \frac{1}{K} \sum_{k=1}^K \Vert \nabla \mathscr{F}(\sigma_{k})\Vert_{\sigma_{k}}^2 \right),
\end{align*}
which converges towards $0$ from \eqref{KroneckerLemma} with $\gamma_K \geq 1/K$ together with the dominated convergence theorem, the domination assumption coming from the boundedness of gradients in Proposition \ref{prop_momentsbounded}.  

\paragraph{Proof of Proposition~\ref{propratePonderatedAvgGrad}.} 
Taking the expectation in \eqref{recursiveIneq}, 
rearranging and telescoping the sum (with $-\mathscr{F}(\sigma_{K+1}) \leq 0$) yields 
$$
 \sum_{k=0}^K \gamma_k \mathbb{E}\Vert \nabla_{W_2}\mathscr{F}(\sigma_k)\Vert_{\sigma_k}^2 \leq \mathscr{F}(\sigma_0) +  \sum_{k=0}^K \gamma_k^2 \mathbb{E}\mathscr{F}(\sigma_k). 
$$
By Proposition \ref{prop_momentsbounded}, $\mathscr{F}(\sigma_k)\leq 2 \mathrm{M}_2(\mu)$. The final result follows from dividing both sides of the above inequality by $\sum_{k=1}^K \gamma_k$.

\section{Proofs of Section \ref{PL_KL_ineq}: \L{}ojasiewicz inequalities}\label{appPLKL}

\subsection{Proof of Proposition \ref{propPoincareflatconvex}: a PL-like inequality for smooth densities}
By (flat) convexity over densities equipped with the $2$-norm~\eqref{intPsidsigma_mugeqSW},
$$
\mathscr{F}(\sigma) \leq \int \mathscr{F}'[\sigma] \rmd(\sigma - \mu) \leq \int (\mathscr{F}'[\sigma]-c) \big(\frac{\rmd\sigma}{\rmd\nu} - \frac{\rmd\mu}{\rmd\nu}\big) \rmd\nu,
$$
where the second inequality uses the notation $c = \int \mathscr{F}'[\sigma] \rmd\nu$ and the fact that $\int c (\sigma - \mu) \rmd\nu=0$ since $\sigma,\mu$ are both probability distributions. 
Using the Cauchy-Schwarz inequality, and then the Poincaré inequality for $\nu$,
\begin{align*}
    \mathscr{F}(\sigma) 
    &\leq
    \Big\Vert \frac{\rmd\sigma}{\rmd\nu} - \frac{\rmd\mu}{\rmd\nu} \Big\Vert_\nu \sqrt{\mbox{Var}_\nu( \mathscr{F}'[\sigma] )} \\
    &\leq 
    C_\nu \Big\Vert \frac{\rmd\sigma}{\rmd\nu} - \frac{\rmd\mu}{\rmd\nu} \Big\Vert_\nu \big\Vert \nabla \mathscr{F}'[\sigma] \big\Vert_{\nu} \,.
\end{align*}    
Additionally, by the boundedness assumption \eqref{densitiesBounded}, $\Vert \nabla \mathscr{F}'[\sigma] \Vert_{\nu} \leq \frac1m \Vert \nabla \mathscr{F}'[\sigma] \Vert_{\sigma}$, and $\Vert \frac{\rmd\sigma}{\rmd\nu} - \frac{\rmd\mu}{\rmd\nu}\Vert_\nu\leq \Vert \frac{\rmd\sigma}{\rmd\nu} \Vert_\nu +  \Vert \frac{\rmd\mu}{\rmd\nu}\Vert_\nu \leq 2M$, so the result follows by using that $    \nabla_{W_2}\mathscr{F} (\sigma)
    =
    \nabla \mathscr{F}'[\sigma] $. 

\subsection{Proof of Proposition \ref{KLGaussians}: a PL-like inequality for Gaussian distributions}
\label{app:proof_PL_Gaussian}

\begin{proposition}[General covariances]\label{KLGaussians}
    Assume that $\sigma = \mathcal{N}(0,\Sigma)$ and  $\mu = \mathcal{N}(0,\Lambda)$, with $\Sigma$ and $\Lambda$ symmetric positive definite. Then, 
    \begin{equation}\label{KLinequalityGaussians}
    \mathscr{F}(\sigma)^2 \leq \frac{1}{2}W_2^2(\sigma,\mu)\Big(1+\frac{\lambda_{\rm max}(\Lambda)}{\lambda_{\rm min}(\Sigma)}\Big) \| \nabla_{W_2} \mathscr{F}(\sigma) \|_\sigma^2. 
    \end{equation}
\end{proposition}

As recalled in Appendix \ref{app:remindWassSpace},  
a Kantorovich potential is solution of the dual formulation of OT. 
For $\psi_\theta$ the Kantorovich potential associated with the transport from $\sigma^{\theta}$ to $\mu^{\theta}$, let 
$
\Psi(x) = \int \psi_\theta(x^\top \theta) d\mathcal{U}(\theta),
$
so that 
$$
SW_2^2(\sigma,\mu) = \int \Psi d\sigma + 
\int \int \psi_\theta^c(y^\top \theta) \rmd\mathcal{U}(\theta)\rmd\mu(y).
$$
Note that this dual formulation was recently proven for generalized sliced metrics \citep[Main Theorem, (6)]{kitagawa2024}.
Then,
\begin{align*}
    \int \Psi \rmd\sigma - \int \Psi \rmd\mu = \int \Psi \rmd\sigma - \int \int \psi_\theta(y^\top \theta) \rmd\mathcal{U}(\theta) \rmd\mu(y). 
\end{align*}
By definition of the $c$-transform, we have for all $u,v$ that $\psi_\theta^c(u) \leq \frac{1}{2}\Vert u-v \Vert^2- \psi_\theta(v)$, and, for $v=u$,  $\psi_\theta^c(u) \leq -\psi_\theta(u)$. 
As a byproduct, one recovers
\begin{align}\label{intPsidsigma_mugeqSW}
\frac{1}{2}SW_2^2(\sigma,\mu) =  \int \Psi \rmd\sigma + \int \psi_\theta^c(y^\top  \theta) \rmd\mathcal{U}(\theta) \rmd\mu(y) \leq 
    \int \Psi \rmd(\sigma-\mu). 
\end{align}
We stress that the above is a rewriting of the (flat) convexity of $L$ with respect to the $2$-norm between densities, because $\Psi$ is the first variation of $L$ at $\sigma$, \citep{cozzi2025}. 
Since $\Psi$ is locally Lipschitz \citep[Theorem 10.4]{rockafellar-1970a}, 
a direct application of \citet[Lemma 13]{chewi2020gradient} yields
\begin{equation}\label{Lemma13Chewi}
    \left| \int \Psi \rmd\sigma - \int \Psi \rmd\mu \right|
    \leq W_2(\sigma,\mu) \int_0^1 \Vert \nabla \Psi\Vert_{\rho_t} \rmd t,
\end{equation}
where $\rho_t = ((1-t)\Id + t T_{\sigma}^{\mu})_\sharp\sigma$ is the Wasserstein geodesic between $\sigma$ and $\mu$. 
Combining \eqref{intPsidsigma_mugeqSW} and \eqref{Lemma13Chewi}, 
\begin{equation*}
SW_2^2(\sigma,\mu) \leq  W_2(\sigma,\mu) \int_0^1 \Vert \nabla \Psi\Vert_{\rho_t} \rmd t.
\end{equation*}
Taking the square and applying Jensen's inequality,
$$
SW_2^4(\sigma,\mu) \leq  W_2^2(\sigma,\mu) \int_0^1 \Vert \nabla \Psi\Vert_{\rho_t}^2 \rmd t.
$$
We stress that $ \nabla \Psi(x) = \int \theta (x^\top \theta - {T_{\sigma^\theta}^{\mu^\theta}}(x^\top \theta)) d\mathcal{U}(\theta)  = (1/d) \nabla_{W_2} \mathscr{F}(\sigma)(x)$, so
\begin{equation}\label{SWleqWintGrad2}
4\mathscr{F}(\sigma)^2  \leq  W_2^2(\sigma,\mu) \int_0^1 \Vert \nabla_{W_2} \mathscr{F}(\sigma)\Vert_{\rho_t}^2 \rmd t.
\end{equation}  
Thus, it only remains to show that $\int_0^1 \Vert \nabla_{W_2} \mathscr{F}(\sigma)\Vert_{\rho_t}^2 \rmd t \lesssim \Vert \nabla_{W_2} \mathscr{F}(\sigma)\Vert_{\sigma}^2$. 
Under conditions on eigenvalues and the linearity of OT maps for Gaussian distributions, we proceed with the same arguments as in the proof of \citet[Theorem 19]{chewi2020gradient}. 
Since $\sigma^\theta = \mathcal{N}(0,\theta^\top\Sigma\theta)$ and $\mu^\theta = \mathcal{N}(0,\theta^\top\Lambda\theta)$, we have ${T_{\sigma^\theta}^{\mu^\theta}} : z \mapsto \tau_\theta z$ for  $\tau_\theta = \sqrt{ \theta^\top\Lambda \theta / \theta^\top\Sigma \theta}$. As a byproduct, 
$$
\nabla_{W_2} \mathscr{F}(\sigma)(x) = 
d\int \theta (x^\top\theta - {T_{\sigma^\theta}^{\mu^\theta}}(x^\top\theta)) \rmd\mathcal{U}(\theta) 
=
d\int (1-\tau_\theta) \theta \theta^\top \rmd\mathcal{U}(\theta) x = Ax, 
$$
for $A = d\int (1-\tau_\theta) \theta \theta^\top \rmd\mathcal{U}(\theta)$. 
Denote by  $B = \Sigma^{-1/2} (\Sigma^{1/2}\Lambda \Sigma^{1/2})^{1/2}\Sigma^{-1/2}$ such that the OT map from $\sigma$ to $\mu$ verifies $T_\sigma^\mu(x) = Bx$. Then, the integration over $\rho_t$ writes, for $X\sim \sigma$,
\begin{equation}\label{normeNablaPsisigmat}
    \Vert \nabla_{W_2} \mathscr{F}(\sigma)\Vert_{\rho_t}^2 =  \mathbb{E} \Vert (1-t)AX + t ABX \Vert^2 \leq (1-t) \mathbb{E}\Vert AX\Vert^2 + t \mathbb{E}\Vert AB X\Vert^2
\end{equation}
Because $BX\sim \mathcal{N}(0,\Lambda)$, one has $ABX\sim \mathcal{N}(0,A\Lambda A^\top)$ and thus $\mathbb{E} \Vert ABX\Vert^2  = \Tr (A\Lambda A)=\Tr (\Lambda A^2)$. 
Using the von Neumann's trace inequality (singular values coincide with eigenvalues for normal and positive matrices),
$$
\mathbb{E} \Vert ABX\Vert^2  = \Tr(\Lambda \Sigma^{-1}\Sigma A^2) 
\leq
\sum_{i=1}^d \lambda_i(\Lambda \Sigma^{-1})\lambda_i(\Sigma A^2)
\leq \frac{\lambda_{\rm max}(\Lambda)}{\lambda_{\rm min}(\Sigma)} \Tr (\Sigma A^2) =  \frac{\lambda_{\rm max}(\Lambda)}{\lambda_{\rm min}(\Sigma)} \mathbb{E} \Vert AX\Vert^2,
$$
Plugging this in \eqref{normeNablaPsisigmat} induces 
$$
\int_0^1 \Vert \nabla_{W_2} \mathscr{F}(\sigma) \Vert_{\rho_t}^2 \rmd t \leq \frac{1}{2}\Big(1+\frac{\lambda_{\rm max}(\Lambda)}{\lambda_{\rm min}(\Sigma)}\Big) \mathbb{E} \Vert AX\Vert^2 
\leq  \frac{1}{2}\Big(1+\frac{\lambda_{\rm max}(\Lambda)}{\lambda_{\rm min}(\Sigma)}\Big) \Vert \nabla_{W_2} \mathscr{F}(\sigma)\Vert^2_\sigma,
$$
and the results follows by combining with \eqref{SWleqWintGrad2}.

\subsection{Proof of Proposition \ref{corollarydiagGaussians}}

This section refines the PL-like inequality between Gaussian distributions with co-diagonalizable covariance matrices. \Cref{corollarydiagGaussians} stems upon the following result. 

\begin{proposition} \label{prop:sw_vs_w_gauss}
    Consider two centered Gaussian measures $\mu_{\Sigma}$ and $\mu_{\Lambda}$ in $\mathbb{R}^d$ with diagonal covariance matrices $\Sigma$ and $\Lambda$. Assume there exists finite constants $0 < m \leq M$ such that all diagonal entries of $\Sigma$ and $\Lambda$ lie in $[m, M]$. Then, 
    \begin{equation}
        SW_2^2(\mu_{\Sigma}, \mu_{\Lambda}) \geq \frac{m}{M d(d+2)} W_2^2(\mu_{\Sigma}, \mu_{\Lambda}) \,.
    \end{equation}
\end{proposition}

\begin{proof}
    For $i \in \{1, \dots, d\}$, denote by $\sigma_i^2$ and $\lambda_i^2$ the $i$-th diagonal element of $\Sigma$ and $\Lambda$ respectively. By the closed-form solution of the Wasserstein distance of order 2 between Gaussians, 
    \begin{align*}
        W_2^2(\mu_\Sigma, \mu_\Lambda) &= \| \Sigma^{1/2} - \Lambda^{1/2} \|^2_F = \sum_{i=1}^d (\sigma_i - \lambda_i)^2 \,.
    \end{align*}
    On the other hand, the Sliced-Wasserstein distance is defined as
    \begin{align} \label{eq:sw_gauss}
        SW_2^2(\mu_\Sigma, \mu_\Lambda) &= \mathbb{E}_{\theta \sim \mathcal{U}(\mathbb{S}^{d-1})}[(\sqrt{\theta^\top \Sigma \theta} - \sqrt{\theta^\top \Lambda \theta})^2] \,. 
    \end{align}
    For all $(x, y) \in \mathbb{R}^2$, $(\sqrt{x}-\sqrt{y})(\sqrt{x}+\sqrt{y}) = x - y$. Additionally, if $0 < x, y < M$, 
    $$
    (\sqrt{x}-\sqrt{y})^2 \geq \frac{(x - y)^2}{4 M} \,.
    $$ 
    Since for all $i \in \{1, \dots, d\}$,  $\sigma_i^2$ and $\lambda_i^2$ are bounded between $m$ and $M$, so are $\theta^\top \Sigma \theta$ and $\theta^\top \Lambda \theta$. We can thus bound \eqref{eq:sw_gauss} as,  
    \begin{align}\label{lowerboundSWVartrace}
        SW_2^2(\mu_\Sigma, \mu_\Lambda) &\geq \frac1{4M}\,\mathbb{E}_{\theta \sim \mathcal{U}(\mathbb{S}^{d-1})}[(\theta^\top \Gamma \theta)^2] \,,
    \end{align}
    where $\Gamma = \Sigma - \Lambda$. Since $\theta$ is uniformly distributed on the sphere, one can show \citep{wiens1992moments}
    \begin{align*}
        \mathbb{E}_{\theta \sim \mathcal{U}(\mathbb{S}^{d-1})}[(\theta^\top \Gamma \theta)^2] = \frac{2\mathrm{Tr}(\Gamma^2) + (\mathrm{Tr}(\Gamma))^2}{d(d+2)} \,.
    \end{align*}
    The final result follows from $\mathrm{Tr}(\Gamma)^2 \geq 0$ and
    \begin{align*}
        \mathrm{Tr}(\Gamma^2) &= \sum_{i=1}^d (\sigma_i^2 - \lambda_i^2)^2 = \sum_{i=1}^d (\sigma_i - \lambda_i)^2(\sigma_i + \lambda_i)^2 \\
        &\geq 4 m \sum_{i=1}^d (\sigma_i - \lambda_i)^2 \\
        &\geq 4 m W_2^2(\mu_\Sigma, \mu_\Lambda) \,.
    \end{align*}
\end{proof}

\begin{remark}[Extension to elliptically contoured distributions]
    Proposition \ref{prop:sw_vs_w_gauss} can be readily extended to the class of elliptically contoured distributions whose positive definite parameters are co-diagonalizable.
\end{remark}

\begin{proof} \textit{(Proof of Proposition \ref{corollarydiagGaussians})} By~\Cref{prop:sw_vs_w_gauss}, for co-diagonalizable covariance matrices, there exists $C_{m,d}>0$ such that 
$W_2^2(\sigma,\mu) \leq SW_2^2(\sigma,\mu) C_{m,d}$. We conclude by rearranging terms in \Cref{KLGaussians}. 
\end{proof}

\section{Proofs of \Cref{SecEigenControlConvergRates}: Eigenvalues control along the iterations}

\paragraph{Objective and bottleneck.}
Recall that the inequality provided in \eqref{KLinequalityGaussians} writes
$$
\mathscr{F}(\sigma)^2 \leq \frac{1}{2}W_2^2(\sigma,\mu)\Big(1+\frac{\lambda_{\rm max}(\Lambda)}{\lambda_{\rm min}(\Sigma)}\Big) \| \nabla_{W_2} \mathscr{F}(\sigma) \|_\sigma^2.
$$
In order to use this inequality for convergence rates, 
one only needs to control eigenvalues along the iterations, as $W_2(\sigma_k,\mu)$ is bounded from Proposition \ref{prop_momentsbounded}. 
This is the purpose of the remaining of the section. 
Firstly, the following recursion for covariances of $(\sigma_k)_k$ is known to hold for Wasserstein geodesics between Gaussians, \citep[Appendix A]{altschuler2021averaging}:
$$
\Sigma_{k+1} = ((1-\gamma_k)\Id + \gamma_k T_{P_{k+1}})\Sigma_k ((1-\gamma_k)\Id + \gamma_k T_{P_{k+1}}),
$$
where $T_{P_{k+1}} = P_{k+1} D_{k}P_{k+1}$ is the matrix form of the sliced map from $\sigma_k$ to $\mu$ in the directions $P_{k+1}$ (it will be detailed in the next  \Cref{ControlEigSigk}). 
A convenient feature is that eigenvalues can be controlled along such Wasserstein geodesics, by eigenvalues of $\sigma_k$ and $T_{P_{k+1}\sharp}\sigma_k $ \citep{chewi2020gradient,altschuler2021averaging}. 
Nonetheless, in our particular setting, sliced maps do not necessarily push the source forward onto the target. 
Hence, the covariance matrix of $T_{P_{k+1}\, \sharp}\sigma_k$ is not necessarily the one of $\mu$, and control of eigenvalues is not a direct byproduct of assumptions on $\mu$. 

\paragraph{Sketch.}
This section is structured as follows. 
A recursive inequality for eigenvalues of the covariance matrix of $T_{P_{k+1}\, \sharp}\sigma_k$ is given in \Cref{ControlEigSigk}.
It includes randomness coming from the stochastic gradients and the choice of projection directions.   
The latter randomness is controlled in  \Cref{ControlEigSigk2} by bounding expectations with  \Cref{upperboundsMomentsQuadratForms}, assuming that the target $\mu$ is isotropic. 
If instead $\mu$ has a general covariance matrix,  \Cref{sufficientconditionarbitrarycov} gives only a sufficient condition.

\subsection{Recursive inequalities on eigenvalues}
\begin{proposition}\label{ControlEigSigk}
Assume that $\sigma_k = \mathcal{N}(0,\Sigma_k)$ and $\mu = \mathcal{N}(0,\Lambda)$, with $\Sigma_k$ and $\Lambda$ symmetric positive definite.
Then, 
  there exist directions $\theta_i,\theta_j$ taken from the basis $P_{k+1}$ such that, for $\tau_\theta = \sqrt{\theta^\top \Lambda \theta/\theta^\top \Sigma \theta}$, 
  \begin{equation}\label{BoundsEigenvalues}
      \sqrt{\lambda_{\rm min}(\Sigma_k)} \big(1 + \gamma_k (\tau_{\theta_i}-1) \big)
      \leq \sqrt{\lambda_{\rm min} (\Sigma_{k+1})}  
      \leq \sqrt{\lambda_{\rm max} (\Sigma_{k+1})} 
      \leq \sqrt{\lambda_{\rm max}(\Sigma_k)} \big(1 + \gamma_k (\tau_{\theta_j}-1)\big).
  \end{equation}
  In particular, $\Sigma_{k+1}$ is symmetric positive definite.   
\end{proposition}
\begin{proof}
The distribution $\sigma_{k+1}$ corresponds to the random vector
\begin{equation}\label{pathSigkk1}
(1-\gamma_k)X + \gamma_k T_{P_{k+1}}(X), 
\end{equation}
where $X \sim \mathcal{N}(0,\Sigma_k)$. 
Also, by definition,
  $$
  T_{P_{k+1}} (X) = \sum_{\ell=1}^d \theta_\ell t_{\theta_\ell}(X^\top\theta_\ell) 
  = \sum_{\ell=1}^d \tau_{\theta_\ell} \theta_\ell\theta_\ell^\top X = P_{k+1} D_k P_{k+1}^\top X,
  $$
  where $D_k = \mbox{diag}(\tau_{\theta_1},\dots,\tau_{\theta_d})$ is positive definite. 
  With these notations, $T_{P_{k+1}}(X)\sim\mathcal{N}(0,\Gamma)$ with $\Gamma = P_{k+1} D_k P_{k+1}^\top\Sigma_k P_{k+1} D_k P_{k+1}^\top$ and $T_{P_{k+1}}$ is the gradient of a convex function. 

As a byproduct, the interpolate \eqref{pathSigkk1} belongs to the path $t \mapsto ((1-t)\Id + t T_{P_{k+1}})_\sharp \sigma_k$ that is a Wasserstein geodesic bridging two Gaussian distributions.
The functionals $-\sqrt{\lambda_{\rm min}}$ and $\sqrt{\lambda_{\rm max}}$ have been shown to be convex along barycenters \citep[Theorem 6]{altschuler2021averaging}, a fortiori convex along Wasserstein geodesics \citep[Proposition 7.3]{agueh2011barycenters}. 
In other words,  
\begin{align}
    (1-\gamma_k) \sqrt{\lambda_{\rm min}(\Sigma_k)} + &\gamma_k \sqrt{\lambda_{\rm min} (\Gamma)} \leq \sqrt{\lambda_{\rm min} (\Sigma_{k+1})} \nonumber\\ 
    & \sqrt{\lambda_{\rm max} (\Sigma_{k+1})} \leq (1-\gamma_k) \sqrt{\lambda_{\rm max}(\Sigma_k)} + \gamma_k \sqrt{\lambda_{\rm max} (\Gamma)}.\label{interpolEigenvalues}
\end{align}
Hence, it remains to control eigenvalues of $\Gamma$. 
  On the one hand, $\overline{\Sigma} = P_{k+1}^\top\Sigma_k P_{k+1} $ and $\Sigma_k$ have the same eigenvalues, by orthonormality of $P_{k+1}$\footnote{Eigenvectors of $\overline{\Sigma}$ are of the form $P_{k+1}^\top u$ for $u$ an eigenvector of $\Sigma_k$. Indeed, $(P_{k+1}^\top u)^\top  P_{k+1}^\top \Sigma_k P_{k+1} (P_{k+1}^\top u) = u^\top  \Sigma_k u$ which equals an eigenvalue of $\Sigma_k$.}.
  On the other hand, $\Gamma$ has the same eigenvalues as $D \overline{\Sigma} D_k$ from the same argument. 
  Also, $D$ is non singular because $\Sigma_k$ and $\Lambda$ have positive eigenvalues, hence $\tau_{\theta_\ell}>0$ for all $\ell = 1,\cdots,d$. 
  Then, 
  a direct application of Ostrowski's Theorem \citep{ostrowski1959quantitative} entails that
  \begin{equation}\label{Ostrowskieig}
  \lambda_i(\Gamma) = \lambda_i( D_k \overline{\Sigma} D_k) = \beta_i \lambda_i (\Sigma_k),
  \end{equation}
  with 
  $$
  \min_j \frac{\theta_j^\top \Lambda\theta_j}{\theta_j^\top \Sigma_k\theta_j} \leq \beta_i \leq \max_j \frac{\theta_j^\top \Lambda\theta_j}{\theta_j^\top \Sigma_k\theta_j}. 
  $$
  Thus, the result follows by combining \eqref{interpolEigenvalues} and \eqref{Ostrowskieig}.
\end{proof}

\subsection{A bound in expectation for isotropic target}

\Cref{ControlEigSigk2} gives a deterministic upper bound on eigenvalues of $(\Sigma_k)$, and a lower bound in expectation. 
It requires bounds on $p$-moments of $  \theta^\top \Sigma_k\theta - 1$, that are provided just after in \Cref{upperboundsMomentsQuadratForms}. 

\begin{proposition}\label{ControlEigSigk2}
Assume that $\sigma_0 = \mathcal{N}(0,\Sigma_0)$, with $\Sigma_0 \in \mathbb{R}^{d \times d}$ symmetric, positive-definite, and  $\mu = \mathcal{N}(0,{\bf I}_d)$. Then, for any $k\geq 1$, the IDT iterates remain Gaussian, $\sigma_k=\mathcal{N}(0,\Sigma_k)$ with
    \begin{align}
    \mathbb{E}[1/\lambda_{\rm min}(\Sigma_k)] &\leq \mathbb{E}[1/\lambda_{\rm min}(\Sigma_1)] \label{boundlambdamin} \\
    \forall p \in \mathbb{N}^*, \quad \mathbb{E}[1/\lambda_{\rm min}(\Sigma_k)^p] &\leq \mathbb{E}[1/\lambda_{\rm min}(\Sigma_1)^{p}] \prod_{l = 1}^k (1 + B_p\gamma_l^2) \,.  \label{boundlambdaminwithP}
    \end{align}
    where $B_p > 0$. Note that $\prod_{l = 1}^\infty (1 + B_p\gamma_l^2)$ is finite for any step-sizes sequence $(\gamma_k)_{k}$ satisfying \eqref{choicestep-size}. 
\end{proposition}

\begin{proof}
A direct byproduct of \Cref{prop_momentsbounded} is that $\lambda_{\rm max}(\Sigma_k)\leq \mathrm{Tr}(\Sigma_k) \leq \mathrm{Tr}(\Id)=d$. 
We now focus on showing \eqref{boundlambdamin}. 

From \eqref{BoundsEigenvalues}, there exists $1\leq i\leq d$ such that 
  $
      \sqrt{\lambda_{\rm min}(\Sigma_k)} \big(1 + \gamma_k (\tau_{\theta_i}-1)\big) \leq \sqrt{\lambda_{\rm min} (\Sigma_{k+1})}.
  $
  Taking the inverse and using that the harmonic mean is always smaller than the arithmetic mean, 
  \begin{equation*}
      (\lambda_{\rm min} (\Sigma_{k+1}))^{-1/2}
      \leq (\lambda_{\rm min}(\Sigma_k))^{-1/2} \big(1-\gamma_k + \gamma_k \tau_{\theta_i}^{-1}\big). 
  \end{equation*}
  Here, everything is positive due to the positivity of all $(\tau_{\theta_i})_i$, so that taking the square and applying Jensen's inequality yields 
  \begin{align}
      (\lambda_{\rm min} (\Sigma_{k+1}))^{-1}
      &\leq (\lambda_{\rm min}(\Sigma_k))^{-1} \big(1-\gamma_k + \gamma_k \theta_i^\top \Sigma_k\theta_i\big).\label{BoundsEigenvaluesINF2} 
  \end{align}
  Recall that $\theta_i$ belongs to the random basis $P_{k+1}$, whose distribution is independent from the $\sigma$-field $\mathcal{A}_k$ generated by $P_1,\dots,P_k$. 
  Also, $\Sigma_k$ is measurable with respect to $\mathcal{A}_k$. 
  Hence, taking the conditional expectation in \eqref{BoundsEigenvaluesINF2} yields
  \begin{align*}
      \mathbb{E} [(\lambda_{\rm min} (\Sigma_{k+1}))^{-1}\vert \mathcal{A}_k]
      &\leq (\lambda_{\rm min}(\Sigma_k))^{-1} \big(1 + \gamma_k  \mathbb{E} \Big[\theta_i^\top \Sigma_k\theta_i-1 \vert \mathcal{A}_k\Big]\big). 
  \end{align*}
     By independence between the distribution of $\theta_j$ and $\mathcal{A}_k$, and by the $\mathcal{A}_k$-measurability of $\Sigma_k$, 
   $$
   \mathbb{E} \big[\theta_i^\top \Sigma_k\theta_i-1 \vert \mathcal{A}_k\big] = \mathbb{E}_\theta \big[\theta^\top \Sigma_k\theta \big] -1 = \frac{1}{d}\mathrm{Tr}(\Sigma_k) - 1.
   $$
 However, moments are bounded along iterations from Proposition \ref{prop_momentsbounded}, so $\mathrm{Tr}(\Sigma_k) \leq \mathrm{Tr}(\Id)=d$. 
 Combining this with the two equations above induces
\begin{align*}
      \mathbb{E} [(\lambda_{\rm min} (\Sigma_{k+1}))^{-1}\vert \mathcal{A}_k]
      \leq (\lambda_{\rm min}(\Sigma_k))^{-1},
\end{align*}
and \eqref{boundlambdamin} follows by induction. 


Now, fix $p\geq 2$. Taking the power $p$ and applying Jensen's inequality in \eqref{BoundsEigenvaluesINF2} induces 
  \begin{equation}\label{BoundsEigenvaluesINF3}
      (\lambda_{\rm min} (\Sigma_{k+1}))^{-p}
      \leq (\lambda_{\rm min}(\Sigma_k))^{-p} \big(1-\gamma_k + \gamma_k \theta_i^\top \Sigma_k\theta_i\big)^p.
  \end{equation} 
By the binomial theorem, for $Z_{k,i} =  \theta_i^\top \Sigma_k\theta_i - 1$, 
$$
\big(1+ \gamma_k Z_{k,i}\big)^p = 1 + p\gamma_k Z_{k,i} + \sum_{r=2}^p \binom{p}{r} \gamma_k^r Z_{k,i}^r.
$$
Taking the expectation with respect to $\mathcal{A}_k$, and using upper-bounds from Lemma~\ref{upperboundsMomentsQuadratForms}, 
  $$
  \mathbb{E}[\big(1+ \gamma_k Z_{k,i}\big)^p \vert \mathcal{A}_k] 
  \leq 1 + \gamma_k^2\sum_{r=2}^p \binom{p}{r} \gamma_k^{r-2} \mathbb{E}(Z_{k,i}^r\vert \mathcal{A}_k) <+\infty.
  $$
  Plugging this in \eqref{BoundsEigenvaluesINF3}, and reasoning by induction, it exists $B>0$ such that the desired result holds. 
  
  \end{proof}
  
  \begin{lemma}\label{upperboundsMomentsQuadratForms}
  Let $A\in \mathbb{R}^{d\times d}$ be a positive semi-definite matrix verifying $\mathrm{Tr}(A)\leq d$. 
  For $\theta$ uniformly distributed over the unit sphere, 
  $$
  \mathbb{E}_\theta(\theta^\top A\theta -1) \leq 0,
  $$
  and, for all $p\geq 2$, 
  $$
  \mathbb{E}_\theta[(\theta^\top A\theta -1)^p] \leq 1+\sum_{r = 1}^p \binom{p}{r} \lambda_{\rm max}(A)^{r-1}(-1)^r <+\infty. 
  $$
  \end{lemma}
  \begin{proof}
  The first point is a byproduct of $\mathbb{E}_\theta(\theta^\top A\theta -1) = \mathrm{Tr}(A)/d - 1$.
  From the fact that $\mathrm{Tr}(a)=a$ if $a\in \mathbb{R}$, and the cyclic property of $Tr$, 
\begin{align*}
    \mathbb{E}_\theta[( \theta^\top  A \theta)^2] &=
      \mathbb{E}_\theta Tr [ \theta^\top A \theta \theta^\top  A \theta]
    \leq \mathbb{E}_\theta Tr [\theta \theta^\top A \theta \theta^\top  A ]. 
\end{align*}
Because $A$ and $\theta\theta^\top $ are positive semi-definite, the von Neumann’s trace inequality implies
$$
   \mathbb{E}_\theta[( \theta^\top  A \theta)^2] \leq \mathbb{E}_\theta \big(\lambda_{\rm max}(\theta\theta^\top )Tr[A \theta \theta^\top A]\big) = \mathbb{E}_\theta \big(\lambda_{\rm max}(\theta\theta^\top )Tr[A^2 \theta \theta^\top ]\big). 
$$
By linearity of $\mathbb{E}_\theta $ and $Tr$, together with $\lambda_{\rm max}(\theta\theta^\top )\leq 1$ and $\mathbb{E}_\theta [\theta\theta^\top ] = \Id /d$, 
    \begin{align*}
    \mathbb{E}_\theta[( \theta^\top  A \theta)^2]
    &\leq Tr[A^2 \mathbb{E}_\theta( \theta \theta^\top )] = \frac{\mathrm{Tr}(A^2)}{d}. 
\end{align*}
Using again the von Neumann's trace inequality, $\mathrm{Tr}(A^2)\leq \lambda_{\rm max} (A)\mathrm{Tr}(A)$, and $\mathrm{Tr}(A)\leq d$, so $\mathrm{Tr}(A^2)/d\leq \lambda_{\rm max} (A)$ which proves the first point: 
$$
\mathbb{E}_\theta((\theta^\top A\theta -1)^2)= \mathbb{E}_\theta [(\theta^\top A\theta)^2 + 1 - 2 \theta^\top A\theta] \leq \lambda_{\rm max} (A) + 1.
$$
With the same arguments as above, one can deduce that, for all $p\geq 1$,
$$
\mathbb{E}_\theta[( \theta^\top  A \theta)^p]\leq \frac{\mathrm{Tr}(A^{p})}{d} \leq \lambda_{\rm max}(A)^{p-1}.
$$
Thus, the last claims follows by the binomial theorem,
\begin{align*} 
\mathbb{E}_\theta[(1-\theta^\top A\theta)^p] &= \mathbb{E}_\theta \sum_{r = 0}^p \binom{p}{r} (\theta^\top A\theta)^r(-1)^r \leq 1+\sum_{r = 1}^p \binom{p}{r} \lambda_{\rm max}(A)^{r-1}(-1)^r.
\end{align*}
\end{proof}

\subsection{A sufficient condition under arbitrary covariance}
Let $\mu = \mathcal{N}(0,\Lambda)$ for a general covariance matrix $\Lambda$. 
\begin{proposition}\label{sufficientconditionarbitrarycov}
    For all $p\geq 1$, a sufficient condition for the existence of a finite constant $C_p>0$ such that
    $$
    \sup_{k\in\mathbb{N}}\mathbb{E}[(\lambda_{\rm min} (\Sigma_{k}))^{-p}]  \leq C_p
    $$
    is the following,
    $$
    \mathbb{E}\left(\sum_{k\geq 0}^\infty  \gamma_k \mathbb{E}_\theta \Big[\Big(\frac{\theta^\top \Sigma_k\theta}{\theta^\top \Lambda\theta}\Big)^p-1\Big]\right) < +\infty.
    $$
\end{proposition}
\begin{proof}
As a byproduct of Proposition \ref{ControlEigSigk}, and proceeding as in the beginning of Proposition \ref{ControlEigSigk2}, the following counterpart of 
\eqref{BoundsEigenvaluesINF2} holds, 
$$
      (\lambda_{\rm min} (\Sigma_{k+1}))^{-1}
      \leq (\lambda_{\rm min}(\Sigma_k))^{-1} \big(1-\gamma_k + \gamma_k \frac{\theta_i^\top \Sigma_k\theta_i}{\theta_i^\top \Lambda\theta_i}\big).
$$
Fix $p\geq 1$, and apply the power $p$ and Jensen's inequality to obtain that 
$$
      (\lambda_{\rm min} (\Sigma_{k+1}))^{-p}
      \leq (\lambda_{\rm min}(\Sigma_k))^{-p} \big(1-\gamma_k + \gamma_k \Big(\frac{\theta_i^\top \Sigma_k\theta_i}{\theta_i^\top \Lambda\theta_i}\Big)^p\big).
$$
Taking the conditional expectation, 
$$
      \mathbb{E}[(\lambda_{\rm min} (\Sigma_{k+1}))^{-p}\vert \mathcal{A}_k]
      \leq (\lambda_{\rm min}(\Sigma_k))^{-p} \big(1+ \gamma_k \mathbb{E}_\theta \Big[\Big(\frac{\theta^\top \Sigma_k\theta}{\theta^\top \Lambda\theta}\Big)^p-1\Big]).
$$
Taking the expectation and reasoning by induction, we deduce that
$$
      \mathbb{E}[(\lambda_{\rm min} (\Sigma_{k+1}))^{-p}]
      \leq \mathbb{E}[(\lambda_{\rm min}(\Sigma_0))^{-1}] + \mathbb{E}\left(\sum_{j=0}^k  \gamma_k \mathbb{E}_\theta \Big[\Big(\frac{\theta^\top \Sigma_k\theta}{\theta^\top \Lambda\theta}\Big)^p-1\Big]\right).
$$
Thus: 
 $
    \forall p\geq 1, \exists C_p>0,\, \, \sup_{k\in\mathbb{N}}\mathbb{E}[(\lambda_{\rm min} (\Sigma_{k}))^{-p}]  \leq C_p.
    $
\end{proof}

\section{Proof of our main result: Theorem \ref{VitesseN01}}\label{ProofMainResult}

\subsection{Proof of Proposition \ref{ASconvergencegaussianHyp}}\label{proofAScvgceGauss}
Recall that \citet{robbins1971convergence} implies $(\mathscr{F}(\sigma_k))_{k \geq 0}$ converges almost surely to a finite random variable, as a direct byproduct of \eqref{recursiveIneq}. Hence, one only needs to show that the limit random variable is zero $a.s$. 
For this, two assumptions are needed in \citet[Theorem 2]{li2023measure}, with standard arguments
\citep{duflo1996algorithmes}:
(i) $(\sigma_k)_{k \geq 1}$ remains in a compact subset $K$ of $(\mathcal{P}_{2,ac}(\RR^d),W_2)$ and (ii) $
\nabla \mathscr{F}(\sigma) = 0\Longrightarrow \sigma = \mu
$ for $\sigma\in K$. 
Under continuity and compact supports, \citet{li2023measure} show (i), while (ii) holds for absolutely continuous measures from \citep[Lemma~5.7.2]{bonnotte2013unidimensional}. 
Taken together, these two facts imply almost-sure convergence. 

To extend this result to Gaussian measures, we verify (i) and (ii) hereafter. 
By \citet[Propositions 4 and 5]{li2023measure}, 
the iterates $\sigma_k$ remain in a compact set of $(\mathcal{P}_{2,ac}(\mathbb{R}^d),W_2)$ if $\sigma,\mu \in \mathcal{P}_{2,ac}(\mathbb{R}^d)$ and have finite third-order moments. 
This is verified for Gaussian distributions, hence there is a set $K$ such that (i) holds. In addition, by \Cref{ControlEigSigk}, each iterate $\sigma_k$ remains Gaussian with strictly positive definite covariance $\Sigma_k$ for every finite $k$.
Therefore, regarding (ii), \Cref{KLGaussians} gives that
\begin{equation*}
    \mathscr{F}(\sigma_k)^2 \leq 2\mathrm{M}_2(\mu)\Big(1+\frac{\lambda_{\rm max}(\Lambda)}{\lambda_{\rm min}(\Sigma_k)}\Big) \| \nabla_{W_2} \mathscr{F}(\sigma_k) \|_{\sigma_k}^2,
\end{equation*}
  where we also use that $W_2^2(\sigma_k,\mu)\leq 4\mathrm{M}_2(\mu)$ from Proposition \ref{prop_momentsbounded}.
Consequently, (ii) is verified under general Gaussian covariances, and the desired result follows.

\subsection{Proof of Theorem \ref{VitesseN01}}\label{proofThmN01}
    By \Cref{corollarydiagGaussians}, the following PL condition holds for any $k \geq 1$,
\begin{equation}\label{resultCorPLIsot}
   \mathscr{F}(\sigma_k) \leq  \frac{C_{k,d}}{2}
\Bigl(1 + \frac{1}{\lambda_{\min}(\Sigma_k)}\Bigr)\Vert \nabla_{W_2} \mathscr{F}(\sigma_k) \Vert_{\sigma_k}^2\,,
\end{equation}
    with $C_{k,d} = d(d+2) M_k/m_k$, $M_k = \max(\lambda_{\rm max}(\Sigma_k),1)$ and $m_k = \min(\lambda_{\rm min}(\Sigma_k),1)$.    
    By \Cref{ControlEigSigk2}, $M_k\leq d$, thus $C_{k,d}\leq d^2(d+2)/m_k$. 
    Additionally, by \Cref{prop_momentsbounded}, we have $\mathrm{Tr}(\Sigma_k) \leq \mathrm{Tr}(\Lambda)$ where $\Lambda$ denotes the covariance matrix of the target Gaussian. A contradiction argument then implies $\lambda_{\rm min}(\Sigma_k) \leq \lambda_{\rm max}(\Lambda)$, and in the special case $\Lambda = {\bf I}_d$, this gives $\lambda_{\rm min}(\Sigma_k) \leq 1$. Therefore, $C_{k,d}\leq d^2(d+2)/\lambda_{\rm min}(\Sigma_k)$ (since $m_k = \lambda_{\rm min}(\Sigma_k)$), and
    $1 + 1/{\lambda_{\rm \min}(\Sigma_k})\leq 2/{\lambda_{\rm \min}(\Sigma_k})$. Therefore, \eqref{resultCorPLIsot} entails that 
    \begin{equation} \label{eq:pl_proof}
      \mathscr{F}(\sigma_k) 
      \leq
      \frac{d^2(d+2)}{\lambda_{\rm min}(\Sigma_k)^2}
      \Vert \nabla_{W_2} \mathscr{F}(\sigma_k) \Vert_{\sigma_k}^2 \,.
    \end{equation}
Denote by $B_k = {d^2(d+2)}/{\lambda_{\rm min}(\Sigma_k)^2}$. \Cref{ControlEigSigk2} gives us that all the moments of $B_k$ are finite: $\sup_k\mathbb{E}[B_k^p] \leq c_p<+\infty$ for all $p\in \mathbb{N}^*$.  
In other words, the expected PL inequality in \Cref{hyp_PL} for $\tau=1$ holds along the iterates $\sigma_k$. 
Thus, the result is a byproduct of  \Cref{thmPL}.

\subsection{Proof of Theorem \ref{thmPL}}
 
    
\paragraph{Random events and main recursion.}
 Since $(B_k)_{k\geq 1}$ is a sequence of random variables (\Cref{hyp_PL}), we condition the analysis on the event $G_k = \{ B_k \leq 1/g_k\}$ to apply the PL inequality. 
This is done by introducing a sequence of positive numbers $(g_k)_{k \geq 1}$ with $\lim_{k \to +\infty} g_k=0$, so that $\mathds{1}_{G_k}$ converges to $1$ almost surely.   
Denote by $G_k^c$ the complementary event: $G_k^c = \{ B_k > 1/g_k\}$. 

 \paragraph{Expected PL inequality.} We begin with the rates obtained under
Assumption \ref{hyp_PL} with $\tau =1$. 
Using the descent lemma \eqref{recursiveIneq} and the PL inequality of Assumption \ref{hyp_PL} on the event $G_k$,
    \begin{align*}
    \mathbb{E}[\mathscr{F}(\sigma_{k+1}) \vert \mathcal{A}_{k}] 
    &\leq
    (1+\gamma_k^2)\mathscr{F}(\sigma_k) -\gamma_k\Vert \nabla_{W_2} \mathscr{F}(\sigma_k) \Vert_{\sigma_k}^2 (\mathds{1}_{G_k} + \mathds{1}_{G_k^c}),\\
    &\leq (1+\gamma_k^2)\mathscr{F}(\sigma_k) -\gamma_k g_k \mathscr{F}(\sigma_k)\mathds{1}_{G_k}
    -\gamma_k
    \Vert \nabla_{W_2} \mathscr{F}(\sigma_k) \Vert_{\sigma_k}^2   \mathds{1}_{G_k^c}. 
    \end{align*}
 Now, we can plug the decomposition \eqref{decompL_Gradient}, that is $2\mathcal{F}(\sigma_k) - \Vert \nabla_{W_2} \mathscr{F}(\sigma_k) \Vert_{\sigma_k}^2 = E_k$ 
 for $E_k = \mathbb{E}_P\!\left[\Vert \overline{T}_{\sigma_k} - T_{{\sigma_k},P}\Vert_{\sigma_k}^2\right]$. 
 This implies 
 \begin{align}
    \mathbb{E}[\mathscr{F}(\sigma_{k+1}) \vert \mathcal{A}_{k}] 
    &\leq (1+\gamma_k^2)\mathscr{F}(\sigma_k) -\gamma_k g_k \mathscr{F}(\sigma_k)\mathds{1}_{G_k}
    - \gamma_k
    2\mathscr{F}(\sigma_k)   \mathds{1}_{G_k^c} 
    + \gamma_k E_k\mathds{1}_{G_k^c},\noindent\\
    &\leq (1+\gamma_k^2-\gamma_kg_k)\mathscr{F}(\sigma_k)\mathds{1}_{G_k} 
    + 
    (1-\gamma_k)^2\mathscr{F}(\sigma_k)\mathds{1}_{G_k^c}
    + \gamma_k E_k\mathds{1}_{G_k^c}.\label{finalRec3terms}
    \end{align}
    The first term holds under the event $G_k$, where the PL inequality holds. 
    Depending on the choice of $g_k$
    (which, for now, only needs to converge to zero), this will result in an exponential rate governed by $\exp(\sum_{j=0}^k \gamma_k^2 - \gamma_k g_k)$ since $\prod_i (1+ a_i) \leq \sum_i \exp(a_i)$.  
    
    On the event $G_k^c$, we obtained two terms from the decomposition  \eqref{decompL_Gradient}. 
    One first recovers $\mathscr{F}(\sigma_k)(1-\gamma_k)^2$, which, if it were alone, would also imply an exponential rate. 
    The final rate will thus be governed by $\gamma_k E_k \mathds{1}_{G_k^c}$. 
    
    \noindent In the absence of any convergence result for $E_k$, we observe that $E_k \leq 2\mathcal{F}(\sigma_k) \leq 4 \mathrm{M}_2(\mu)$ from \Cref{prop_momentsbounded}. 
    Therefore, $\gamma_k E_k \mathds{1}_{G_k^c} \leq \gamma_k 4 \mathrm{M}_2(\mu)\mathds{1}_{G_k^c}$, and the final rate will be governed solely by $\mathds{1}_{G_k^c}$. 
    Combining this with the recursion \eqref{finalRec3terms} and using that 
     $-2\gamma_k \leq -\gamma_k g_k$, we obtain
 \begin{align*}
    \mathbb{E}[\mathscr{F}(\sigma_{k+1})] 
    &\leq (1+\gamma_k^2-\gamma_kg_k)\mathbb{E}[\mathscr{F}(\sigma_k)]
    + 4 \mathrm{M}_2(\mu) \gamma_k \mathbb{E}[\mathds{1}_{G_k^c}]
    \end{align*}
    after taking the expectation. 
     By Markov's inequality, 
    \begin{equation}\label{PBkpMarkovIneq}
        \forall p\in \mathbb{N}^*, \quad \mathbb{P}(G_k^c) = \mathbb{P}(B_k > 1/g_k) \leq g_k^p \mathbb{E}[B_k^p]\leq c_pg_k^p. 
    \end{equation}
    Here, $p$ can be chosen freely. 
    Choosing a larger value of $p$ improves the decay of the factor $g_k^p$, at the cost of increasing the constant $c_p$.
    This yields the main recursion: for any $p\in \mathbb{N}^*$, 
     \begin{align}\label{RecursionFsigmak}
    \mathbb{E}[\mathscr{F}(\sigma_{k+1})] 
    &\leq (1+\gamma_k^2-\gamma_kg_k)\mathbb{E}[\mathscr{F}(\sigma_k)]
    + 4 \mathrm{M}_2(\mu)c_p \gamma_k g_k^p. 
    \end{align}
    It only remains to choose $(g_k)$ and $p$. For two different choices, we obtain two different rates, whose optimality depends on $\alpha$. 

    \paragraph{Case 1: $0<\alpha<2/3$.}
    For $\gamma_k = 1/(k+1)^\alpha$, let $g_k = 1/(k+1)^\epsilon$ so that
the objective rephrases as choosing $\epsilon$ and $p$.  
For any $k\geq 0$, one has for $\epsilon < \alpha$,
    $$
    \gamma_k^2 - \gamma_kg_k 
    =
    \frac{1}{(k+1)^{\alpha+\epsilon}}\Big(\frac{1}{(k+1)^{\alpha-\epsilon} }-1 \Big)
    \leq 
    \frac{1}{(k+1)^{\alpha+\epsilon}}\Big(\frac{1}{2^{\alpha-\epsilon} }-1 \Big).
    $$
    In other words, $\gamma_k^2 - \gamma_kg_k = -a/(k+1)^{\alpha+\epsilon}$ for $a = 1 - 1/2^{\alpha - \epsilon} \in (0,1)$. Thus,   
\eqref{RecursionFsigmak} rewrites
\begin{align}\label{RecursionFsigmak2}
    \mathbb{E}[\mathscr{F}(\sigma_{k+1})] 
    &\leq (1- \frac{a}{(k+1)^{\alpha + \epsilon}})\mathbb{E}[\mathscr{F}(\sigma_k)]
    + 4 \mathrm{M}_2(\mu)c_p \frac{1}{(k+1)^{\alpha + \epsilon p}}. 
    \end{align}
The desired rate follows from \citet[Lemma A.3]{BBAOS}, which is a variant of Chung's Lemma \citep{chung1954stochastic} (see also \citet[Theorem 1]{moulines2011non}).
To allow for $\alpha < 1/2$, we rewrite \citet[Lemma A.3]{BBAOS} in Lemma \ref{lemmaBB}.
This gives the existence of a  
 constant $C>0$ such that
\begin{equation}\label{AlmostRate}
\mathbb{E}[\mathscr{F}(\sigma_{k})]
\leq 
\frac{C}{k^{\epsilon(p-1)}},
\end{equation}
under $\epsilon < \alpha$ and conditions that rephrase as
$\alpha + \epsilon <1 < \alpha + \epsilon p.
$
Equivalently, the above needs
\begin{equation*}
\epsilon < \min(\alpha,1-\alpha)
\qquad \mbox{and} \qquad
1 -\alpha <\epsilon p.
\end{equation*}
For all $0<\epsilon< \min(\alpha,1-\alpha)$, there is $p\geq 1$ that meets this condition. 
Now, $1-\alpha-\epsilon< \epsilon(p-1)$ yields 
\begin{equation*}
\forall \epsilon < \min(\alpha,1-\alpha): \qquad 
\mathbb{E}[\mathscr{F}(\sigma_{k})]
\leq 
\frac{C}{k^{\epsilon(p-1)}}
\leq \frac{C}{k^{1-\alpha-\epsilon}}.
\end{equation*}
In particular, this is true for $0 < \alpha <1$, even if $0< \alpha <1/2$ does not satsify the Robbins-Monro conditions. 

\paragraph{Case 2: $2/3<\alpha<1$.} We now turn to the second rate, that is faster in the regime $2/3<\alpha<1$. 
One can use \Cref{prop_momentsbounded} to bound $\gamma_k^2 \mathbb{E}[\mathscr{F}(\sigma_k)]$ by $2\mathrm{M}_2(\mu)\gamma_k^2$, hence \eqref{RecursionFsigmak} becomes
\begin{align}\label{RecursionFsigmak2}
    \mathbb{E}[\mathscr{F}(\sigma_{k+1})] 
    &\leq (1-\gamma_kg_k)\mathbb{E}[\mathscr{F}(\sigma_k)]
    + 2 \mathrm{M}_2(\mu)\gamma_k^2
    + 4 \mathrm{M}_2(\mu)c_p \gamma_k g_k^p. 
    \end{align}
By choosing $\gamma_k = 1/(k+1)^\alpha$ and $g_k = 1/(k+1)^{1-\alpha}$, 
we have $\gamma_kg_k = 1/(k+1)$ and $g_k^{p}= 1/(k+1)^{(1-\alpha)p}$. 
    Therefore, $g_k^{p} \leq \gamma_k$ as soon as one chooses $p\geq \alpha/(1-\alpha)$. 
    In this case, \eqref{RecursionFsigmak2} becomes, for $C=2\mathrm{M}_2(\mu)(1 + 2c_p)$, 
    $$
    \mathbb{E}[\mathscr{F}(\sigma_{k+1}) ] \leq \Big(1  - \frac{1}{k+1}\Big)\mathbb{E}[\mathscr{F}(\sigma_k)] + \frac{C}{(k+1)^{2\alpha}} \,.
    $$
   The desired rate follows directly from Chung's Lemma \citep{chung1954stochastic}: 
    $$
    \mathbb{E}[\mathscr{F}(\sigma_{k}) ] \leq \frac{C}{k^{2\alpha-1}} \,.
    $$

\begin{remark}[Comparison of constants]
When $\alpha<1/2$, $\min(\alpha,1-\alpha)=\alpha$, and one needs $p> (1-\alpha)/\epsilon > (1-\alpha)/\alpha$. 
When $\alpha > 2/3$, the second rates requires $p \leq \alpha / (1-\alpha)$. 
Thus, when $\alpha \rightarrow 0$ or when $\alpha \rightarrow 1$, the order of magnitude of $p$ is the same, and a fortiori the constant $c_p$. Hence, constants in both rates are similar, and this does not explain the faster convergence observed when $\alpha$ is almost $0$ in practice. 
\end{remark}

\color{black}

\paragraph{Expected PL-like inequality.}
    We now turn to show the second result. 
Using the descent lemma \eqref{recursiveIneq}, Assumption \ref{hyp_PL} with $\tau =2$ would imply instead 
    \begin{equation}\label{RecInIndic0_KL}
    \mathbb{E}[\mathscr{F}(\sigma_{k+1})  \vert \mathcal{A}_{k}]  \leq  \mathscr{F}(\sigma_k)\mathds{1}_{G_k} - \gamma_k g_k \mathscr{F}(\sigma_k)^2\mathds{1}_{G_k} + \mathscr{F}(\sigma_k)\mathds{1}_{G_k^c} + \mathscr{F}(\sigma_k) \gamma_k^2,
    \end{equation} 
    where, in the last term, we also do not bound $\mathscr{F}(\sigma_k)$ by $ 2\mathrm{M}_2(\mu)$. 
    After taking the expectation, Markov inequality can be applied as in
    \eqref{PBkpMarkovIneq} to deduce from \eqref{RecInIndic0_KL} that   
    \begin{equation}\label{RecIneqIndicGkkl_}
    \mathbb{E}[\mathscr{F}(\sigma_{k+1}) ] 
    \leq \mathbb{E}[\big(\mathscr{F}(\sigma_k) 
    - \gamma_k g_k \mathscr{F}(\sigma_k)^2\big)\mathds{1}_{G_k}] + 
 2\mathrm{M}_2(\mu) \sqrt{c_p} g_k^{p/2} + \mathbb{E}[\mathscr{F}(\sigma_k)] \gamma_k^2.
    \end{equation}
    To remove $\mathds{1}_{G_k}$ above, note that
    $
    \gamma_kg_k \mathscr{F}(\sigma_k)^2 \leq (2\mathrm{M}_2(\mu))^{-1} \mathscr{F}(\sigma_k)^2  \leq \mathscr{F}(\sigma_k),
    $
    as soon as $\gamma_kg_k \leq (2\mathrm{M}_2(\mu))^{-1}$. 
    Because $g_k$ is a flexible choice, this just means that $k$ needs to be large enough.  
For such $k$, we deduce 
$\mathbb{E}[\big(\mathscr{F}(\sigma_k)-\gamma_k g_k \mathscr{F}(\sigma_k)^2\big)\mathds{1}_{G_k^c}]\geq 0$. Adding this to \eqref{RecIneqIndicGkkl_} gives
    \begin{equation}\label{recKLalmostthere}
    \mathbb{E}[\mathscr{F}(\sigma_{k+1}) ] 
    \leq \mathbb{E}[\mathscr{F}(\sigma_k)](1+\gamma_k^2)
    - \gamma_k g_k \mathbb{E}[\mathscr{F}(\sigma_k)]^2 + 
 2\mathrm{M}_2(\mu) \sqrt{c_p} g_k^{p/2},
    \end{equation}
    where we also use that $\mathbb{E}[\mathscr{F}(\sigma_k)]^2 \leq \mathbb{E}[\mathscr{F}(\sigma_k)^2]$ by Jensen's inequality. 
    Now, all that remains is to play around with the constants to obtain the recursion necessary for an extended Chung's lemma.

    Denote by $C =  2\mathrm{M}_2(\mu) \sqrt{c_p}$ and fix $\gamma= C^{3/2}$.
    Let $\gamma_k = 1/(k+\gamma)^\alpha$ with $1/2<\alpha <2/3$. 
    Let $g_k = 1/(k+\gamma)^{2-3\alpha}$ hence 
    $$
    \gamma_k g_k = \frac{1}{(k+\gamma)^{2-2\alpha}}
    \hspace{1cm} \mbox{and} \hspace{1cm} 
    g_k^{p/2} = \frac{1}{(k+\gamma)^{p(1-3\alpha/2)}}.
    $$
    This leads to $g_k^{p/2} \leq 1/(k+\gamma)^{2\alpha}$ if $p\geq 2\alpha/(2-3\alpha)$, hence \eqref{recKLalmostthere} rewrites 
\begin{equation}\label{RecIneqIndicGkkl_2}
    \mathbb{E}[\mathscr{F}(\sigma_{k+1}) ] 
    \leq \mathbb{E}[\mathscr{F}(\sigma_k)]\big(1+\frac{1}{(k+\gamma)^{2\alpha}}\big)
    - \frac{1}{(k+\gamma)^{2-2\alpha}}
     \mathbb{E}[\mathscr{F}(\sigma_k)]^2 + 
    \frac{C}{(k+\gamma)^{2\alpha}}. 
    \end{equation}
    Recall that this holds as soon as $\gamma_kg_k \leq (2\mathrm{M}_2(\mu))^{-1}$, which is equivalent to $(k+\gamma)^{2-2\alpha}\geq 2\mathrm{M}_2(\mu)$. 
But since $\alpha <2/3$, $(k+\gamma)^{2-2\alpha}\geq (k+\gamma)^{2/3}\geq \gamma^{2/3}\geq C \geq 2\mathrm{M}_2(\mu)$.
So the recursion \eqref{RecIneqIndicGkkl_2} holds for all $k\geq 0$. 

We stress that this relates to an extension of Chung's Lemma in the case of a PL-type inequality with $\tau = 2$.  
\citet[Theorem 4]{moulines2011non} deals with a similar recursion, and 
\citet[Lemma 19]{jiang2024generalizedversionchungslemma} generalizes this in several ways. 
Thus, it only remains to verify that \eqref{RecIneqIndicGkkl_2} fulfills the correct requirements. 

To stick to the notations of \citet[Lemma 19]{jiang2024generalizedversionchungslemma}, we introduce $y_{k} = \mathbb{E}[\mathscr{F}(\sigma_k)]$, $a_k = 1/(k+\gamma)^{2-2\alpha}$, $\ell_1 = \ell_2= 1$, $\ell_3=C$, $\tau = 2\alpha/(2-2\alpha)$, so that \eqref{RecIneqIndicGkkl_2} rewrites, for all $k\geq 0$,
\begin{equation*}
    y_{k+1} \leq (1 + \ell_1a_k^\tau)y_k - \ell_2 a_k y_k^2 + \ell_3 a_k^{\tau}.
\end{equation*}

This is exactly the recursion in \citet[Lemma 19]{jiang2024generalizedversionchungslemma}, and our parameters lead to their statement $(b)$. 
Again, with their notations,  
$\zeta = C^{1/2}$, $u_2 = 2\alpha - 1$, 
$p = \rho$, which fulfills the requirements 
$$
1 \geq \big(\frac{2u_2}{\zeta}\big)^\rho 
\quad \mbox{and} \quad 
\gamma \geq \max\{ (\frac{1}{\zeta})^{1/u_2},\zeta\}=\zeta,
$$
thus leading to
$$
y_{k+1}\leq 4 C^{1/2} (k+1+\gamma)^{-u_2}  
+ y_0 \big( \gamma^{-1}(k+1+\gamma)\big)^{-\zeta}.
$$
This is the desired result, as it rewrites
$$
\mathbb{E}[\mathscr{F}(\sigma_k)] \leq 
\frac{4 C^{1/2}}{ (k+\gamma)^{2\alpha - 1} }
+ \frac{\mathbb{E}[\mathscr{F}(\sigma_0)]}{
\big( \gamma^{-1}(k+\gamma)\big)^{\sqrt{C}}},
$$ 
and the second term in the above is faster than the first one. \\

\paragraph{Auxiliary lemma.} The next lemma is taken from \citet[Lemma A.3]{BBAOS}. We just verify that the proof holds without the assumptions $\beta <2$ and $\beta\leq 2\alpha$.  
\begin{lemma}\label{lemmaBB}
    Let $Z_k$ be a sequence of positive numbers satisfying, for all $k\geq 0$, 
    $$
    Z_{k+1} \leq \big( 1- \frac{a}{(k+1)^\alpha} \big)Z_k + \frac{b}{(k+1)^\beta},
    $$
    where $a,b,\alpha,\beta$ are positive constants satisfying $a\leq 1$ and $\alpha<1< \beta$. Then, there exists a positive constant $C$ such that 
    $$
    Z_k \leq \frac{C}{k^{\beta-\alpha}}.
    $$
\end{lemma}
\begin{proof}
    Proceeding as in \citet[Lemma A.3]{BBAOS}, with $0\leq a\leq 1$ and $0\leq\alpha< 1$,
    \begin{equation}\label{boundZk}
    Z_k \leq \exp(\eta(1-k^{1-\alpha}))Z_0 + b \sum_{\ell=1}^k P_{\ell+1}^k \frac{1}{\ell^\beta}
    \end{equation}
    where $\eta=a/(1-\alpha)$ and
    $$
    P_\ell^k = \prod_{i = \ell}^k (1-\frac{a}{i^\alpha}),
    $$
    with the convention that $P_{k+1}^k=1$. 
    For some integer $m$ such that $2k \leq 4m\leq 3k$, 
    \begin{equation}\label{DecompPkm}
    \sum_{\ell=1}^k P_{\ell+1}^k \frac{1}{\ell^\beta} \leq P_{m+1}^k\sum_{\ell=1}^m  \frac{1}{\ell^\beta} + \sum_{\ell=m+1}^k P_{\ell+1}^k \frac{1}{\ell^\beta},
    \end{equation}
    since $P_{m+1}^k\geq P_{\ell+1}^k$ for all $\ell \leq m$. Besides, for $\xi=1 - (3/4)^{1-\alpha}$, 
    $$
    P_{m+1}^k \leq \exp (-a\sum_{i = m+1}^k \frac{1}{i^\alpha} )\leq \exp (\eta (2-\alpha - \xi k^{1-\alpha})),
    $$
    where the first inequality uses $1-x\leq e^{-x}$ and the second
    uses that $m\leq (3/4)k$ and
    $$
    \sum_{i = m+1}^k \frac{1}{i^\alpha} \geq \frac{1}{1-\alpha} (k^{1-\alpha} - m^{1-\alpha} ) - \frac{2-\alpha}{1-\alpha} \geq \frac{\xi k^{1-\alpha} - (2-\alpha)}{1-\alpha}.
    $$
    Plugging this in \eqref{DecompPkm} together with the bound $\sum_{\ell=1}^m  \frac{1}{\ell^\beta} \leq 1 + \frac{1}{\beta-1}$ that holds since $\beta >1$, 
    \begin{equation}\label{DecompPkm2}
    \sum_{\ell=1}^k P_{\ell+1}^k \frac{1}{\ell^\beta} \leq \exp (\eta (2-\alpha - \xi k^{1-\alpha}))(1 + \frac{1}{\beta-1}) + \sum_{\ell=m+1}^k P_{\ell+1}^k \frac{1}{\ell^\beta}.
    \end{equation}
    Regarding the right term above, remark that
    $$
    P_{\ell+1}^k - P_{\ell}^k = \frac{a}{\ell^\alpha} P_{\ell+1}^k. 
    $$
    Therefore, 
    \begin{align*}
         \sum_{\ell=m+1}^k P_{\ell+1}^k \frac{1}{\ell^\beta} &= \frac{1}{a} \sum_{\ell = m+1}^k (P_{\ell+1}^k - P_{\ell}^k) \frac{1}{\ell^{\beta-\alpha}},
    \end{align*}
    where we recall that $\beta > \alpha$. Because $\ell \geq m \geq k/2 $,
        \begin{align*}
         \sum_{\ell=m+1}^k P_{\ell+1}^k \frac{1}{\ell^{\beta}}
         &\leq \frac{2^{\beta-\alpha}}{ak^{\beta-\alpha}} \sum_{\ell = m+1}^k P_{\ell+1}^k - P_{\ell}^k \leq   \frac{2^{\beta-\alpha}}{ak^{\beta-\alpha}} \big(P_{k+1}^k - P_{m+1}^k \big) 
         \leq \frac{2^{\beta-\alpha}}{ak^{\beta-\alpha}}. 
    \end{align*}
    Combining this with \eqref{DecompPkm2} and the bound \eqref{boundZk} gives the final recursion, and the desired rate, 
    $$
    Z_k \leq \exp(\eta(1-k^{1-\alpha}))Z_0 + b \exp (\eta (2-\alpha - \xi k^{1-\alpha}))(1 + \frac{1}{\beta-1}) +\frac{b}{a}\frac{2^{\beta-\alpha}}{k^{\beta-\alpha}}.
    $$
\end{proof}

\color{black}

\section{Additional Numerical Experiments}\label{OtherNumExp}

\subsection{Continuous setting with explicit updates}
In \Cref{ContinuousIDTtargetGeneral}, we extend the experiment of \Cref{ContinuousIDTtargetstandard} by considering a non-isotropic target distribution 
$\mu = \mathcal{N}(0,\Lambda)$, where $\Lambda$ is a diagonal matrix with entries drawn from a Gaussian distribution of mean $10$ and variance $1$ (negative values are discarded).  
Conclusions are similar in this general-covariance setting, where our analysis provide convergence rates only up to the condition \eqref{eq:sufficient_condition}.

\begin{figure}[h!]
    \centering
    {\subfigure[Setting]{\includegraphics[width=0.18\textwidth]{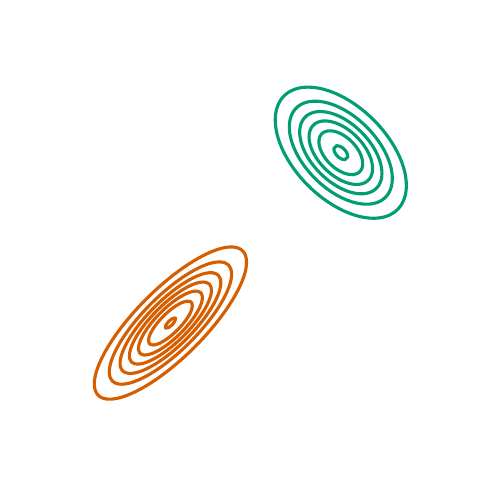} \label{Viz_BetweenGaussiansContinuous2}}}
    {\subfigure[$\alpha=0$]{\includegraphics[width=0.19\textwidth]{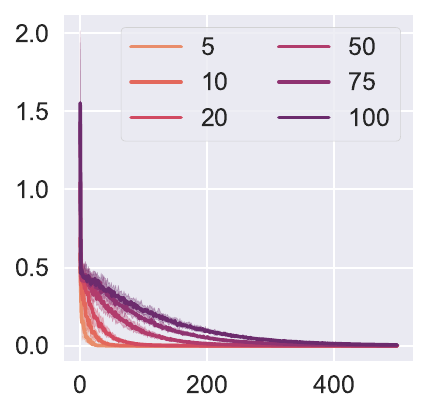} }}
    {\subfigure[$\alpha=0.1$]{\includegraphics[width=0.19\textwidth]{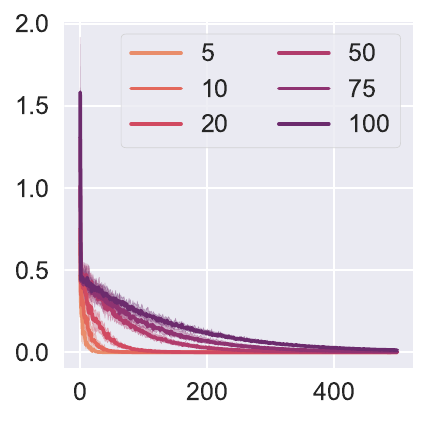} }}
    {\subfigure[$\alpha=0.51$]{\includegraphics[width=0.19\textwidth]{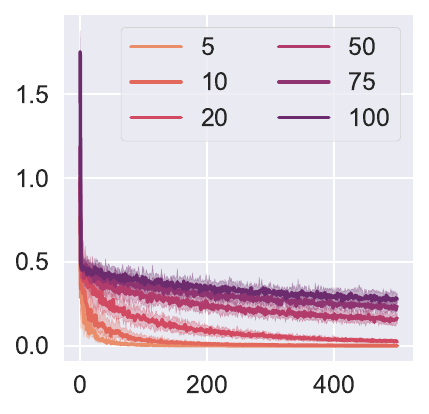} }}
    {\subfigure[$\alpha=0.9$]{\includegraphics[width=0.19\textwidth]{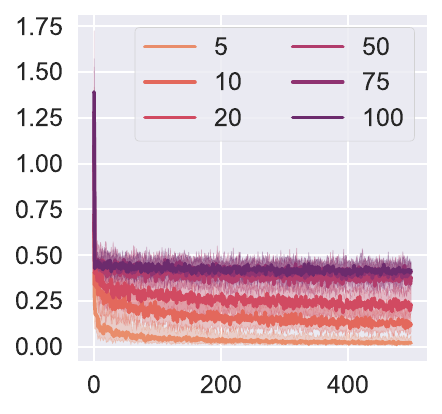} }}
    \caption{Evolution of $SW_2^2(\sigma_k, \mu)$ when $\sigma =\mathcal{N}(0,\Sigma)$ and $\mu=\mathcal{N}(0,\Lambda)$}
    \label{ContinuousIDTtargetGeneral}
\end{figure}

\subsection{Discrete source and target}

We also complement \Cref{BlobToBlob} with \Cref{BlobToN01} and \Cref{BlobToGaussian} on empirical distributions sampled with $n=500$ observations.
The source is drawn from a mixture of Gaussians. 
The target is a Gaussian distribution, either with isotropic or non-isotropic covariance.  
The evolution of the Sliced-Wasserstein distance between iterates and the target reflects again that convergence is faster for learning rates close to $1$, especially for the case $\alpha=0.1$. 
The corresponding slowly decreasing learning rate leads to faster convergence than the fixed learning rate $\gamma_k \equiv 1$ ($\alpha=0$) in all our experiments on discrete samples.

\begin{figure}[h!]
    \centering
    {\subfigure[Setting]{\includegraphics[width=0.18\textwidth]{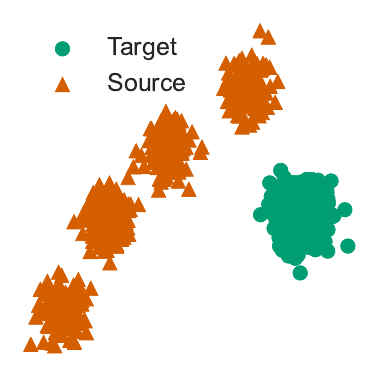}}}
    {\subfigure[$\alpha = 0$]{\includegraphics[width=0.19\textwidth]{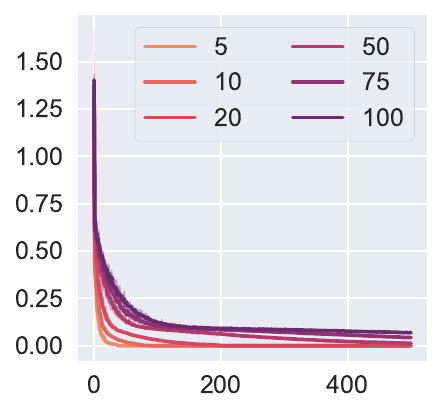}}}
    {\subfigure[$\alpha = 0.1$]{\includegraphics[width=0.19\textwidth]{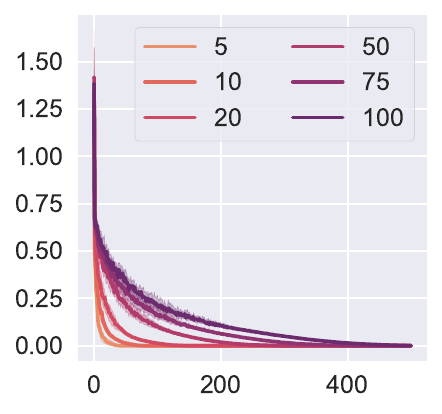}}}
    {\subfigure[$\alpha = 0.51$]{\includegraphics[width=0.19\textwidth]{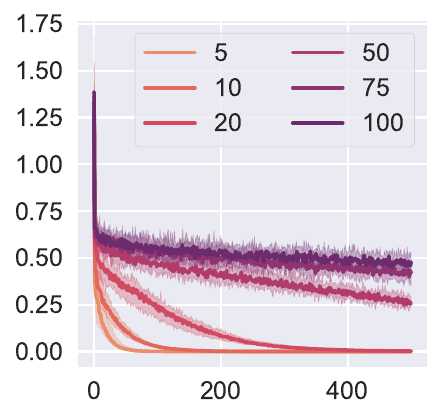}}}
    {\subfigure[$\alpha = 0.9$]{\includegraphics[width=0.19\textwidth]{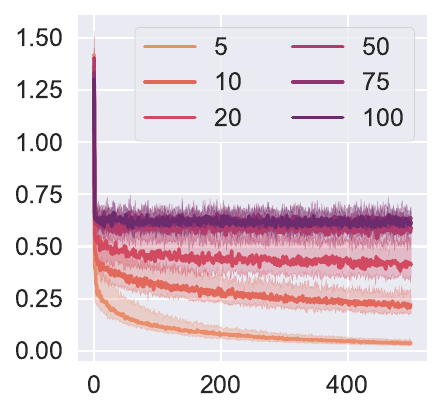}}}
    \caption{Evolution of $SW_2^2(\sigma_k, \mu)$ for discrete source and target distributions. The source is sampled from a mixture of Gaussians, and the target is sampled from $\mathcal{N}(0,{\bf I}_d)$}
    \label{BlobToN01}
\end{figure}

\begin{figure}[h!]
    \centering
    {\subfigure[Setting]{\includegraphics[width=0.18\textwidth]{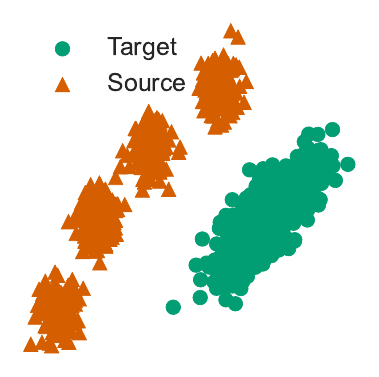} }}
    {\subfigure[$\alpha = 0$]{\includegraphics[width=0.19\textwidth]{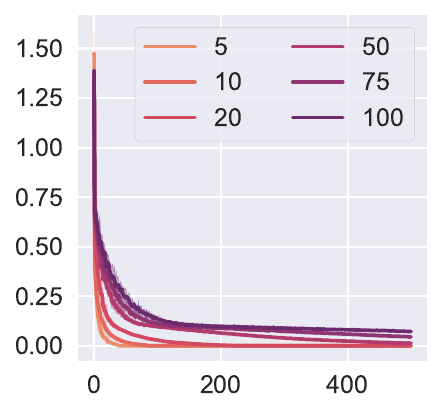} }}
    {\subfigure[$\alpha = 0.1$]{\includegraphics[width=0.19\textwidth]{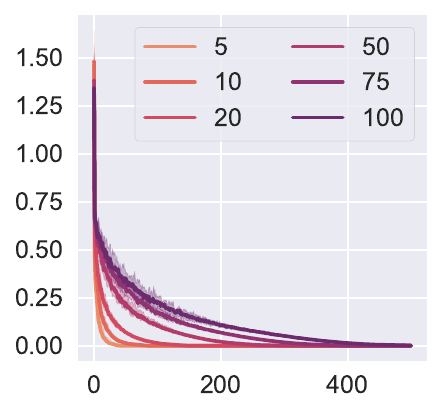} }}
    {\subfigure[$\alpha = 0.51$]{\includegraphics[width=0.19\textwidth]{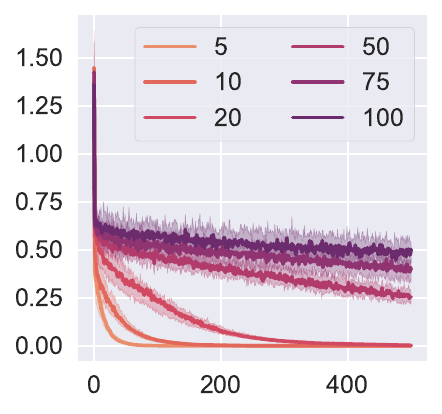}}}
    {\subfigure[$\alpha = 0.9$]{\includegraphics[width=0.19\textwidth]{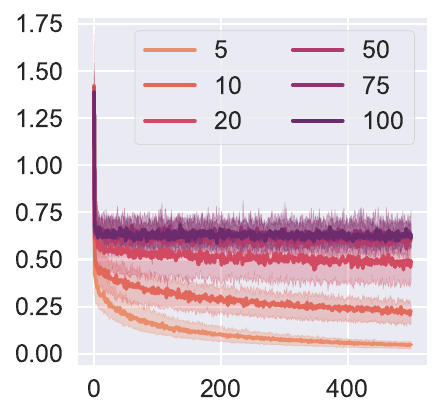} }}
    \caption{Evolution of $SW_2^2(\sigma_k, \mu)$ for discrete source and target distributions. The source is sampled from a mixture of Gaussians, and the target is sampled from $\mathcal{N}(0,\Lambda)$}
    \label{BlobToGaussian}
\end{figure}

\subsection{A single direction for the slice-matching scheme}

\Cref{ContinuousIDTtargetstandardONEDIR} and 
\Cref{curve_eigONEDIR} provide alternative experiments when one replaces the orthonormal set of directions $P_{k+1}$ by a single direction $\theta_{k+1}$. 
We consider continuous Gaussian source and target distributions, so that iterates are explicit. 
\Cref{ContinuousIDTtargetstandardONEDIR} shows the evolution of the Sliced-Wasserstein loss for this alternative algorithm, and \Cref{curve_eigONEDIR} shows the evolution of the min/max eigenvalues. Each considers $N=10$ independent runs, each with a different source covariance. 
This illustrates how the convergence is worsened for all learning rates and all dimensions $d$, as compared to our experiments with multiple directions $P_{k+1}$.

\begin{figure}[t!]
    \centering
    {\subfigure[$\alpha=0$]{\includegraphics[width=0.19\textwidth]{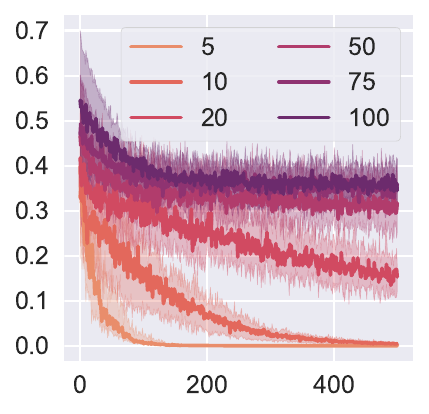} }}
    {\subfigure[$\alpha=0.1$]{\includegraphics[width=0.19\textwidth]{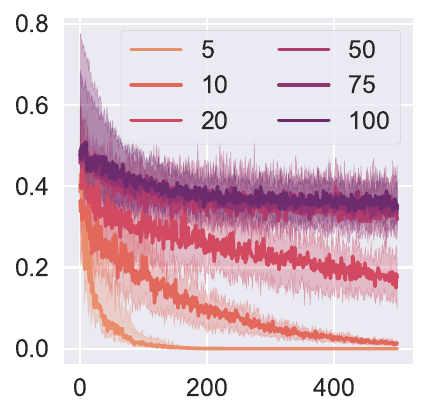}}}
    {\subfigure[$\alpha=0.51$]{\includegraphics[width=0.19\textwidth]{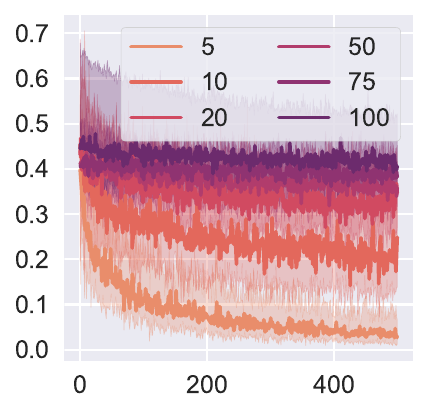} }}
    {\subfigure[$\alpha=0.9$]{\includegraphics[width=0.19\textwidth]{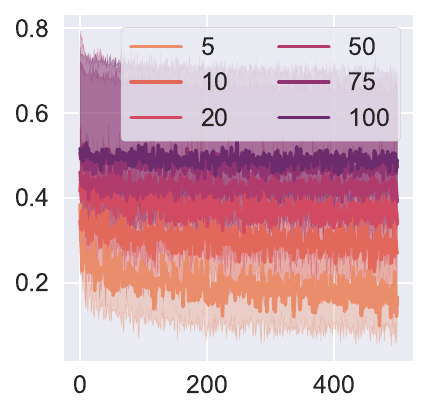}}}
    \caption{Evolution of $SW_2^2(\sigma_k, \mu)$ when $\sigma =\mathcal{N}(0,\Sigma)$ and $\mu=\mathcal{N}(0,{\bf I }_d)$, with slice-matching maps along a single direction $\theta_{k+1}$ instead of an orthonormal basis $P_{k+1}$}
    \label{ContinuousIDTtargetstandardONEDIR}
\end{figure}
\begin{figure}[t!]
    \centering
    {\subfigure[$\alpha=0$]{\includegraphics[width=0.2\textwidth]{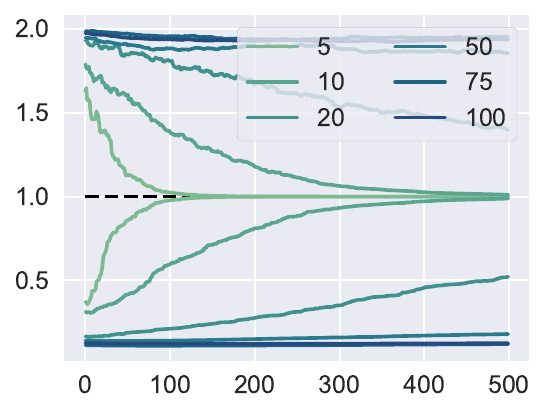} }}
    {\subfigure[$\alpha=0.1$]{\includegraphics[width=0.2\textwidth]{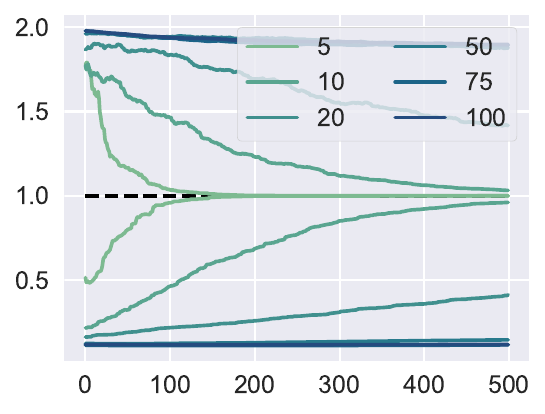} }}
    {\subfigure[$\alpha=0.51$]{\includegraphics[width=0.2\textwidth]{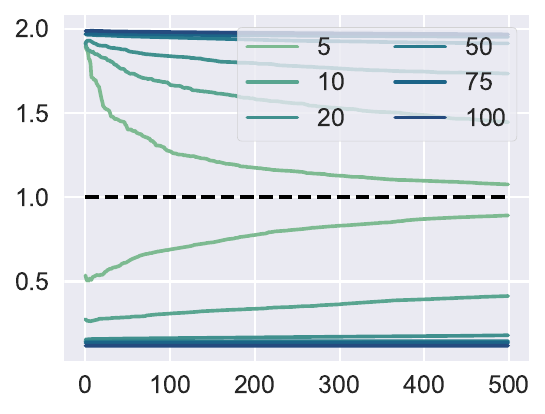} }}
    {\subfigure[$\alpha=0.9$]{\includegraphics[width=0.2\textwidth]{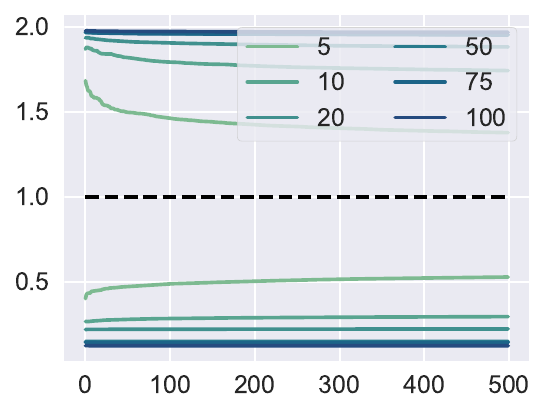} }}
    \caption{Minimum and maximum eigenvalues of the estimated covariances $\Sigma_k$ when $\sigma =\mathcal{N}(0,\Sigma)$ and $\mu=\mathcal{N}(0,{\bf I}_d)$, with slice-matching maps along a single direction $\theta_{k+1}$ instead of an orthonormal basis $P_{k+1}$}
    \label{curve_eigONEDIR}
\end{figure}

\end{document}